%% file: paper.tex
\documentclass[a4paper]{article} 
\usepackage{amsfonts}
\usepackage{amsmath, latexsym, amssymb, bm, ascmac, here, mathtools, comment}
\usepackage{amsthm}
\usepackage{algorithm, algorithmic}
\usepackage{mathrsfs}
\usepackage{multicol}

\usepackage{mcr}
\usepackage{thmE}
\usepackage{algE}
\usepackage{arydshln}
\usepackage{comment}

\usepackage{hyperref}

\newtheorem{theorem}{Theorem}[section]

\newtheorem{lemma}{Lemma}[section]
\newtheorem{corollary}{Corollary}[section]

\newcommand{\argmax}{\mathop{\rm arg~max}\limits}

\newcommand{\1}{\mbox{1}\hspace{-0.25em}\mbox{l}}

\title{Mean-Variance Analysis in Bayesian Optimization under Uncertainty} 
\date{}
\author{
Shogo Iwazaki\thanks{Department of Computer Science, Nagoya Institute of Technology}
\and
Yu Inatsu\thanks{RIKEN Center for Advanced Intelligence Project}
\and
Ichiro Takeuchi
\thanks{Department of Computer Science/Research Institute for Information Science, Nagoya Institute of Technology, 
mail:takeuchi.ichiro@nitech.ac.jp}
\footnotemark[2]
}

\begin{document}
\maketitle
\input{abstract}
\input{section1}
\input{section2}
\input{section3}
\input{section4}
\input{section5}
\input{section6}
\input{bib}
\clearpage
\appendix
\input{supp_A}
\input{supp_B}
\input{supp_C}
\input{supp_D}
\end{document}

%% file: abstract.tex
\begin{abstract}
    We consider active learning (AL) in an uncertain environment 
    in which trade-off between multiple risk measures need to be 
    considered. As an AL problem in such an uncertain environment,
     we study Mean-Variance Analysis in Bayesian Optimization (MVA-BO)
      setting. Mean-variance analysis was developed in the field of 
      financial engineering and has been used to make decisions that 
      take into account the trade-off between the average and variance 
      of investment uncertainty. In this paper, we specifically focus 
      on BO setting with an uncertain component and consider multi-task, 
      multi-objective, and constrained optimization scenarios for the 
      mean-variance trade-off of the uncertain component. When the target 
      blackbox function is modeled by Gaussian Process (GP), we derive the 
      bounds of the two risk measures and propose AL algorithm for each of 
      the above three problems based on the risk measure bounds. 
      We show the effectiveness of the proposed AL algorithms through 
      theoretical analysis and numerical experiments.    
\end{abstract}

%% file: section1.tex
\section{Introduction}

Decision making in an uncertain environment has been studied in various domains. 
For example, in financial engineering, the mean-variance analysis~\cite{hm1952portfolio,markowitz2000mean,keeley1990reexamination} has been introduced as a framework for making investment decisions, taking into account the trade-off between the return (mean) and the risk (variance) of the investment.
In this paper we study active learning (AL) in an uncertain environment.
In many practical AL problems, there are two types of parameters called \emph{design parameters} and \emph{environmental parameters}.
For example, in a product design, while the design parameters are fully controllable, the environmental parameters vary depending on the environment in which the product is used. 
In this paper, we examine AL problems under such an uncertain environment, where the goal is to efficiently find the optimal design parameters by properly taking into account the uncertainty of the environmental parameters.

Concretely, let $f(\bm x, \bm w)$ be a blackbox function indicating the performance of a product, where $\bm x \in \cX$ is the set of controllable design parameters and $\bm w \in \Omega$ is the set of uncontrollable environmental parameters whose uncertainty is characterized by a probability distribution $p(\bm w)$.
We particularly focus on the AL problem where the mean and the variance of the environmental parameters,
\begin{subequations} 
\begin{align}
 \label{eq:meanfunc}
 \EE_{\bm w}[f(\bm x, \bm w)]
 &
 =
 \int_{\Omega} f(\bm x, \bm w) p(\bm w) \text{d} \bm w,
 \\
 \label{eq:varifunc}
 \VV_{\bm w}[f(\bm x, \bm w)]
 &
 =
 \int_{\Omega} \left(f(\bm x, \bm w) - \EE_{\bm w}[f(\bm x, \bm w)]\right)^2 p(\bm w) \text{d} \bm w,
\end{align}
\end{subequations}
respectively, are taken into account.
Specifically, we work on these two uncertainty measures in three different scenarios: multi-task learning scenario, multi-objective optimization scenario, and constrained optimization scenario.
In the first scenario, we study AL for optimizing a weighted sum of these two measures.
In the second scenario, we discuss how to obtain the Pareto frontier of these two measures in an AL setting.
In the third scenario, we consider optimizing one of the two measures under some constraint on the other measure. 
We refer to these problems and the proposed framework for solving them as \emph{Mean-Variance Analysis in Bayesian Optimization (MVA-BO)}.
Figure~\ref{fig:mv_image} shows an illustration of a multi-task learning scenario.

In this study, we employ a Gaussian process (GP) to model the uncertainty of the blackbox function $f(\bm x, \bm w)$.
In a conventional GP-based AL problem (without uncontrollable environmental parameters $\bm w$), the acquisition function (AF) is designed based on how the uncertainty of the blackbox function changes when an input point is selected and the blackbox function is evaluated at the input point.
On the other hand, in MVA-BO, we need to know how the uncertainties of the mean function \eq{eq:meanfunc} and the variance function \eq{eq:varifunc} change by evaluating the blackbox function at the selected input point.
Note that we face the difficulty of not being able to directly evaluate the target functions \eq{eq:meanfunc} and \eq{eq:varifunc}.
It has been shown in a previous study~\cite{o1991bayes} that, when $f(\bm x, \bm w)$ follows a GP, the mean function \eq{eq:meanfunc} also follows a GP.
Unfortunately, however, the variance function \eq{eq:varifunc} does not follow a GP, indicating that we need to develop a new method to quantify how the uncertainty of the variance function changes by evaluating the blackbox function at the selected input point. 
In this study, we extend the GP-UCB algorithm~\cite{DBLP:conf/icml/SrinivasKKS10} to realize MVA-BO in the above mentioned three scenarios by overcoming these technical difficulties.
We demonstrate the effectiveness of the proposed MVA-BO framework through theoretical analyses and numerical experiments.

\paragraph{Related Work}
Various problem setups and methods have been studied for AL and Bayesian optimization (BO) problems when there are multiple target functions. 
One of such problem setup is multi-task BO~\cite{swersky2013multi}.
In this problem setup, the AF is designed to select input points that commonly contribute to optimizing multiple target functions.
Another popular problem setup is multi-objective BO~\cite{emmerich2005single,JMLR:v17:15-047,icml2020_6243}.
The goal of a multi-objective optimization is to obtain so-called \emph{Pareto-optimal} solutions.
The AF in this problem setup is designed to efficiently identify solutions on the Pareto frontier.
Another common problem setup is constrained BO~\cite{gardner2014bayesian,10.5555/3020751.3020778,hernandez2016general}. 
The goal of this problem setup is to find the optimal solution to a constrained optimization problem in a situation where both the objective function and constraint function are blackbox functions that are costly to evaluate.
The AF in this problem setup is designed to select input points that are useful not only for maximizing the objective function but also for identifying the feasible region.
In this paper, we study these three scenarios as concrete examples of MVA-BO.
Unlike conventional multi-task, multi-objective and constrained BOs, the main technical challenges of MVA-BO are that the two target functions \eq{eq:meanfunc} and \eq{eq:varifunc} cannot be directly evaluated and that the latter does not follow a GP. 

Various studies have been published on BO under various types of uncertainty.
The most relevant one to our study is on \emph{Bayesian quadrature optimization (BQO)} \cite{toscano2018bayesian}, the goal of which is to optimize the mean function \eq{eq:meanfunc}.
When the blackbox function follows a GP, the mean function (1a) also follows a GP, suggesting that one can efficiently solve BQO problems by properly modifying the AFs in conventional BO.
By replacing the integrand in \eq{eq:meanfunc} with different uncertainty measures, one can consider various types of AL problems under uncertainty~\cite{beland2017bayesian, iwazaki2020bayesian}.
Another line of research dealing with uncontrollable and uncertain factors in BO is known as \emph{robust BO}. 
The goal of robust BO is to make robust decisions that appropriately take into account the uncertainty of the BO process and the GP model.
For example, input uncertainty in BO has been studied, in which probabilistic noise is inevitably added to the input points when evaluating the target blackbox function.
Although research on BO in an uncertain environment has steadily progressed over the past few years, to our knowledge, there are no AL nor BO studies that take into account the trade-offs between multiple uncertainty measures such as mean-variance analysis.

Decision making under uncertainty is being examined in the field of robust optimization~\cite{ben2009robust,beyer2007robust,ben2002robust}, with especially applications to financial engineering in mind~\cite{schied2006risk,alexander2002economic,fabozzi2007robust}.
It has been pointed out that when making decisions under uncertainty, it is important to balance multiple uncertainty measures appropriately, as represented by the Nobel prize-winning mean-variance analysis in portfolio theory~\cite{hm1952portfolio,markowitz2000mean,keeley1990reexamination}.
Various risk measures, such as Value at Risk (VaR), have been proposed in financial engineering, and these multiple risk measures are used in combination, depending on the purpose of the decision making.
However, to our knowledge, there have been AL or BO studies that have appropriately taken into account multiple uncertainty measures.

\begin{figure}[t]
    \centering
    \includegraphics[width=0.5\linewidth]{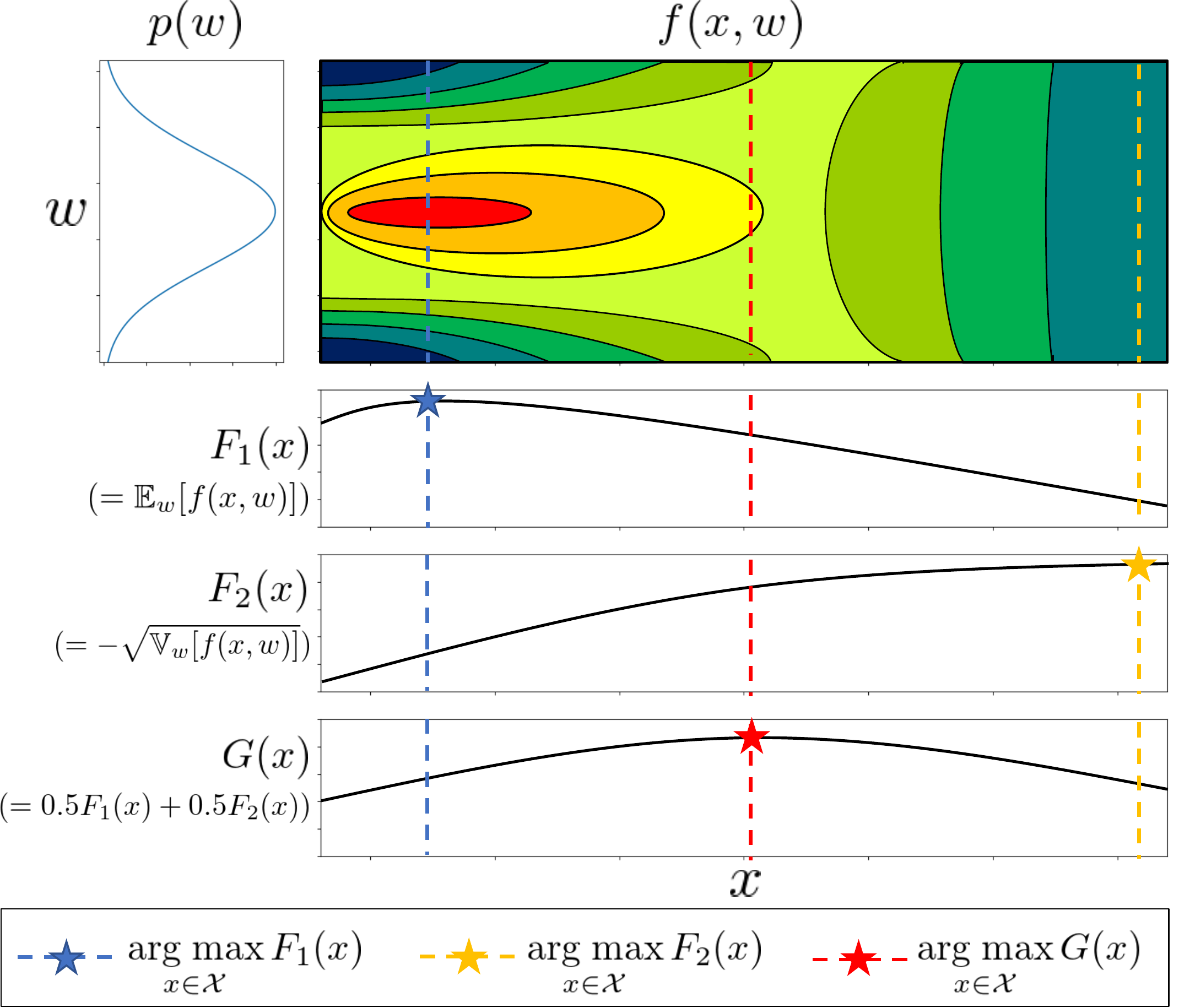}
    \caption{
        2D synthetic example under a multi-task scenario. The horizontal and vertical axes
        represent the design and environmental parameters
        $\bm{x}$ and $\bm{w}$, respectively. 
        Blue and yellow dotted lines indicate the points where 
        expected value $F_1(\bm{x})$ and negative standard deviation
        $F_2(\bm{x})$ of $f(\bm{x}, \bm{w})$ are maximum.
        Our goal is to identify the point on the red line that 
        simultaneously maximize both of $F_1$ and $F_2$ .
    }
    \label{fig:mv_image}
\end{figure}

%% file: section2.tex
\section{Preliminaries}
\subsection{Problem Setup}

Let $f: \mathcal{X} \times \Omega \rightarrow \mathbb{R}$ be a blackbox function
which is expensive to evaluate, where $\mathcal{X} \subset \mathbb{R}^{d_1}$
and $\Omega \subset \mathbb{R}^{d_2}$
are a finite set \footnote{We discuss the case where $\mathcal{X}$ is a continuous set in
appendix \ref{sec:suppD}.}
 and a compact convex set, respectively.
In our setting, a variable $\bm{w} \in \Omega$ is
probabilistically fluctuated by the given density function $p(\bm{w})$ \footnote{
    Note that a probability mass function can also be considered when $\Omega$
    is a finite set. In that case, the subsequent discussions still hold if 
    integral operations are replaced by summation operations.
}.
At every step $t$, a user chooses the next observation point $\bm{x}_t \in \mathcal{X}$, 
whereas $\bm{w}_t \in \Omega$ will be given as a realization of the random variable,
which follows the distribution $p(\bm{w})$.
Next, the user gets the
noisy observation $y_t = f(\bm{x}_t, \bm{w}_t) + \eta_t$,
where $\eta_t$ is independent Gaussian noise following $\mathcal{N}(0, \sigma^2)$.

Furthermore, as a regularity assumption, we assume that $f$ is an element of reproducing
kernel Hilbert space (RKHS) and has a bounded norm, which is also assumed in the standard
BO literature \cite{DBLP:conf/icml/SrinivasKKS10}.
Let $k$ be a positive definite kernel over $(\mathcal{X} \times \Omega) \times (\mathcal{X} \times \Omega)$
and $\mathcal{H}_k$ be an RKHS corresponding to $k$. In this paper, for some positive
constant $B$, we assume $f \in \mathcal{H}_k$ with $\|f\|_{\mathcal{H}_k} \leq B$,
where $\|\cdot\|_{\mathcal{H}_k}$ denotes the Hilbert norm defined on $\mathcal{H}_k$.

\paragraph{Models}
Our algorithm uses the GP method~\cite{gpml} to
navigate the optimization process.
First, we assume $\mathcal{GP}(0, k)$ as a prior of $f$, where $\mathcal{GP}(\mu, k)$
is a GP that is characterized by a mean function $\mu$ and a kernel function $k$.
Given the sequence of data $\{\left((\bm{x}_i, \bm{w}_i),~y_i\right)\}_{i=1}^t$, the posterior
distribution of $f(\bm{x}, \bm{w})$ is the Gaussian distribution that
has mean $\mu_t(\bm{x}, \bm{w})$ and variance $\sigma_t^2(\bm{x}, \bm{w})$ defined as follows:
\begin{align*}
    \mu_t(\bm{x}, \bm{w}) &= \bm{k}_t(\bm{x}, \bm{w})^{\top}
    \left( \bm{K}_t + \sigma^2 \bm{I}_t \right)^{-1} \bm{y}_t, \\
    \sigma_t^2(\bm{x}, \bm{w}) &= k\left( (\bm{x}, \bm{w}), (\bm{x}, \bm{w}) \right)
    - \bm{k}_t(\bm{x}, \bm{w})^{\top} \left( \bm{K}_t + \sigma^2 \bm{I}_t \right)^{-1} \bm{k}_t(\bm{x}, \bm{w}),
\end{align*}
where $\bm{k}_t(\bm{x}, \bm{w}) =
\left( k((\bm{x}, \bm{w}), (\bm{x}_1, \bm{w}_1)),
\ldots, k((\bm{x}, \bm{w}), (\bm{x}_t, \bm{w}_t))\right)^{\top}$,
$\bm{y}_t = (y_1, \ldots, y_t)$,
$\bm{I}_t$ is the identity matrix of size $t$, and $\bm{K}_t$ is the $t \times t$ kernel matrix whose
$(i, j)$th element is $k((\bm{x}_i, \bm{w}_i), (\bm{x}_j, \bm{w}_j))$.

We will make use of the following lemma,
 to construct the confidence bound of $f$ by using
the posterior mean $\mu_t$ and the variance $\sigma_t^2$.
\begin{lemma}[Theorem 3.11 in \cite{abbasi2013online}]
    \label{lem:f_cred}
    Fix $f \in \mathcal{H}_k$ with $\|f\|_{\mathcal{H}_k} \leq B$.
    Given $\delta \in (0, 1)$, let define 
    \begin{equation*}
        \beta_t = \left( \sqrt{\ln \det (\bm{I}_t + \sigma^{-2} \bm{K}_t) + 2 \ln
    (1/\delta)} + B \right)^2.
    \end{equation*}
    Then, the following holds with probability at least $1 - \delta$:    
    \begin{equation}
        \label{eq:f_cred}
        | f(\bm{x}, \bm{w}) - \mu_{t-1}(\bm{x}, \bm{w}) |
        \leq \beta_t^{1/2} \sigma_{t-1}(\bm{x}, \bm{w}),
        ~\forall \bm{x} \in \mathcal{X},~\forall \bm{w} \in \Omega,~\forall t \geq 1.
    \end{equation}
\end{lemma}

Based on the above lemma, the confidence bound $Q_t(\bm{x}, \bm{w}) \coloneqq [l_t(\bm{x}, \bm{w}), u_t(\bm{x}, \bm{w})]$
of $f(\bm{x}, \bm{w})$ can be computed by
\begin{align*}
    l_t(\bm{x}, \bm{w}) &= \mu_{t-1}(\bm{x}, \bm{w})
    - \beta_t^{1/2} \sigma_{t-1}(\bm{x}, \bm{w}), \\
    u_t(\bm{x}, \bm{w}) &= \mu_{t-1}(\bm{x}, \bm{w})
    + \beta_t^{1/2} \sigma_{t-1}(\bm{x}, \bm{w}).
\end{align*}

\subsection{Objective Functions and Optimization Goal}\label{sec:opt_goal}
Here, we consider the expectation and variance of $f(\bm{x}, \bm{w})$
under the uncertainty of $p(\bm{w})$ as follows:
\begin{align}
    \mathbb{E}_{\bm{w}}\left[ f(\bm{x}, \bm{w}) \right]
    &= \int_{\Omega} f(\bm{x}, \bm{w}) p(\bm{w})
    \text{d}\bm{w}, \\
    \mathbb{V}_{\bm{w}}\left[ f(\bm{x}, \bm{w}) \right]
    &= \int_{\Omega} \left\{f(\bm{x}, \bm{w}) -
    \mathbb{E}_{\bm{w}}\left[ f(\bm{x}, \bm{w}) \right] \right\}^2
    p(\bm{w}) \text{d}\bm{w}.
\end{align}
Using these $\mathbb{E}_{\bm{w}}\left[ f(\bm{x}, \bm{w}) \right]$ and
$\mathbb{V}_{\bm{w}}\left[ f(\bm{x}, \bm{w}) \right]$,
we define the objective functions $F_1$ and $F_2$ as follows:
\begin{align}
    \label{eq:obj}
    F_1(\bm{x}) = \mathbb{E}_{\bm{w}}\left[ f(\bm{x}, \bm{w}) \right],
    ~F_2(\bm{x}) = -\sqrt{\mathbb{V}_{\bm{w}}\left[ f(\bm{x}, \bm{w}) \right]}.
\end{align}
Our goal is to maximize $F_1$ and $F_2$ simultaneously with
as few function evaluations as possible.
To this end, we handle these objective functions in 
multi-task and multi-objective optimization frameworks.
 \footnote{
In appendix \ref{sec:suppB}, as another formulation,
we also consider the constrained optimization problem whose
objective and constraint functions are $F_1$ and $F_2$ respectively.}

\paragraph{Multi-task Optimization Scenario}
First, we formulate the problem
as a single-objective optimization problem
whose objective function is defined as a weighted sum of $F_1$
and $F_2$. Given a user-specified weight $\alpha \in [0, 1]$,
let $G$ be a new objective function defined as follows:
\begin{equation*}
    G(\bm{x}) = \alpha F_1(\bm{x}) + (1 - \alpha) F_2(\bm{x}).
\end{equation*}
In this formulation, our goal is to find
$\bm{x}^\ast \coloneqq {\rm argmax}_{\bm{x} \in \mathcal{X}} G(\bm{x})$
efficiently.
To rigorously determine the theoretical properties,
we introduce the notion of an {\it $\epsilon$-accurate solution}.
Let $\hat{\bm{x}}_t$ be an estimated solution which
is defined by the algorithm at step $t$. Given a
fixed constant $\epsilon \geq 0$,
we say that $\hat{\bm{x}}_t$
is {\it $\epsilon$-accurate} if the following inequality holds:
\begin{equation*}
    G(\hat{\bm{x}}_t) \geq G(\bm{x}^\ast) - \epsilon.
\end{equation*}
In section \ref{sec:theory}, for an arbitrarily small $\epsilon$,
we show that our algorithm can find
the $\epsilon$-accurate solution with high probability after finite step $T$.

\paragraph{Multi-objective Optimization Scenario}
In the multi-task scenario, we assume that the user can specify the weight
$\alpha$ before the optimization; however this is sometimes unrealistic.
We also consider the more general formulation based on
the Pareto optimality criterion.
Hereafter, we use the vector representation of the objective functions like $\bm{F}(\bm{x}) = (F_1(\bm{x}), F_2(\bm{x}))$.
First, let $\preceq$ be a relational operator defined
over $\mathcal{X} \times \mathcal{X}$ or $\mathbb{R}^2 \times \mathbb{R}^2$.
Given $\bm{x}, \bm{x}^\prime \in \mathcal{X}$,
we write $\bm{x} \preceq \bm{x}^\prime$ or $\bm{F}(\bm{x}) \preceq \bm{F}(\bm{x}^\prime)$ provided that
$F_1(\bm{x}) \leq F_1(\bm{x}^\prime)$ and $F_2(\bm{x}) \leq F_2(\bm{x}^\prime)$ hold simultaneously.
We say that $\bm{x}^\prime$ dominates $\bm{x}$ if $\bm{x} \preceq \bm{x}^\prime$.
Furthermore, we write $\bm{x} \prec \bm{x}^\prime$ or $\bm{F}(\bm{x}) \prec \bm{F}(\bm{x}^\prime)$
provided that either $F_1(\bm{x}) < F_1(\bm{x}^\prime)$
or $F_2(\bm{x}) < F_2(\bm{x}^\prime)$ holds.

The goal of this scenario is to identify the following {\it Pareto set} $\Pi$
efficiently:
\begin{align*}
    \Pi = \{ \bm{x} \in \mathcal{X} \mid \forall \bm{x}^\prime \in E_{\bm{x}}, 
    \bm{F}(\bm{x}) \npreceq \bm{F}(\bm{x}^\prime)\}, 
    ~\text{where}~E_{\bm{x}} = \left\{\bm{x}^\prime \in \mathcal{X} 
    ~\middle|~ \bm{F}(\bm{x}) \neq \bm{F}(\bm{x}^\prime) \right\}.
\end{align*}
Moreover, {\it Pareto front} $Z$ is defined by
\begin{align*}
    Z = \partial\{ \bm{y} \in \mathbb{R}^2 \mid \exists \bm{x} \in \mathcal{X}, \bm{y} \preceq \bm{F}(\bm{x})\}.
\end{align*}

Next, we introduce the notion of an $\epsilon$-accurate Pareto set \cite{JMLR:v17:15-047},
which is an idea similar to the $\epsilon$-accurate solution in the
multi-task scenario.
Given a non-negative vector $\bm{\epsilon} = (\epsilon_1, \epsilon_2)$,
we define the relational operator $\preceq_{\bm{\epsilon}}$, which is the relaxed version
of $\preceq$. For $\bm{x}, \bm{x}^\prime \in \mathcal{X}$, we write
$\bm{x} \preceq_{\bm{\epsilon}} \bm{x}^\prime$ or $\bm{F}(\bm{x})
\preceq_{\bm{\epsilon}} \bm{F}(\bm{x}^\prime)$ if
$F_1(\bm{x}) \leq F_1(\bm{x}^\prime) + \epsilon_1$ and
$F_2(\bm{x}) \leq F_2(\bm{x}^\prime) + \epsilon_2$ hold simultaneously.
Then, the {\it $\bm{\epsilon}$-Pareto front} is defined as:
\begin{align*}
    Z_{\bm{\epsilon}} = \{ \bm{y} \in \mathbb{R}^2 \mid \exists \bm{y}^\prime \in Z,~
    \bm{y} \preceq \bm{y}^\prime ~\text{and}~ \exists \bm{y}^{\prime \prime} \in Z,~
    \bm{y}^{\prime \prime} \preceq_{\bm{\epsilon}} \bm{y}\}.
\end{align*}
We say that estimated Pareto set $\hat{\Pi}_t$ of the algorithm is an $\bm{\epsilon}$-accurate Pareto set
if the following two conditions are satisfied:
\begin{enumerate}
    \item $\bm{F}(\hat{\Pi}_t) \subset Z_{\bm{\epsilon}}$, where $\bm{F}(\hat{\Pi}_t) \coloneqq \left\{ \bm{F}(\bm{x})
    \mid \bm{x} \in \hat{\Pi}_t \right\}$.
    \item For any $\bm{x} \in \Pi$, there is at least one point $\bm{x}^\prime \in \hat{\Pi}_t$ such that
    $\bm{x} \preceq_{\bm{\epsilon}} \bm{x}^\prime$.
\end{enumerate}
Intuitively, condition $1$ guarantees that the estimated
solutions are worse than the true Pareto front by at most $\bm{\epsilon}$.
Condition 2 indicates that $\hat{\Pi}$ can cover
all points in the true Pareto set $\Pi$.

We emphasize that although many studies about
multi-task or multi-objective optimization based on a GP have been reported,
their methods cannot be directly applied to our setting because
the objective functions $F_1$ and $F_2$ are not observed directly.

%% file: section3.tex
\section{Proposed Method}
First, we explain the basic idea of our proposed algorithms.
To maximize $F_1$ and $F_2$ efficiently,
one simple way is to consider the
predicted distributions of $F_1$ and $F_2$,
and apply existing methods
(e.g. expected improvement, entropy search).
However, it is difficult to handle the predicted distribution of $F_2$
although $f$ is modeled by a GP.
In this paper, we first derive the intervals
in which $F_1$ and $F_2$ exist with high probability
from the confidence bound of $f$, and
construct the algorithm based on these derived intervals.
Hereafter, with a slight abuse of notation, we refer to
these derived intervals as the confidence bounds of $F_1$ and $F_2$.

\subsection{Confidence Bounds of Objective Functions}\label{sec:cred}
First, we consider the confidence bound
$Q_t^{(F_1)}(\bm{x}) = [l_t^{(F_1)}(\bm{x}), u_t^{(F_1)}(\bm{x})]$
of $F_1(\bm{x})$.
When (\ref{eq:f_cred}) holds, the following inequity holds for
any $\bm{x} \in \mathcal{X}$, $t \geq 1$:
\begin{equation*}
    \int_{\Omega} l_t(\bm{x}, \bm{w}) p(\bm{w}) \text{d}\bm{w} \leq
    \int_{\Omega} f(\bm{x}, \bm{w}) p(\bm{w}) \text{d}\bm{w} \leq
    \int_{\Omega} u_t(\bm{x}, \bm{w}) p(\bm{w}) \text{d}\bm{w}.
\end{equation*}
This implies that $F_1(\bm{x}) \in Q_t^{(F_1)}(\bm{x})$ for any $\bm{x} \in \mathcal{X}$, $t \geq 1$
with probability at least $1-\delta$ for $l_t^{(F_1)}$ and $u_t^{(F_1)}$
defined as
\begin{equation*}
    l_t^{(F_1)}(\bm{x}) = \int_{\Omega} l_t(\bm{x}, \bm{w}) p(\bm{x}) \text{d}\bm{w},~
    u_t^{(F_1)}(\bm{x}) = \int_{\Omega} u_t(\bm{x}, \bm{w}) p(\bm{x}) \text{d}\bm{w}.
\end{equation*}

We construct the confidence bound
$Q_t^{(F_2)}(\bm{x}) = [l_t^{(F_2)}(\bm{x}), u_t^{(F_2)}(\bm{x})]$
of $F_2(\bm{x})$ in a similar way.
First, we consider the quantity $f(\bm{x}, \bm{w}) - \mathbb{E}_{\bm{w}}[f(\bm{x}, \bm{w})]$,
which appears in the integrand of $\mathbb{V}_{\bm{w}}[f(\bm{x}, \bm{w})]$.
Under condition (\ref{eq:f_cred}), the following inequity holds:
\begin{equation}
    \label{eq:V_int}
    \tilde{l}_t(\bm{x}, \bm{w}) \leq
     f(\bm{x}, \bm{w}) - \mathbb{E}_{\bm{w}}[f(\bm{x}, \bm{w})] \leq
     \tilde{u}_t(\bm{x}, \bm{w}),
\end{equation}
where $\tilde{l}_t(\bm{x}, \bm{w}) = l_t(\bm{x}, \bm{w}) - \mathbb{E}_{\bm{w}}[u_t(\bm{x}, \bm{w})]$ and
$\tilde{u}_t(\bm{x}, \bm{w}) = u_t(\bm{x}, \bm{w}) - \mathbb{E}_{\bm{w}}[l_t(\bm{x}, \bm{w})]$.
Next, the integrand of $\mathbb{V}_{\bm{w}}[f(\bm{x}, \bm{w})]$
can be evaluated based on (\ref{eq:V_int}) as follows:
\begin{equation*}
    \tilde{l}_t^{(\text{sq})}(\bm{x}, \bm{w}) \leq
    \left\{f(\bm{x}, \bm{w}) - \mathbb{E}_{\bm{w}}[f(\bm{x}, \bm{w})]\right\}^2
    \leq \tilde{u}_t^{(\text{sq})}(\bm{x}, \bm{w}),
\end{equation*}
where
\begin{align*}
    \tilde{l}_t^{(\text{sq})}(\bm{x}, \bm{w}) &= \begin{cases}
        0 & \text{if}~\tilde{l}_t(\bm{x}, \bm{w}) \leq 0 \leq \tilde{u}_t(\bm{x}, \bm{w}), \\
        \min\left\{\tilde{l}_t^2(\bm{x}, \bm{w}), \tilde{u}_t^2(\bm{x}, \bm{w})\right\}
        & \text{otherwise}
    \end{cases}, \\
    \tilde{u}_t^{(\text{sq})}(\bm{x}, \bm{w}) &= \max\left\{\tilde{l}_t^2(\bm{x}, \bm{w}), \tilde{u}_t^2(\bm{x}, \bm{w})\right\}.
\end{align*}
Finally, from the monotonicity of square root,
the confidence bound $Q_t^{(F_2)}(\bm{x}) = [l_t^{(F_2)}(\bm{x}),u_t^{(F_2)}(\bm{x})]$
of $F_2(\bm{x})$ is computed using the following 
equations for $l_t^{(F_2)}$ and $u_t^{(F_2)}$:
\begin{equation*}
    l_t^{(F_2)}(\bm{x}) = -\sqrt{\int_{\Omega} \tilde{u}_t^{(\text{sq})}(\bm{x}, \bm{w}) p(\bm{x}) \text{d}\bm{w}},
    ~u_t^{(F_2)}(\bm{x}) = -\sqrt{\int_{\Omega} \tilde{l}_t^{(\text{sq})}(\bm{x}, \bm{w}) p(\bm{x}) \text{d}\bm{w}}.
\end{equation*}
\subsection{Algorithms}\label{sec:alg}
\paragraph{Multi-task Scenario}
In the multi-task scenario, our algorithm chooses
the next observation point $\bm{x}_t$ based on the
upper confidence bound (UCB) of the function $G$.
From $Q_t^{(F_1)}(\bm{x})$ and $Q_t^{(F_2)}(\bm{x})$,
the confidence bound $Q_t^{(G)}(\bm{x}) \coloneqq [l_t^{(G)}(\bm{x}), u_t^{(G)}(\bm{x})]$ of $G(\bm{x})$
can be constructed by defining
\begin{equation*}
    l_t^{(G)}(\bm{x}) = \alpha l_t^{(F_1)}(\bm{x}) + (1 - \alpha) l_t^{(F_2)}(\bm{x}),~
     u_t^{(G)}(\bm{x}) = \alpha u_t^{(F_1)}(\bm{x}) + (1 - \alpha) u_t^{(F_2)}(\bm{x}).
\end{equation*}
At every step $t$, the next observation point $\bm{x}_t$ of our algorithm is
defined by $\bm{x}_t = {\rm argmax}_{\bm{x} \in \mathcal{X}} u_t^{(G)}(\bm{x})$. 
Hereafter, we call this strategy Multi-Task (MT)-MVA-BO.
The pseudo-code of MT-MVA-BO is shown as Algorithm \ref{alg:scal}.
\paragraph{Multi-objective Scenario}
Next, we explain the proposed algorithm for
finding the Pareto set efficiently.
From the confidence bounds of $F_1$ and $F_2$,
we define $\bm{F}_t^{(\text{opt})}$ and
$\bm{F}_t^{(\text{pes})}$ by
$\bm{F}_{t}^{(\text{opt})}(\bm{x}) =
\left(u_t^{(F_1)}(\bm{x}), u_t^{(F_2)}(\bm{x})\right)$ and
$\bm{F}_{t}^{(\text{pes})}(\bm{x}) =
\left(l_t^{(F_1)}(\bm{x}), l_t^{(F_2)}(\bm{x})\right)$,
which respectively represent the optimistic and pessimistic predictions
of the objective functions at step $t$.
First, we define the estimated Pareto set $\hat{\Pi}_t$ at
step $t$ by
\begin{equation}
    \label{eq:est_par_set}
    \hat{\Pi}_t = \left\{ \bm{x} \in \mathcal{X} ~\middle|~
    \forall \bm{x}^\prime \in E_{t,\bm{x}}^{(\text{pes})},~
    \bm{F}_{t}^{(\text{pes})}(\bm{x}) \npreceq \bm{F}_{t}^{(\text{pes})}(\bm{x}^\prime)\right\},
    ~\text{where}~E_{t,\bm{x}}^{(\text{pes})} 
    = \left\{\bm{x}^\prime \in \mathcal{X} ~\middle|~ \bm{F}_t^{(\text{pes})}(\bm{x}) 
    \neq \bm{F}_t^{(\text{pes})}(\bm{x}^\prime)\right\}.
\end{equation}
For theoretical reasons, we define $\hat{\Pi}_t$ based on
pessimistic predictions and the same idea is used in the existing GP-based optimization literatures
\cite{sui2015safe, JMLR:v17:15-047, bogunovic2018adversarially, DBLP:conf/aistats/KirschnerBJ020}.
Furthermore, using $\hat{\Pi}_t$,
the potential Pareto set $M_t$ is defined by
\begin{equation*}
    M_t = \left\{\bm{x} \in \mathcal{X} \setminus \hat{\Pi}_t ~\middle|~
    \forall \bm{x}^\prime \in \hat{\Pi}_t,~
    \bm{F}_{t}^{(\text{opt})}(\bm{x}) \npreceq_{\bm{\epsilon}} \bm{F}_{t}^{(\text{pes})}(\bm{x}^\prime)\right\}.
\end{equation*}
An intuitive interpretation of $M_t$ is the set which
excludes the points that are $\bm{\epsilon}$-dominated by other points with high probability.
At every step $t$, our algorithm chooses $\bm{x}_t$ based on the
uncertainty defined by the confidence bounds of $F_1$ and $F_2$.
In this paper, we adopt the diameter $\lambda_t(\bm{x})$ of rectangle
$\text{Rect}_t(\bm{x}) = \left[l_t^{(F_1)}(\bm{x}), u_t^{(F_1)}(\bm{x})\right] \times
\left[l_t^{(F_2)}(\bm{x}), u_t^{(F_2)}(\bm{x})\right]$ as
the uncertainty of $\bm{x}$:
\begin{equation}
    \label{eq:lambda}
    \lambda_t(\bm{x}) = \max_{\bm{y}, \bm{y}^\prime \in \text{Rect}_t(\bm{x})}
    \|\bm{y} - \bm{y}^\prime\|_{2}.
\end{equation}
Namely, the next observation point $\bm{x}_t$ is defined by $\bm{x}_t =
{\rm argmax}_{\bm{x} \in M_t \cup \hat{\Pi}_t}
 \lambda_t(\bm{x})$ at every step $t$.

Our proposed algorithm terminates when estimated Pareto set $\hat{\Pi}_t$
is guaranteed to be an $\bm{\epsilon}$-Pareto set  with high probability.
To this end, our algorithm checks the uncertainty set
$U_t$ that is defined by
\begin{equation*}
    U_t = \left\{\bm{x} \in \hat{\Pi}_t ~\middle|~ \exists \bm{x}^\prime \in
    \hat{\Pi}_t \setminus \{\bm{x}\},
    \bm{F}_{t}^{(\text{pes})}(\bm{x}) + \bm{\epsilon} \prec
    \bm{F}_{t}^{(\text{opt})}(\bm{x}^\prime) \right\}.
\end{equation*}
Intuitively, $U_t$ is the set of points
where it is not possible to decide
whether it is an $\bm{\epsilon}$-Pareto solution
based on the current
confidence bounds.
Our algorithm terminates at a step $t$ where both $M_t = \emptyset$ and
$U_t = \emptyset$ hold.

Hereafter, we call this algorithm Multi-Objective (MO)-MVA-BO.
The pseudo-code of MO-MVA-BO
 is shown as Algorithm \ref{alg:pareto}.

\begin{table}[tb]
    \begin{multicols*}{2}
        \begin{algorithm}[H]
            \caption{Multi-task MVA-BO (MT-MVA-BO)}
            \label{alg:scal}
            \begin{algorithmic}
                \REQUIRE GP prior $\mathcal{GP}(0,\ k)$,
                ~$\{\beta_t\}_{t \leq T}$,~$\alpha \in (0, 1)$.
                \FOR {$t = 0$ to $T$}
                    \STATE Compute $u_t^{(G)}(\bm{x})$ for any $\bm{x} \in \mathcal{X}$
                    \STATE Choose $\bm{x}_t = {\rm argmax}_{\bm{x} \in \mathcal{X}} u_t^{(G)}(\bm{x})$.
                    \STATE Sample $\bm{w}_t \sim p(\bm{w})$.
                    \STATE Observe $y_t \leftarrow f(\bm{x}_t, \bm{w}_t) + \eta_t$
                    \STATE Update the GP by adding $((\bm{x}_t, \bm{w}_t), y_t)$.
                \ENDFOR
                \ENSURE ${\rm argmax}_{\bm{x} \in \{\bm{x}_1, \ldots, \bm{x}_T\}} l_T^{(G)}(\bm{x})$.
            \end{algorithmic}
        \end{algorithm}
        \columnbreak
        \begin{algorithm}[H]
            \caption{Multi-objective MVA-BO (MO-MVA-BO)}
            \label{alg:pareto}
            \begin{algorithmic}
                \REQUIRE GP prior $\mathcal{GP}(0,\ k)$,
                ~$\{\beta_t\}_{t \in \mathbb{N}}$,~Non-negative vector $\bm{\epsilon} = (\epsilon_1, \epsilon_2)$.
                \STATE $t \leftarrow 0$.
                \REPEAT
                    \STATE Compute $\hat{\Pi}_t, M_t$.
                    \STATE Compute $\lambda_t(\bm{x})$ for any $\bm{x} \in M_t \cup \hat{\Pi}_t$.
                    \STATE Choose $\bm{x}_t = {\rm argmax}_{\bm{x} \in M_t \cup \hat{\Pi}_t}
                    \lambda_t(\bm{x})$.
                    \STATE Sample $\bm{w}_t \sim p(\bm{w})$.
                    \STATE Observe $y_t \leftarrow f(\bm{x}_t, \bm{w}_t) + \eta_t$.
                    \STATE Update the GP by adding $((\bm{x}_t, \bm{w}_t), y_t)$.
                    \STATE $t \leftarrow t + 1$.
                    \STATE Compute $U_t$.
                \UNTIL {$M_t = \emptyset$ and $U_t = \emptyset$}
                \ENSURE $\hat{\Pi}_t$.
            \end{algorithmic}
        \end{algorithm}
    \end{multicols*}
\end{table}

\subsection{Extensions and Practical Considerations}\label{sec:ext_setting}
In this section, we consider several extensions of
the proposed method to deal with situations which
arise in some practical applications,
leaving the details for appendix \ref{sec:suppC}.
\subsubsection{Unknown Distribution}
Thus far, we have assumed that $p(\bm{w})$ is known; however, this assumption
is sometimes unrealistic.
Considering how to deal with the case where $p(\bm{w})$ is unknown, one simple
way is to estimate $p(\bm{w})$ during the optimization process.
For example, if we estimate $p(\bm{w})$ by using an empirical distribution,
we can apply our algorithm by replacing $p(\bm{w})$
with the following $\tilde{p}_t(\bm{w})$ when computing the confidence bounds:
\begin{equation*}
    \tilde{p}_t(\bm{w}) = \frac{1}{t} \sum_{t^\prime=1}^t \1 [\bm{w}_{t^\prime} = \bm{w}].
\end{equation*}
As a more advanced method, it may be possible to consider 
extension to the distributionally robust setting \cite{DBLP:conf/aistats/KirschnerBJ020, pmlr-v108-nguyen20a}; however, we leave this as future work.

\subsubsection{Extension to Noisy Input}\label{sec:noise_ext}
One setting similar to that in this paper is the noisy input
 setting \cite{beland2017bayesian, pmlr-v108-frohlich20a}.
In this setting, observation point $\bm{x}_t$ is fluctuated
by noise $\bm{\xi} \in \Delta$ which follows the known density $p(\bm{\xi})$
defined over $\Delta$.
At every step $t$, the user chooses $\bm{x}_t$ and
obtains observation $y_t$ as $y_t = f(\bm{x}_t + \bm{\xi}) + \eta_t$, $\bm{\xi} \sim p(\bm{\xi})$. Our problem can be extended to the noisy input setting
by defining $F_1$ and $F_2$ through expectation $\mathbb{E}_{\bm{\xi}}[f(\bm{x} + \bm{\xi})]$ and variance $\mathbb{V}_{\bm{\xi}}[f(\bm{x} + \bm{\xi})]$
 defined as follows:
\begin{align}
    \label{eq:noisy_e}
    \mathbb{E}_{\bm{\xi}}[f(\bm{x} + \bm{\xi})] &= \int_{\Delta} f(\bm{x} + \bm{\xi}) p(\bm{\xi}) \text{d}\bm{\xi}, \\
    \label{eq:noisy_v}
    \mathbb{V}_{\bm{\xi}}[f(\bm{x} + \bm{\xi})] &= \int_{\Delta}
    \{f(\bm{x} + \bm{\xi}) - \mathbb{E}_{\hat{\bm{\xi}}}[f(\bm{x} +
     \hat{\bm{\xi}})]\}^2 p(\bm{\xi}) \text{d}\bm{\xi}.
\end{align}
We can apply the same algorithms as those in section \ref{sec:alg} by
constructing the confidence bounds via a way similar to that in section \ref{sec:cred}.
\subsubsection{Simulator-Based Experiment}\label{sec:sim_ext}
Some applications can be allowed to control the variable $\bm{w}$
in the optimization. For example, the case that the user run the
optimization process by evaluating
$f(\bm{x}, \bm{w})$ with the computer simulation.
Such scenarios have often been considered in similar 
studies reported in the BO
literature that assumed the existence of an uncontrollable
variable $\bm{w}$ \cite{toscano2018bayesian, DBLP:conf/aistats/KirschnerBJ020,
pmlr-v108-nguyen20a}.
Our method can be extended to such
a scenario by choosing $\bm{w}_t$
according to $\bm{w}_t = {\rm argmax}_{\bm{w} \in \Omega} \sigma_{t-1}(\bm{x}_t, \bm{w})$
after the selection of $\bm{x}_t$.

%% file: section4.tex
\section{Theoretical Results}\label{sec:theory}
In this section, we show the theoretical results of the proposed
algorithms. The details of the proofs are in appendix \ref{sec:suppA}.

First, we introduce the {\it maximum information gain}
\cite{DBLP:conf/icml/SrinivasKKS10} as a
sample complexity parameter of a GP.
Now, Let $A = \{\bm{a}_1, \ldots, \bm{a}_T\}$ be a finite subset of
$\mathcal{X} \times \Omega$, and $\bm{y}_A$ be a vector whose
$i$th element is $y_{\bm{a}_i} = f(\bm{a}_i) + \varepsilon_{\bm{a}_i}$.
Maximum information gain $\gamma_T$ at step $T$ is defined by
\begin{equation*}
    \gamma_T = \max_{A \subset \mathcal{X} \times \Omega; |A|=T} I(\bm{y}_A; f),
\end{equation*}
where $I(\bm{y}_A; f)$ denotes the mutual information between $\bm{y}_A$ and $f$.
Maximum information gain $\gamma_T$ is often used in BO, and its analytical form of the upper bound
is derived in commonly used kernels \cite{DBLP:conf/icml/SrinivasKKS10}.

The following two theorems
show the convergence properties of the proposed algorithms
for the multi-task and multi-objective scenarios, respectively.
\begin{theorem}\label{thm:sca_conv}
    Fix positive definite kernel $k$, and assume
    $f \in \mathcal{H}_k$ with $\|f\|_{\mathcal{H}_k} \leq B$.
    Let $\delta \in (0, 1)$ and $\epsilon > 0$
    , and set $\beta_t$ according to 
    $\beta_t = \left( \sqrt{\ln \det (\bm{I}_t + \sigma^{-2} \bm{K}_t) + 2 \ln \frac{3}{\delta}} + B\right)^2$ at every step $t$.
    Furthermore, for any $t \geq 1$, define $\hat{\bm{x}}_t$ by
    $\hat{\bm{x}}_t = {\rm argmax}_{\bm{x}_{t^\prime} \in \{\bm{x}_1, \ldots, \bm{x}_t\}} l_{t^\prime}^{(G)}(\bm{x}_{t^\prime})$.
    When applying MT-MVA-BO under the above conditions,
    with probability at least $1-\delta$, $\hat{\bm{x}}_T$ is an $\epsilon$-accurate solution, where $T$ is the smallest positive integer
    which satisfies the following inequity:
    \begin{align}
        \label{eq:sca_conv_rate}
        \alpha T^{-1} \beta_T^{1/2}\left( \sqrt{2TC_1\gamma_T} + C_2\right) + (1 - \alpha)T^{-1} \sqrt{2T\tilde{B} \beta_T^{1/2} \left( \sqrt{8TC_1\gamma_T} + 2C_2 \right)
         + 5 T\beta_T \left(C_1\gamma_T + 2C_2 \right)} \leq \epsilon.
    \end{align}
    Here, $\tilde{B} = {\rm max}_{(\bm{x}, \bm{w}) \in (\mathcal{X} \times \Omega)}
    \left| f(\bm{x}, \bm{w}) - \mathbb{E}_{\bm{w}}[f(\bm{x}, \bm{w})]\right|$
    and $C_1=\frac{16}{\log(1 + \sigma^{-2})}, C_2=16\log \frac{18}{\delta}$.
\end{theorem}
\begin{theorem}\label{thm:par_conv}
    Fix positive definite kernel $k$, and assume
    $f \in \mathcal{H}_k$ with $\|f\|_{\mathcal{H}_k} \leq B$.
    Let $\delta \in (0, 1)$ and $\epsilon > 0$,
    and set $\beta_t$ according to
    $\beta_t = \left( \sqrt{\ln \det (\bm{I}_t + \sigma^{-2} \bm{K}_t) + 2 \ln \frac{3}{\delta}} + B\right)^2$ at every step $t$.
    When applying MO-MVA-BO under the above conditions,
    the following 1. and 2. hold with probability at least $1-\delta$:
    \begin{description}
        \item [1.] The algorithm terminates at most step $T$
        where T is the smallest positive integer that
        satisfies the following inequity:
        \begin{align}
            \label{eq:par_conv_rate}
            T^{-1} \beta_T^{1/2}\left( \sqrt{2TC_1\gamma_T} + C_2 \right) + T^{-1} \sqrt{2T\tilde{B} \beta_T^{1/2} \left( \sqrt{8TC_1\gamma_T} + 2C_2 \right)
             + 5 T\beta_T \left(C_1\gamma_T + 2C_2 \right)} \leq \min\{\epsilon_1, \epsilon_2\}.
        \end{align}
        Here, $\tilde{B} = {\rm max}_{(\bm{x}, \bm{w}) \in (\mathcal{X} \times \Omega)}
        \left| f(\bm{x}, \bm{w}) - \mathbb{E}_{\bm{w}}[f(\bm{x}, \bm{w})]\right|$, $C_1=\frac{16}{\log(1 + \sigma^{-2})}, C_2=16\log \frac{18}{\delta}$.
        \item [2.] When the algorithm terminates, estimated Pareto set $\hat{\Pi}_t$
        is an $\bm{\epsilon}$-accurate Pareto set.
    \end{description}
\end{theorem}
The first term $ \beta_T^{1/2} \left( \sqrt{TC_1\gamma_T} + C_2\right)$ of the left hand side in (\ref{eq:sca_conv_rate}) and
(\ref{eq:par_conv_rate}) also appears
in the theoretical result of the existing algorithm,
which only considers the expectation $F_1$ (e.g. Theorem 2 in \cite{DBLP:conf/aistats/KirschnerBJ020}).
The second term $\sqrt{2T\tilde{B} \beta_T^{1/2} \left( \sqrt{8TC_1\gamma_T} + 2C_2 \right)
 + 5 T\beta_T \left(C_1\gamma_T + 2C_2 \right)}$ is specific
 to our problem.
 This term depends on the complexity parameter $\tilde{B}$, which quantifies the variation of 
 function $f(\bm{x}, \bm{w})$ around its expectation.

%% file: section5.tex
\section{Numerical Experiments}
In this section, we show the performance of the proposed methods
through numerical experiments.
As the baseline methods in both the multi-task and multi-objective scenarios, we adopted random
sampling ({\tt RS}) and uncertainty sampling ({\tt US}).
{\tt RS} choose $\bm{x}_t$ from $\mathcal{X}$ uniformly at random,
and {\tt US} choose $\bm{x}_t$ such that $\bm{x}_t$ achieve the largest average
posterior variance $\bm{x}_t = {\rm argmax}_{\bm{x} \in \mathcal{X}} \int_{\Omega} \sigma_{t-1}(\bm{x}, \bm{w}) p(\bm{w}) \text{d}\bm{w}$.
To measure the performance, in the multi-task scenario,
we computed the regret, $G(\bm{x}^\ast) - G(\hat{\bm{x}}_t)$,
at every step $t$, where $\bm{x}_t$ is the estimated solution
defined by the algorithms.
We defined $\hat{\bm{x}}_t$ as $\hat{\bm{x}}_t = {\rm argmax}_{t^\prime=1, \ldots, t}l_t^{(G)}(\bm{x}_{t^\prime})$
in {\tt RS}, {\tt US}, and proposed method ({\tt MT-MVA-BO}).
Furthermore, we set $\alpha=0.5$.
Also, in the multi-objective scenario, we computed the gap of hyper-volume \cite{emmerich2005single},
$\text{HV} - \hat{\text{HV}}_t$ to measure the performance,
where $\text{HV}$ and $\hat{\text{HV}}_t$ denote the hyper
volumes computed based on the true Pareto set $\Pi$
and the estimated Pareto set $\hat{\Pi}_t$, respectively.
The hyper volume gap measures how close the estimated Pareto front
is to the true Pareto front. We defined $\hat{\Pi}_t$ by
(\ref{eq:est_par_set}) in {\tt RS}, {\tt US} and the proposed method
({\tt MO-MVA-BO}).
Furthermore, in the multi-task scenario, to show the effect of
difference of objective functions, we also adopt the
two methods {\tt BQOUCB} \cite{DBLP:conf/aistats/KirschnerBJ020, pmlr-v108-nguyen20a} and {\tt BO-VO}. {\tt BQOUCB}
is the existing method which aims to maximize $F_1$,
and {\tt BO-VO} is the variant of our method which corresponds
to the case $\alpha=0$. These methods choose $\bm{x}_t$ as
the maximizing point of $u_t^{(F_1)}(\bm{x})$ and $u_t^{(F_2)}(\bm{x})$
 respectively. In addition, estimated solution
 $\hat{\bm{x}}_t$ is defined by
 $\hat{\bm{x}}_t = {\rm argmax}_{t^\prime = 1,\ldots,t} l_t^{(F_1)}(\bm{x}_{t^\prime})$ and
 $\hat{\bm{x}}_t = {\rm argmax}_{t^\prime = 1,\ldots,t} l_t^{(F_2)}(\bm{x}_{t^\prime})$ respectively.
Moreover, we also make comparisons to the adaptive versions of 
these methods, {\tt ADA-BQOUCB} and {\tt ADA-BO-VO}.
{\tt ADA-BQOUCB} and {\tt ADA-BO-VO} choose
$\bm{x}_t$ in the same way as do {\tt BQOUCB} and {\tt BO-VO},
but the estimated solutions are defined as $\hat{\bm{x}}_t = {\rm argmax}_{t^\prime=1, \ldots, t}l_t^{(G)}(\bm{x}_{t^\prime})$.

\subsection{Artificial Data Experiments}
In this subsection, we show the results of the artificial-data experiments.
\paragraph{GP Test Functions}
We experimented with the true oracle functions $f$ that
are generated from the 2D GP prior.
First, we divided $[-1, 1]^2$ into $25$ uniformly spaced grid points 
in each dimension and generated the sample path from the 
GP prior. Next, we created the GP model with these grid 
points and set the true oracle function as its GP posterior mean.
In this experiment, we created $50$
sample paths from different seeds, and conducted
$10$ experiments for each function.
Thus, we report the average performance of a total
of $500$ experiments.
To create a GP sample path, we use the Gaussian kernel
$k((\bm{x}, \bm{w}), (\bm{x}^\prime, \bm{w}^\prime))
= \sigma_{\text{ker}}^2\exp(\frac{\|\bm{x} - \bm{x}^\prime\|_2^2 + \|\bm{w} - \bm{w}^\prime\|_2^2}{2l^2})$ with $\sigma_{\text{ker}}=1, l=0.25$,
as well as to construct the confidence bound in the algorithms.
Furthermore, we set noise variance as $\sigma^2=10^{-4}$.
In addition, we divided $[-1, 1]$ into $100$ grid points uniformly, and set $\mathcal{X}$ and $\Omega$ as these grid points.
Moreover, we define $p(w)$ by $p(w) = \sum_{w \in \Omega} \phi(w)/Z$,
$Z = \sum_{w \in \Omega} \phi(w)$ where $\phi$ is the
density function of the standard normal distribution.

\paragraph{Benchmark Functions of Optimization}
We also experimented with the Bird function (2D) and Rosenbrock function (3D),
which are often used as the benchmark function in the field of the optimization.
First, we scaled the input domain to $[-1, 1]$ 
divided into $100$
grid points in each dimension. In Bird function,
we set $\mathcal{X}$ and $\Omega$ as the grid points of the first and the second dimensions, 
respectively. In the Rosenbrock function,
we set $\Omega$ as the grid points of the third dimension and
the remaining points as $\mathcal{X}$.
Furthermore, we set $p(w)$ as in the same way as the experiment of the GP test functions.
We use ARD Gaussian kernel $k((\bm{x}, \bm{w}), (\bm{x}^\prime, \bm{w}^\prime))
= \sigma_{\text{ker}}^2\exp\left\{\sum_{i=1}^{d_1}\frac{(\bm{x}_i-\bm{x}_i^\prime)^2}{2l_i^{(x)2}}
+ \sum_{j=1}^{d_2}\frac{(\bm{w}_j-\bm{w}_j^\prime)^2}{2l_{j}^{(w)2}} \right\}$, and tune these hyperparameters by maximizing
the marginal likelihood at every $10$ step in the algorithms.
Furthermore, we set the noise variance as $\sigma^2=10^{-4}$ and report the
average performance of $100$ simulations with different seeds.

\begin{figure}[t]
    \centering
    \includegraphics[width=0.31\linewidth]{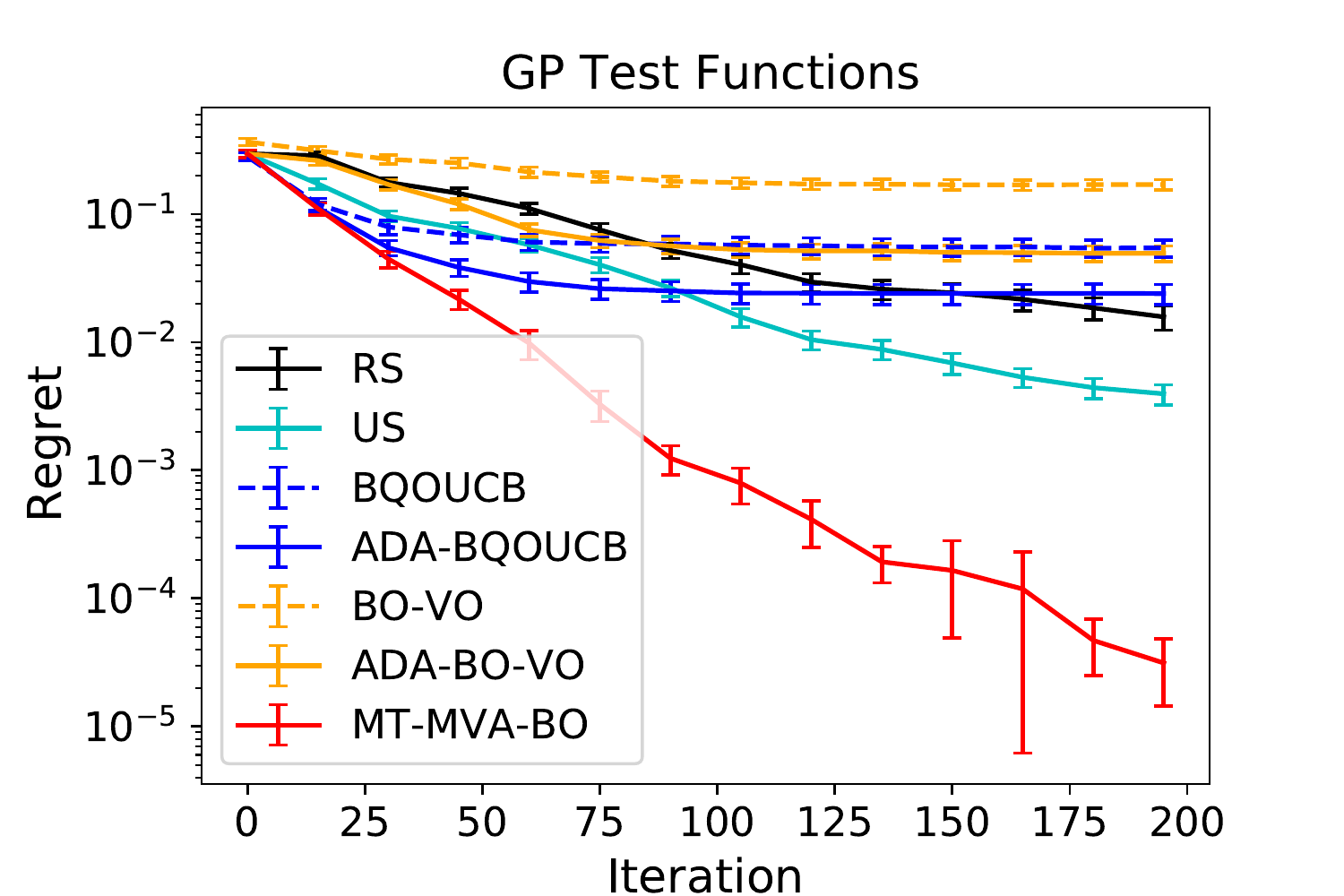}
    \includegraphics[width=0.31\linewidth]{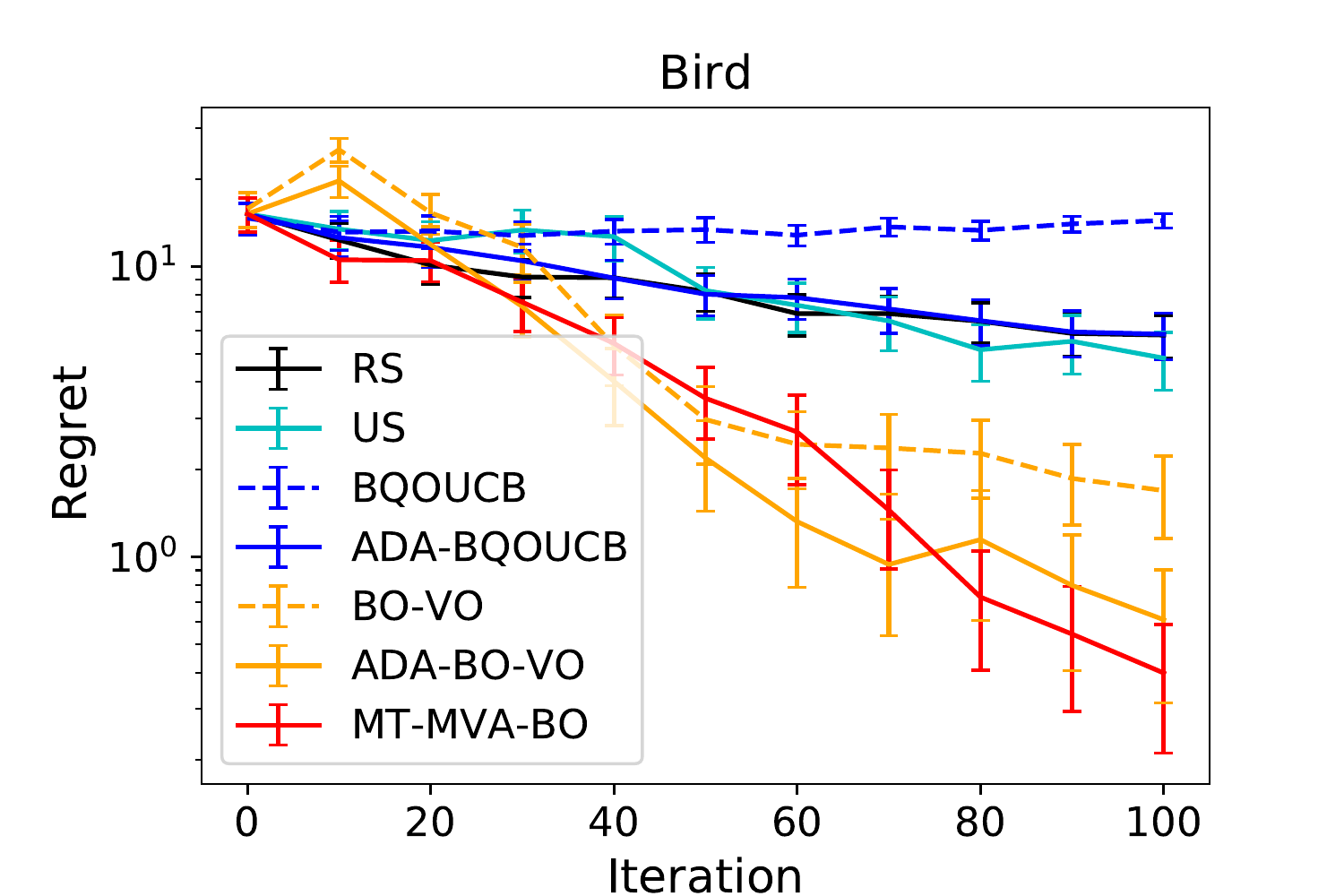}
    \includegraphics[width=0.31\linewidth]{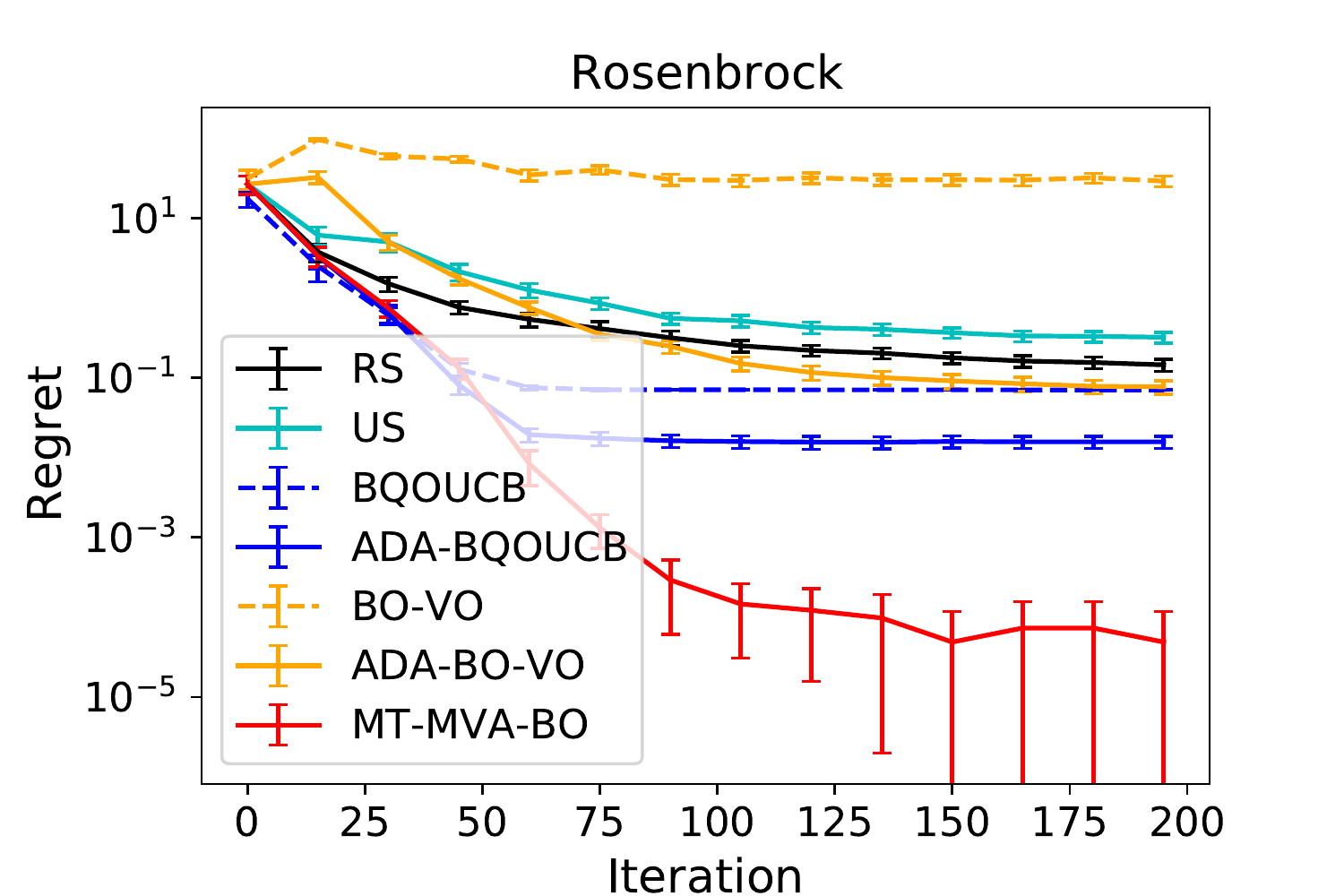}
    \centering
    \includegraphics[width=0.31\linewidth]{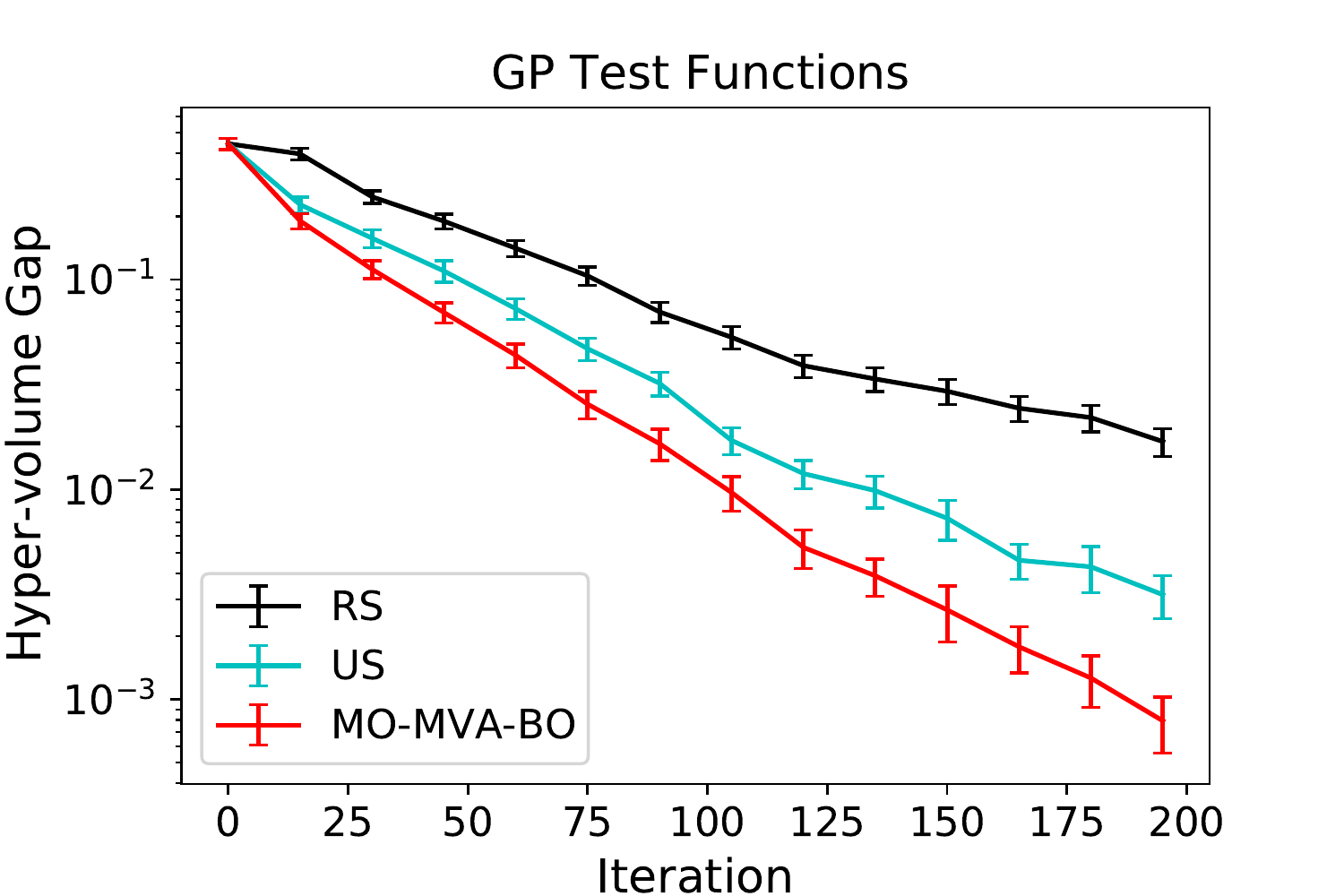}
    \includegraphics[width=0.31\linewidth]{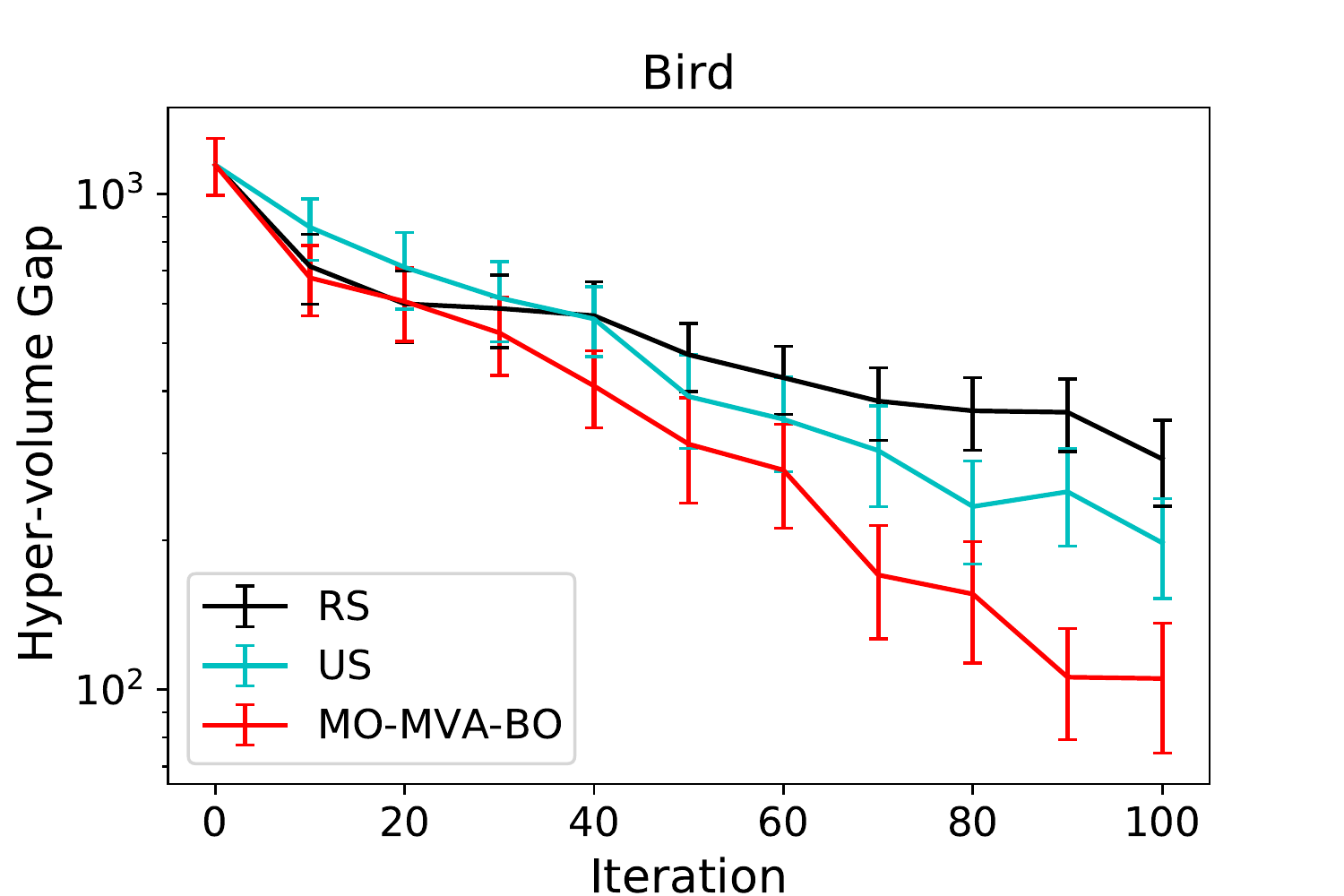}
    \includegraphics[width=0.31\linewidth]{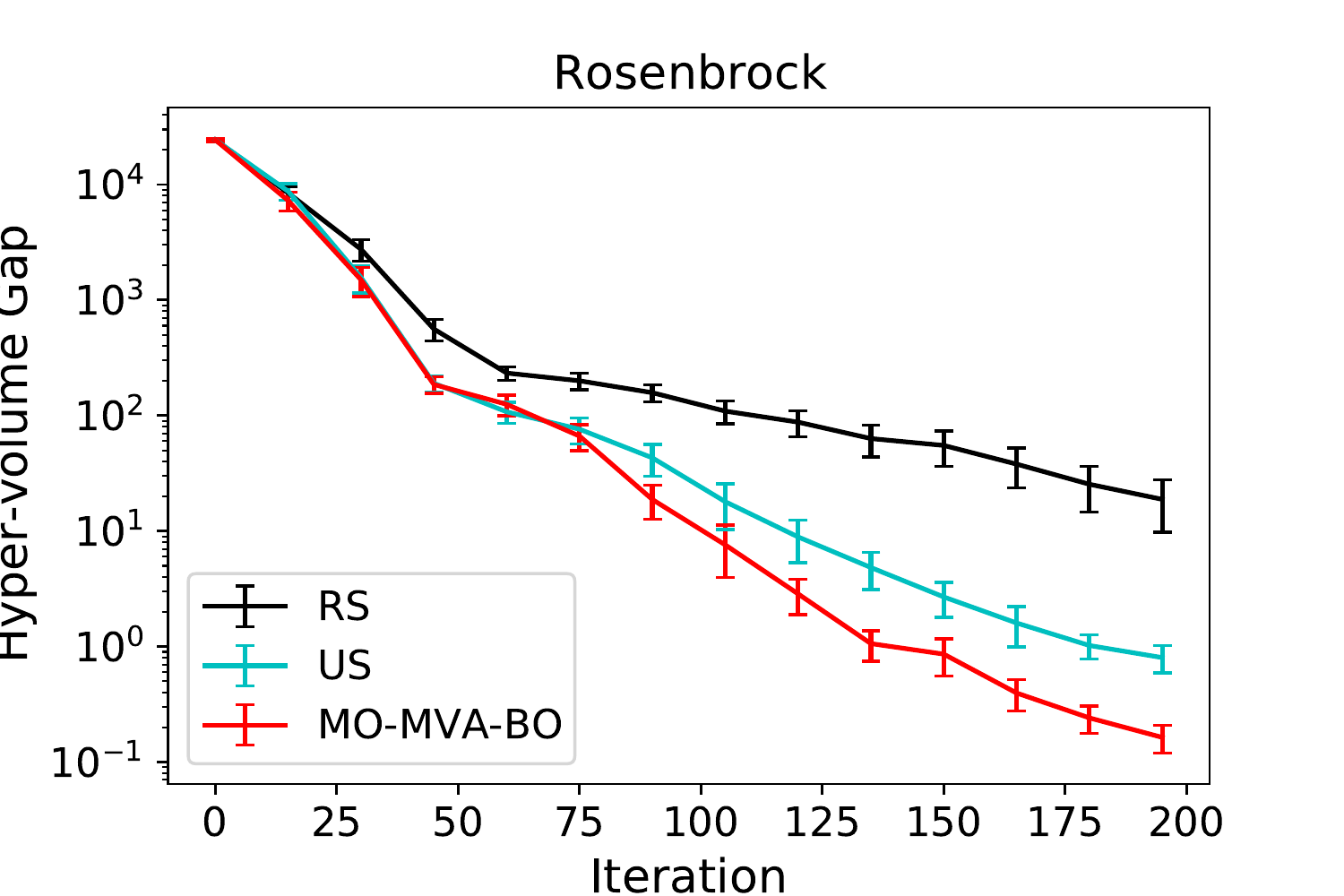}
    \caption{Average performances in artificial data experiments.
    The error bars represent $2 \times $[standard error].
    The top and bottom figures show the results of the multi-task ($\alpha = 0.5$) and multi-objective scenarios respectively.}
    \label{fig:art_exp}
\end{figure}
Figure {\ref{fig:art_exp}} shows the results of the artificial data experiments.
We confirmed that the proposed methods achieve better performances than
the other methods. In the experiments of the multi-task scenario,
we also confirmed that the regrets of {\tt BQOUCB}, {\tt BO-VO}, {\tt ADA-BQOUCB}, and {\tt ADA-BO-VO} stop decreasing at an early stage. Note that
these are reasonable results because objective functions of these methods are inconsistent with our settings.

\subsection{Real-data Experiment}
We applied the proposed methods to {\it Newsvendor problem under dynamic consumer substitution} 
\cite{mahajan2001stocking}, whose goal is to optimize the initial inventory levels
under uncertainty of customer behaviors. The parameter $\bm{x}$ and 
$\bm{w}$ respectively correspond to the initial inventory level of products and the uncertain 
purchasing behaviors of customers, which follow mutually independent Gamma distributions. 
The goal of this problem is to find the $\bm{x}$ which optimizes the 
profit $f(\bm{x}, \bm{w})$ under the uncertainty of $\bm{w}$. For this problem, 
we conducted the experiments  in the simulator-based setting described in section \ref{sec:sim_ext} because 
profit $f(\bm{x}, \bm{w})$ can be evaluated based on a computer simulation.
Figure~\ref{fig:newsvendor} shows the average performances of $100$ simulations with different seeds.
\begin{figure}[t]
    \centering
    \includegraphics[width=0.4\linewidth]{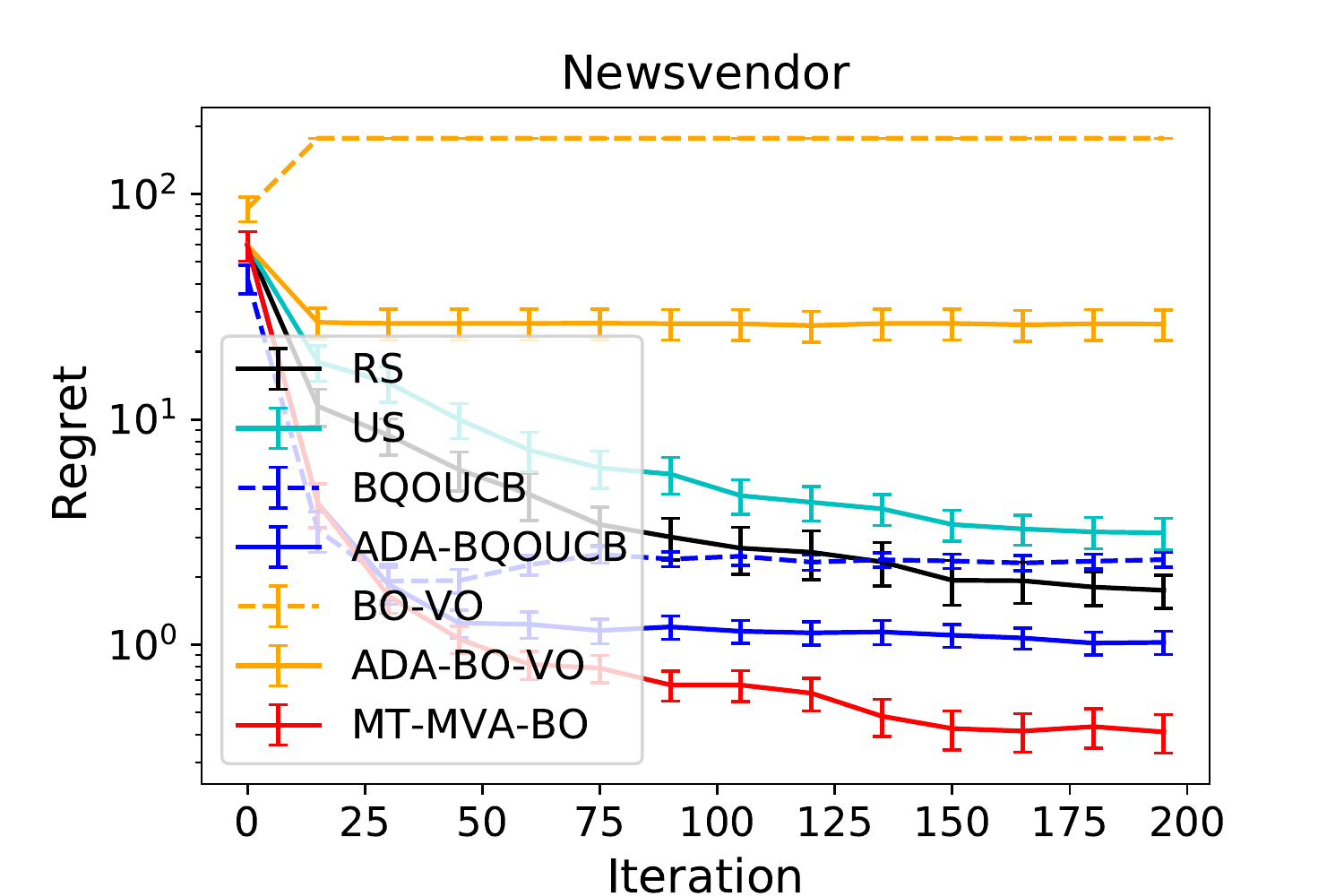}
    \includegraphics[width=0.4\linewidth]{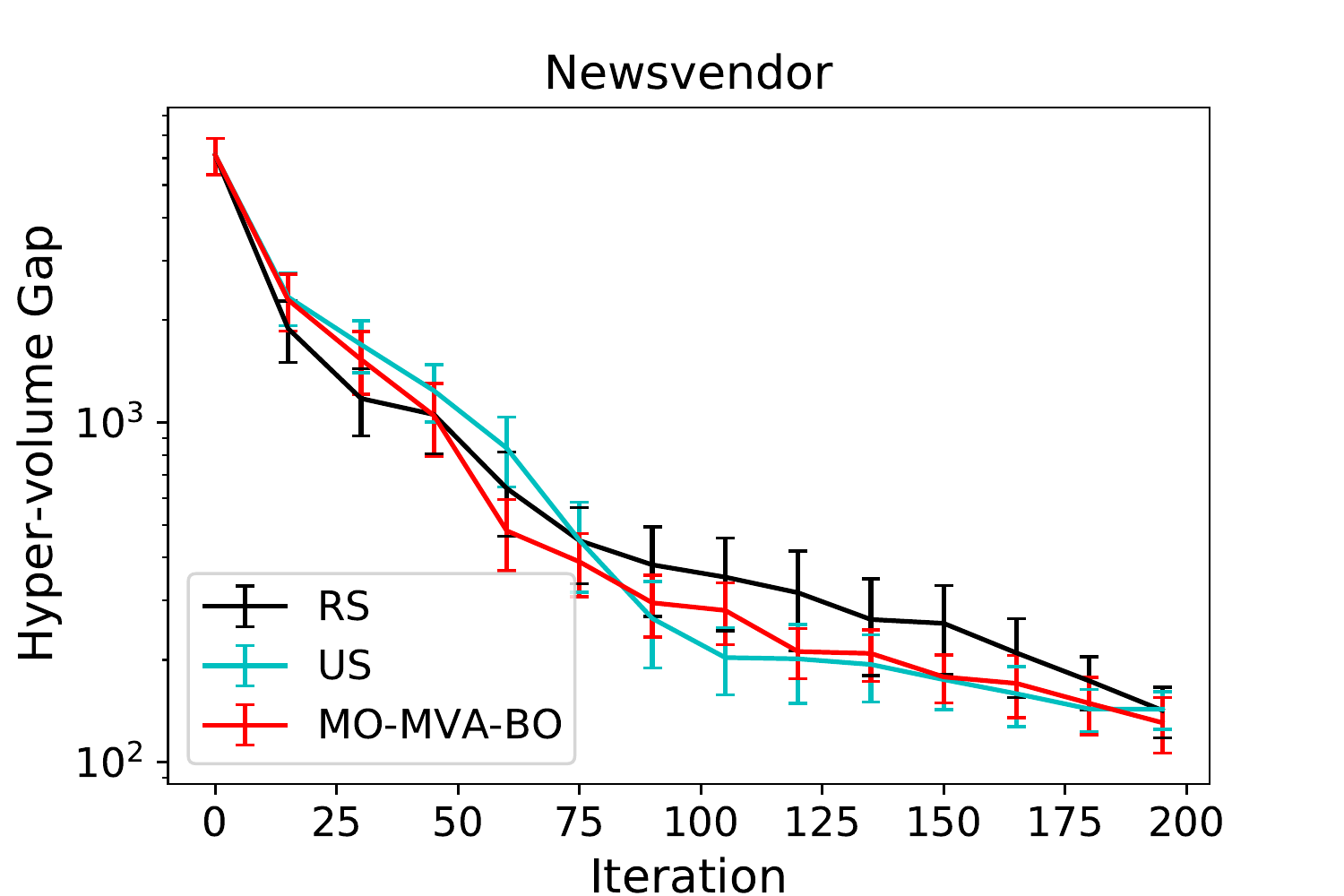}
    \caption{The results of experiments in Newsvendor problem. 
    The left and right figures present the results of the multi-task 
    ($\alpha = 0.5$) and multi-objective scenarios, respectively.}
    \label{fig:newsvendor}
\end{figure}

%% file: section6.tex
\section{Conclusion}
We introduced the novel Bayesian optimization framework: MVA-BO, which 
simultaneously considers two objective functions: expectation and variance 
under an uncertainty environment. 
In this framework, we considered the three scenarios; multi-task, 
multi-objective and constraint optimization scenarios, which 
often appear in real-world applications.
We studied the rigorous convergence properties of our MVA-BO
algorithms and 
demonstrated the effectiveness of them through both artificial and 
real-data experiments.

\section*{Acknowledgement}
This work was partially supported by MEXT KAKENHI (20H00601, 16H06538), JST CREST (JPMJCR1502), and RIKEN Center for Advanced Intelligence Project.

%% file: bib.tex
\bibliography{myref}
\bibliographystyle{unsrt}

%% file: supp_A.tex
\section{Proofs}\label{sec:suppA}
\subsection{Proof of Theorem \ref{thm:sca_conv}}
From the definition of 
$\beta_t$ and Lemma \ref{lem:f_cred}, the following holds with probability at least $1-\delta/3$:
\begin{equation}
    \label{eq:f_cred_cond}
    \forall~\bm{x} \in \mathcal{X},~\forall \bm{w} \in \Omega,~
    \forall t \geq 1,~|f(\bm{x}, \bm{w}) - \mu_{t-1}(\bm{x}, \bm{w})|
    \leq \beta_t^{1/2} \sigma_{t-1}(\bm{x}, \bm{w}).
\end{equation}
Moreover, we give the following lemma about the confidence bound 
$Q_t^{(G)}(\bm{x}_t)$:

\begin{lemma}\label{lem:R_bound}
    Assume that (\ref{eq:f_cred_cond}) holds. Then, for any $T \geq 1$, it holds that 
    \begin{align*}
        &\sum_{t=1}^T\left\{ u_t^{(G)}(\bm{x}_t) - l_t^{(G)}(\bm{x}_t)\right\}
        \leq 2 \alpha \beta_T^{1/2} \sum_{t=1}^T \int_{\Omega} \sigma_{t-1}(\bm{x}_t, \bm{w})
        p(\bm{w}) {\rm d}\bm{w} \\
        &\hspace{20pt}+ (1-\alpha) \sqrt{8T\tilde{B} \beta_T^{1/2}
        \sum_{t=1}^T\int_{\Omega} \sigma_{t-1}(\bm{x}_t, \bm{w}) p(\bm{w}) {\rm d}\bm{w}
         + 20 T\beta_T \sum_{t=1}^T\int_{\Omega} \sigma_{t-1}^2(\bm{x}_t, \bm{w}) p(\bm{w})
          {\rm d}\bm{w}} ,
    \end{align*}
    where $\tilde{B} = {\rm max}_{(\bm{x}, \bm{w}) \in (\mathcal{X} \times \Omega)}
    | f(\bm{x}, \bm{w}) - \mathbb{E}_{\bm{w}}[f(\bm{x}, \bm{w})] |$.
\end{lemma}
\begin{proof}
    From the definition of  $u_t^{(G)}$ and $l_t^{(G)}$, we have
    \begin{equation}
        \label{eq:cred_G_bound}
        \sum_{t=1}^T\left\{ u_t^{(G)}(\bm{x}_t) - l_t^{(G)}(\bm{x}_t) \right\}
        = \alpha \sum_{t=1}^T \left\{ u_t^{(F_1)}(\bm{x}_t) - l_t^{(F_1)}(\bm{x}_t) \right\}
        +(1 - \alpha)\sum_{t=1}^T \left\{ u_t^{(F_2)}(\bm{x}_t) - l_t^{(F_2)}(\bm{x}_t) \right\}.
    \end{equation}
    Similarly, from the definition of  $u_t^{(F_1)}$ and $l_t^{(F_1)}$, we get the following inequality:
    \begin{align}
        \sum_{t=1}^T \left\{u_t^{(F_1)}(\bm{x}_t) - l_t^{(F_1)}(\bm{x}_t)\right\}
        &= \sum_{t=1}^T \int_{\Omega} \left\{ u_t(\bm{x}_t, \bm{w}) - l_t(\bm{x}_t, \bm{w})\right\}
        p(\bm{w}) \text{d}\bm{w} \nonumber \\
        &= 2 \sum_{t=1}^T \beta_t^{1/2} \int_{\Omega} \sigma_{t-1}(\bm{x}_t, \bm{w})
        p(\bm{w}) \text{d}\bm{w} \nonumber \\
        \label{eq:F1_ineq_int}
        &\leq 2  \beta_T^{1/2} \sum_{t=1}^T \int_{\Omega} \sigma_{t-1}(\bm{x}_t, \bm{w})
        p(\bm{w}) \text{d}\bm{w}.
    \end{align}
    Here, the last inequality is given by monotonicity of   $\beta_t$.
    In addition, noting that the definition of  $u_t^{(F_2)}$ and  $l_t^{(F_2)}$ we obtain
    \begin{align}
        u_t^{(F_2)}(\bm{x}_t) - l_t^{(F_2)}(\bm{x}_t)
        &= \sqrt{\int_{\Omega} \tilde{u}_t^{(\text{sq})}(\bm{x}_t, \bm{w}) p(\bm{w}) \text{d}\bm{w}} -
        \sqrt{\int_{\Omega} \tilde{l}_t^{(\text{sq})}(\bm{x}_t, \bm{w}) p(\bm{w}) \text{d}\bm{w}} \nonumber \\
        \label{eq:F2_cred_bound}
        &\leq \sqrt{\int_{\Omega} \left\{ \tilde{u}_t^{(\text{sq})}(\bm{x}_t, \bm{w})  -
        \tilde{l}_t^{(\text{sq})}(\bm{x}_t, \bm{w}) \right\} p(\bm{w}) \text{d}\bm{w}},
    \end{align}
    where the last inequality is obtained  by using the fact that  $\sqrt{a} - \sqrt{b} \leq \sqrt{a - b}$ for any 
 $a \geq b \geq 0$. Furthermore, we have 
    \begin{equation}
        \label{eq:sq_bound}
        \tilde{u}_t^{(\text{sq})}(\bm{x}_t, \bm{w}) -
         \tilde{l}_t^{(\text{sq})}(\bm{x}_t, \bm{w})
         = \max\left\{ \tilde{l}_t^2(\bm{x}_t, \bm{w}), \tilde{u}_t^2(\bm{x}_t, \bm{w})\right\}
         - \min\left\{ \tilde{l}_t^2(\bm{x}_t, \bm{w}), \tilde{u}_t^2(\bm{x}_t, \bm{w})\right\} + \text{STR}_{0, t}^2(\bm{x}_t, \bm{w}),
    \end{equation}
   where $\text{STR}_{0, t}(\bm{x}_t, \bm{w}) = \max\left\{
    0, \min\left(\tilde{u}_t(\bm{x}_t, \bm{w}), -\tilde{l}_t(\bm{x}_t, \bm{w})\right)\right\}$. 
    Moreover, we define $\tilde{\mu}_{t-1}(\bm{x}, \bm{w})$ and $\tilde{\sigma}_{t-1}(\bm{x}, \bm{w})$
     as 
    \begin{align*}
        \tilde{\mu}_{t-1}(\bm{x}, \bm{w}) =
        \mu_{t-1}(\bm{x}, \bm{w}) - \mathbb{E}_{\bm{w}}[\mu_{t-1}(\bm{x}, \bm{w})], \\
        \tilde{\sigma}_{t-1}(\bm{x}, \bm{w}) =
        \sigma_{t-1}(\bm{x}, \bm{w}) + \mathbb{E}_{\bm{w}}[\sigma_{t-1}(\bm{x}, \bm{w})].
    \end{align*}
    Then, $\tilde{l}_t(\bm{x}, \bm{w})$ and $\tilde{u}_t(\bm{x}, \bm{w})$ can be expressed as follows:
    \begin{align*}
        \tilde{l}_t(\bm{x}, \bm{w}) = \tilde{\mu}_{t-1}(\bm{x}, \bm{w}) - \beta_t^{1/2} \tilde{\sigma}_{t-1}(\bm{x}, \bm{w}), \\
        \tilde{u}_t(\bm{x}, \bm{w}) =  \tilde{\mu}_{t-1}(\bm{x}, \bm{w}) + \beta_t^{1/2} \tilde{\sigma}_{t-1}(\bm{x}, \bm{w}).
    \end{align*}
    Here, if $\tilde{l}_t^2(\bm{x}_t, \bm{w}) \leq \tilde{u}_t^2(\bm{x}_t, \bm{w})$, then we have 
     $\tilde{\mu}_{t-1}(\bm{x}_t, \bm{w}) \geq 0$ and 
    \begin{align*}
        &\quad \max\left\{\tilde{l}_t^2(\bm{x}_t, \bm{w}), \tilde{u}_t^2(\bm{x}_t, \bm{w})\right\}
        - \min\left\{\tilde{l}_t^2(\bm{x}_t, \bm{w}), \tilde{u}_t^2(\bm{x}_t, \bm{w})\right\} \\
        &= \left\{ \tilde{\mu}_{t-1}(\bm{x}_t, \bm{w}) + \beta_t^{1/2} \tilde{\sigma}_{t-1}(\bm{x}_t, \bm{w}) \right\}^2
        - \left\{ \tilde{\mu}_{t-1}(\bm{x}_t, \bm{w}) - \beta_t^{1/2} \tilde{\sigma}_{t-1}(\bm{x}_t, \bm{w}) \right\}^2 \\
        &= 4 \beta_t^{1/2} \tilde{\mu}_{t-1}(\bm{x}_t, \bm{w}) \tilde{\sigma}_{t-1}(\bm{x}_t, \bm{w}) \\
        &= 4 \beta_t^{1/2} |\tilde{\mu}_{t-1}(\bm{x}_t, \bm{w})| \tilde{\sigma}_{t-1}(\bm{x}_t, \bm{w}).
    \end{align*}
    On the other hand, if 
     $\tilde{l}_t^2(\bm{x}_t, \bm{w}) > \tilde{u}_t^2(\bm{x}_t, \bm{w})$, then we get 
     $\tilde{\mu}_{t-1}(\bm{x}_t, \bm{w}) < 0$ and 
    \begin{align*}
        &\quad \max\left\{\tilde{l}_t^2(\bm{x}_t, \bm{w}), \tilde{u}_t^2(\bm{x}_t, \bm{w})\right\}
        - \min\left\{\tilde{l}_t^2(\bm{x}_t, \bm{w}), \tilde{u}_t^2(\bm{x}_t, \bm{w})\right\} \\
        &= \left\{ \tilde{\mu}_{t-1}(\bm{x}_t, \bm{w}) - \beta_t^{1/2} \tilde{\sigma}_{t-1}(\bm{x}_t, \bm{w}) \right\}^2
        - \left\{ \tilde{\mu}_{t-1}(\bm{x}_t, \bm{w}) + \beta_t^{1/2} \tilde{\sigma}_{t-1}(\bm{x}_t, \bm{w}) \right\}^2 \\
        &= -4 \beta_t^{1/2} \tilde{\mu}_{t-1}(\bm{x}_t, \bm{w}) \tilde{\sigma}_{t-1}(\bm{x}_t, \bm{w}) \\
        &= 4 \beta_t^{1/2} |\tilde{\mu}_{t-1}(\bm{x}_t, \bm{w})| \tilde{\sigma}_{t-1}(\bm{x}_t, \bm{w}).
    \end{align*}
    Therefore, in all cases the following equality holds:
    \begin{equation*}
        \max\left\{\tilde{l}_t^2(\bm{x}_t, \bm{w}), \tilde{u}_t^2(\bm{x}_t, \bm{w})\right\}
        - \min\left\{\tilde{l}_t^2(\bm{x}_t, \bm{w}), \tilde{u}_t^2(\bm{x}_t, \bm{w})\right\} = 4 \beta_t^{1/2} |\tilde{\mu}_{t-1}(\bm{x}_t, \bm{w})| \tilde{\sigma}_{t-1}(\bm{x}_t, \bm{w}).
    \end{equation*}
Next, since  \eqref{eq:f_cred_cond} holds, we get  
    $f(\bm{x}, \bm{w}) - \mathbb{E}_{\bm{w}}[f(\bm{x}_t, \bm{w})] \in [\tilde{l}_t(\bm{x}, \bm{w}), \tilde{u}_t(\bm{x}, \bm{w})]$. 
    This implies that 
    \begin{equation*}
        |f(\bm{x}, \bm{w}) - \mathbb{E}_{\bm{w}}[f(\bm{x}_t, \bm{w})]
        - \tilde{\mu}_{t-1}(\bm{x}, \bm{w})|
        \leq \beta_t^{1/2} \tilde{\sigma}(\bm{x}, \bm{w}).
    \end{equation*}
   Hence, we have
    \begin{align*}
        &\quad |f(\bm{x}, \bm{w}) - \mathbb{E}_{\bm{w}}[f(\bm{x}_t, \bm{w})]
        - \tilde{\mu}_{t-1}(\bm{x}, \bm{w})|
        \leq \beta_t^{1/2} \tilde{\sigma}_{t-1}(\bm{x}, \bm{w}) \\
        &\Rightarrow |\tilde{\mu}_{t-1}(\bm{x}, \bm{w})| \leq
        |f(\bm{x}, \bm{w}) - \mathbb{E}_{\bm{w}}[f(\bm{x}_t, \bm{w})]| +
        \beta_t^{1/2} \tilde{\sigma}_{t-1}(\bm{x}, \bm{w}) \\
        &\Rightarrow |\tilde{\mu}_{t-1}(\bm{x}, \bm{w})| \leq
        \tilde{B} +
        \beta_t^{1/2} \tilde{\sigma}_{t-1}(\bm{x}, \bm{w}).
    \end{align*}
    Thus, the following inequality holds:
    \begin{align}
        &\quad\max\left\{\tilde{l}_t^2(\bm{x}_t, \bm{w}), \tilde{u}_t^2(\bm{x}_t, \bm{w})\right\}
        - \min\left\{\tilde{l}_t^2(\bm{x}_t, \bm{w}), \tilde{u}_t^2(\bm{x}_t, \bm{w})\right\} \nonumber \\
        &\leq 4\beta_t^{1/2} \tilde{\sigma}_{t-1}(\bm{x}_t, \bm{w}) \left\{\tilde{B} +
        \beta_t^{1/2}\tilde{\sigma}_{t-1}(\bm{x}_t, \bm{w})\right\} \nonumber \\
        \label{eq:max_min_bound}
        &= 4\tilde{B} \beta_t^{1/2} \tilde{\sigma}_{t-1}(\bm{x}_t, \bm{w})
        + 4 \beta_t \tilde{\sigma}_{t-1}^2(\bm{x}_t, \bm{w}).
    \end{align}
    Moreover, $\text{STR}_{0, t}(\bm{x}_t, \bm{w})$ can be bounded as 
    \begin{align}
        \text{STR}_{0, t}(\bm{x}_t, \bm{w})
        &\leq \frac{\tilde{u}_t(\bm{x}_t, \bm{w}) - \tilde{l}_t(\bm{x}_t, \bm{w})}{2} \nonumber \\
        \label{eq:str_bound}
        &= \beta_t^{1/2}\tilde{\sigma}_{t-1}(\bm{x}_t, \bm{w}).
    \end{align}
    Hence, from  (\ref{eq:sq_bound}), (\ref{eq:max_min_bound}) and 
     (\ref{eq:str_bound}), we obtain 
    \begin{equation*}
        u_t^{(\text{sq})}(\bm{x}_t, \bm{w}) - l_t^{(\text{sq})}(\bm{x}_t, \bm{w}) \leq 4\tilde{B}
        \beta_t^{1/2} \tilde{\sigma}_{t-1}(\bm{x}_t, \bm{w}) +
        5 \beta_t \tilde{\sigma}_{t-1}^2(\bm{x}_t, \bm{w})
    \end{equation*}
    and 
    \begin{align*}
        &\quad \int_{\Omega} \left\{ u_t^{(\text{sq})}(\bm{x}_t, \bm{w}) - l_t^{(\text{sq})}(\bm{x}_t, \bm{w}) \right\} p(\bm{w}) \text{d}\bm{w} \\
        &\leq 4\tilde{B}
         \beta_t^{1/2}\int_{\Omega} \tilde{\sigma}_{t-1}(\bm{x}_t, \bm{w}) p(\bm{w})
         \text{d}\bm{w} + 5\beta_t \int_{\Omega} \tilde{\sigma}_{t-1}^2(\bm{x}_t, \bm{w}) p(\bm{w})
          \text{d}\bm{w}.
    \end{align*}
     In addition, from the definition of  $\tilde{\sigma}_{t-1}(\bm{x}_t, \bm{w})$, the following holds:
    \begin{align*}
        \int_{\Omega} \tilde{\sigma}_{t-1}(\bm{x}_t, \bm{w}) p(\bm{w})
         \text{d}\bm{w}
         &= \mathbb{E}_{\bm{w}}[\sigma_{t-1}(\bm{x}_t, \bm{w})]
         + \int_{\Omega} \sigma_{t-1}(\bm{x}_t, \bm{w}) p(\bm{w}) \text{d}\bm{w} \\
         &= 2\int_{\Omega} \sigma_{t-1}(\bm{x}_t, \bm{w}) p(\bm{w})
         \text{d}\bm{w}, \\
         \int_{\Omega} \tilde{\sigma}_{t-1}^2(\bm{x}_t, \bm{w}) p(\bm{w})
          \text{d}\bm{w}
          &= \int_{\Omega} \sigma_{t-1}^2(\bm{x}_t, \bm{w}) p(\bm{w})
          \text{d}\bm{w}
          + 2\mathbb{E}_{\bm{w}}[\sigma_{t-1}(\bm{x}_t, \bm{w})]
          \int_{\Omega} \sigma_{t-1}(\bm{x}_t, \bm{w}) p(\bm{w}) \text{d}\bm{w}
          + \left\{\mathbb{E}_{\bm{w}}[\sigma_{t-1}(\bm{x}_t, \bm{w})]\right\}^2 \\
          &= \int_{\Omega} \sigma_{t-1}^2(\bm{x}_t, \bm{w}) p(\bm{w})
          \text{d}\bm{w}
          + 3\left\{\int_{\Omega} \sigma_{t-1}(\bm{x}_t, \bm{w})p(\bm{w})
          \text{d}\bm{w} \right\}^2 \\
          &\leq 4\int_{\Omega} \sigma_{t-1}^2(\bm{x}_t, \bm{w}) p(\bm{w})
          \text{d}\bm{w}.
    \end{align*}
    Here, the last inequality is obtained by using Jensen's inequality and convexity of $g(x)=x^2$. 
Therefore, we have
    \begin{align}
        &\quad\int_{\Omega} \left\{ u_t^{(\text{sq})}(\bm{x}_t, \bm{w}) - l_t^{(\text{sq})}(\bm{x}_t, \bm{w}) \right\} p(\bm{w}) \text{d}\bm{w} \nonumber \\
        \label{eq:F2_sq_cred_bound}
        &\leq 8\tilde{B} \beta_t^{1/2}
        \int_{\Omega} \sigma_{t-1}(\bm{x}_t, \bm{w}) p(\bm{w}) \text{d}\bm{w}
         + 20 \beta_t \int_{\Omega} \sigma_{t-1}^2(\bm{x}_t, \bm{w}) p(\bm{w})
          \text{d}\bm{w}.
    \end{align}
Thus, by using (\ref{eq:F2_sq_cred_bound}) and Schwartz's inequality for 
      (\ref{eq:F2_cred_bound}), we get 
    \begin{align}
        &\quad\sum_{t=1}^T\left\{u_t^{(F_2)}(\bm{x}_t) - l_t^{(F_2)}(\bm{x}_t)\right\} \nonumber \\
        \label{eq:F2_cred_sum}
        &\leq \sqrt{8T\tilde{B} \beta_T^{1/2}
        \sum_{t=1}^T\int_{\Omega} \sigma_{t-1}(\bm{x}_t, \bm{w}) p(\bm{w}) \text{d}\bm{w}
         + 20 T\beta_T \sum_{t=1}^T\int_{\Omega} \sigma_{t-1}^2(\bm{x}_t, \bm{w}) p(\bm{w})
          \text{d}\bm{w}}.
    \end{align}
    Therefore, from  (\ref{eq:cred_G_bound}), (\ref{eq:F1_ineq_int}) and (\ref{eq:F2_cred_sum}), 
    we have the desired inequality.
\end{proof}

Next, in order to evaluate $\sum_{t=1}^T\int_{\Omega}\sigma_{t-1}(\bm{x}_t, \bm{w}) \text{d}\bm{w}$ and 
 $\sum_{t=1}^T\int_{\Omega}\sigma_{t-1}^2(\bm{x}_t, \bm{w}) \text{d}\bm{w}$ in the right hand side of 
the inequality of Lemma  \ref{lem:R_bound}, 
we introduce the following lemma given by \cite{kirschner18heteroscedastic}:
\begin{lemma}
    \label{lem:mean_bound}
Let 
    $S_t$ be any non-negative stochastic process adapted to a filtration $\{\mathcal{F}_t\}$, and define 
      $m_t = E[S_t \mid \mathcal{F}_{t-1}]$. 
Assume that $S_t \leq K$ for 
 $K \geq 1$. 
Then, for any $T \geq 1$, the following holds with probability at least $1-\delta$:
    \begin{align*}
        \sum_{t=1}^T m_t \leq 2 \sum_{t=1}^T S_t + 8K \ln \frac{6K}{\delta}.
    \end{align*}
\end{lemma}
Furthermore, from the assumption about the kernel function, we get 
 $k((\bm{x}_t, \bm{w}), (\bm{x}_t, \bm{w})) \leq 1$ and 
$\sigma_{t-1}(\bm{x}_t, \bm{w}) \leq k((\bm{x}_t, \bm{w}), (\bm{x}_t, \bm{w})) \leq  1$. Hence, from 
Lemma \ref{lem:mean_bound}, with probability at least $1-\delta/3$, it holds that 
\begin{equation}\label{eq:sigma_int}
        \sum_{t=1}^T \int_{\Omega} \sigma_{t-1}(\bm{x}_t, \bm{w})
        p(\bm{w}) \text{d}\bm{w} \leq 2 \sum_{t=1}^T
        \sigma_{t-1}(\bm{x}_t, \bm{w}_t) + 8 \ln \frac{18}{\delta}.
\end{equation}
Similarly, the following inequality holds with probability at least  $1-\delta/3$:
\begin{equation}\label{eq:sigma_int_sq}
        \sum_{t=1}^T \int_{\Omega} \sigma_{t-1}^2(\bm{x}_t, \bm{w})
        p(\bm{w}) \text{d}\bm{w} \leq 2 \sum_{t=1}^T
        \sigma_{t-1}^2(\bm{x}_t, \bm{w}_t) + 8 \ln \frac{18}{\delta}.
\end{equation}

In addition, we introduce the following lemma given by \cite{DBLP:conf/icml/SrinivasKKS10} about the 
 maximum information gain $\gamma_T$:
\begin{lemma}
    \label{lem:gamma_T}
    Fix $T \geq 1$.  
Then, the following inequality holds:
    \begin{equation}
        \label{eq:sigma_sum}
        \sum_{t=1}^T \sigma_{t-1}^2(\bm{x}_t, \bm{w}_t)
        \leq \frac{2}{\ln(1 + \sigma^{-2})}\gamma_T.
    \end{equation}
\end{lemma}
Moreover, from Schwarz's inequality and Lemma  
\ref{lem:gamma_T}, we get the following inequality:
\begin{equation}
    \label{eq:sigma_sq_sum}
    \sum_{t=1}^T \sigma_{t-1}(\bm{x}_t, \bm{w}_t)
    \leq \sqrt{\frac{2T}{\ln(1 + \sigma^{-2})}\gamma_T}.
\end{equation}
Thus, from 
(\ref{eq:sigma_int}), (\ref{eq:sigma_int_sq}), (\ref{eq:sigma_sum}) and 
(\ref{eq:sigma_sq_sum}) we obtain the following corollary:
\begin{corollary}
    \label{cor:R_bound_final}
    Assume that (\ref{eq:f_cred_cond}), (\ref{eq:sigma_int}) and (\ref{eq:sigma_int_sq}) hold. Then, for any 
    $T \geq 1$, it holds that 
    \begin{align*}
        &\sum_{t=1}^T\left\{ u_t^{(G)}(\bm{x}_t) - l_t^{(G)}(\bm{x}_t)\right\}
        \leq \alpha \beta_T^{1/2}\left\{ \sqrt{2TC_1\gamma_T} + C_2\right\} \\
        &\hspace{20pt} +(1 - \alpha) \sqrt{2T\tilde{B} \beta_T^{1/2}
        \left\{ \sqrt{8TC_1\gamma_T} + 2C_2 \right\}
         + 5 T\beta_T \left\{C_1\gamma_T + 2C_2 \right\}},
    \end{align*}
    where $C_1=\frac{16}{\ln(1 + \sigma^{-2})}$ and $C_2=16\ln \frac{18}{\delta}$.
\end{corollary}
\begin{proof}
    From Lemma \ref{lem:R_bound}, 
     (\ref{eq:sigma_int}) and (\ref{eq:sigma_int_sq}), it holds that 
    \begin{align}
        &\sum_{t=1}^T\left\{ u_t^{(G)}(\bm{x}_t) - l_t^{(G)}(\bm{x}_t)\right\}
        \leq 4\alpha \beta_T^{1/2}\left\{ \sum_{t=1}^T \sigma_{t-1}(\bm{x}_t, \bm{w}_t) + 4\ln \frac{18}{\delta}\right\} \nonumber \\
        \label{eq:R_bound_nonint}
        &\hspace{20pt} + (1 - \alpha) \sqrt{16T\tilde{B} \beta_T^{1/2}
        \left\{ \sum_{t=1}^T \sigma_{t-1}(\bm{x}_t, \bm{w}_t) + 4\ln \frac{18}{\delta} \right\}
         + 40 T\beta_T \left\{\sum_{t=1}^T\sigma_{t-1}^2(\bm{x}_t, \bm{w}_t) + 4\ln \frac{18}{\delta} \right\}}.
    \end{align}
    Therefore, by combining 
     (\ref{eq:sigma_sum}), (\ref{eq:sigma_sq_sum}) and 
     (\ref{eq:R_bound_nonint}), we get the desired inequality.
\end{proof}
Finally, we prove Theorem 
 \ref{thm:sca_conv}. 
Let  $T\geq 1$, and define  
$\hat{T} = {\rm argmax}_{t = 1,\ldots,T} l_t^{(G)}(\bm{x}_t)$. 
Assume that 
 (\ref{eq:f_cred_cond}) holds. 
Then, for any $\bm{x} \in \mathcal{X}$, it holds that 
 $G(\bm{x}) \in [l_t^{(G)}(\bm{x}), u_t^{(G)}(\bm{x})]$. Thus, for any  $t^\prime = 1, \ldots, T$, we get 
\begin{align*}
    G(\bm{x}^\ast) - G(\hat{\bm{x}}_T)
    &\leq u_{{t}^\prime}^{(G)}(\bm{x}_{t^\prime})
    - l_{\hat{T}}^{(G)}(\hat{\bm{x}}_T) \\
    &= u_{{t}^\prime}^{(G)}(\bm{x}_{t^\prime}) -
    \max_{t = 1, \ldots, T} l_t^{(G)}(\hat{\bm{x}}_t) \\
    &\leq u_{t^\prime}^{(G)}(\bm{x}_{t^\prime}) - l_{t^\prime}^{(G)}(\bm{x}_{t^\prime}).
\end{align*}
 This implies that 
\begin{equation}
    \label{eq:regret_bound}
    G(\bm{x}^\ast) - G(\hat{\bm{x}}_T) \leq
    \frac{1}{T}\sum_{t=1}^T \left\{ u_{t}^{(G)}(\bm{x}_{t}) - l_{t}^{(G)}(\bm{x}_{t}) \right\}.
\end{equation}
Here, note that with probability at least $1-\delta$, 
 (\ref{eq:f_cred_cond}), (\ref{eq:sigma_int}) and  (\ref{eq:sigma_int_sq}) hold. 
Therefore, by combining Corollary 
 \ref{cor:R_bound_final}, the following holds with probability at least 
 $1-\delta$:
\begin{align*}
    G(\bm{x}^\ast) - G(\hat{\bm{x}}_T) \leq
    \alpha T^{-1} \beta_T^{1/2}\left( \sqrt{2TC_1\gamma_T} + C_2\right) + (1 - \alpha)T^{-1} \sqrt{2T\tilde{B} \beta_T^{1/2} \left( \sqrt{8TC_1\gamma_T} + 2C_2 \right)
     + 5 T\beta_T \left(C_1\gamma_T + 2C_2 \right)}.
\end{align*}
Hence, if $T$ satisfies  \eqref{eq:sca_conv_rate}, with probability at least 
 $1-\delta$, it holds that 
$G(\bm{x}^\ast) - G(\hat{\bm{x}}_T) \leq \epsilon$. Therefore, 
  $\hat{\bm{x}}_T$ is the $\epsilon$-accurate solution.

\subsection{Proof of Theorem \ref{thm:par_conv}}
In this subsection, we prove Theorem \ref{thm:par_conv}. 
First, we show several lemmas.
\begin{lemma}\label{lem:existence_pihat}
For any $t \geq 1$, $\hat{\Pi}_t$ has at least one element (i.e.,  $\hat{\Pi}_t \neq \emptyset$).
\end{lemma}
\begin{proof}
Let $t \geq 1$.  We define $\tilde{\bm x}_t$ and ${\bm x}^\dagger _t$ as 
\begin{align*}
\tilde{\bm x}_t &= \argmax _{ {\bm x} \in \mathcal{X} } l^{(F_2)}_t ({\bm x} ), \\
{\bm x}^\dagger_t &= \argmax _{ {\bm x} \in \mathcal{X} ; l^{(F_2)}_t ({\bm x} )=l^{(F_2)}_t (\tilde{\bm x}_t )} l^{(F_1)}_t ({\bm x} ).
\end{align*} 
 Assume that $ E^{(\text{pes})}_{t,{\bm x}^\dagger_t}=\emptyset$. 
Then, it holds that 
$$
\forall {\bm x}^{\prime} \in   \emptyset = E^{(\text{pes})}_{t,{\bm x}^\dagger_t}, \ 
 \bm{F}_{t}^{(\text{pes})}(\bm{x}^\dagger_t) \npreceq \bm{F}_{t}^{(\text{pes})}(\bm{x}^{\prime }).
$$
This implies that ${\bm x}^\dagger_t \in \hat{\Pi} _t $. 

On the other hand, if $ E^{(\text{pes})}_{t,{\bm x}^\dagger_t} \neq \emptyset$, then the following holds for any 
${\bm x}^\prime \in E^{(\text{pes})}_{t,{\bm x}^\dagger_t}$:
$$
l^{(F_2)}_t ({\bm x}^\dagger_t ) =  l^{(F_2)}_t (\tilde{\bm x}_t )  \geq l^{(F_2)}_t ({\bm x}^\prime ) .
$$
Here, if $l^{(F_2)}_t ({\bm x}^\dagger_t )  >  l^{(F_2)}_t ({\bm x}^\prime ) $, it holds that 
$ \bm{F}_{t}^{(\text{pes})}(\bm{x}^\dagger_t) \npreceq \bm{F}_{t}^{(\text{pes})}(\bm{x}^{\prime })$. 
Similarly, if  $l^{(F_2)}_t ({\bm x}^\dagger_t )  =  l^{(F_2)}_t ({\bm x}^\prime ) $, it holds that 
$$
l^{(F_1)}_t ({\bm x}^\dagger_t )   \geq l^{(F_1)}_t ({\bm x}^\prime ) .
$$
Noting that ${\bm F}^{(\text{pes})}_t ({\bm x}^\dagger_t ) \neq {\bm F}^{(\text{pes})}_t ({\bm x}^\prime )$ and 
 $l^{(F_2)}_t ({\bm x}^\dagger_t )  =  l^{(F_2)}_t ({\bm x}^\prime ) $, we have 
 $l^{(F_1)}_t ({\bm x}^\dagger_t )  >  l^{(F_1)}_t ({\bm x}^\prime ) $. 
Thus, we have $ \bm{F}_{t}^{(\text{pes})}(\bm{x}^\dagger_t) \npreceq \bm{F}_{t}^{(\text{pes})}(\bm{x}^{\prime })$. 
Form the definition of $\hat{\Pi}_t$, we get ${\bm x}^\dagger_t \in \hat{\Pi} _t $. 
\end{proof}

\begin{lemma}\label{lem:existence_Mt}
Let $t \geq 1$, and assume that  $M_t  \neq \emptyset$. Also let 
${\bm x}^{(1)} $ be an element of $M_t$. Then, there exists an element ${\bm x}^\prime \in \hat{\Pi}_t $ such that 
$$
 \bm{F}_{t}^{(\text{pes})}(\bm{x}^{(1)}) \preceq \bm{F}_{t}^{(\text{pes})}(\bm{x}^{\prime }).
$$
\end{lemma}
\begin{proof}
Let $t \geq 1$, $M_t \neq \emptyset$ and ${\bm x}^{(1)} \in M_t$. 
Assume that the following holds:
\begin{equation}
 \bm{F}_{t}^{(\text{pes})}(\bm{x}^{(1)}) \npreceq \bm{F}_{t}^{(\text{pes})}(\bm{x}^{\prime }),\ \forall   {\bm x}^\prime \in \hat{\Pi}_t  . \label{eq:w_assumption}
\end{equation}
From the definition of $M_t$, we have ${\bm x}^{(1)} \notin \hat{\Pi}_t $. 
Here,  since  ${\bm x}^{(1)} \notin \hat{\Pi}_t $, there exists ${\bm x}^{(2)} \in E^{(\text{pes})}_{t,{\bm x}^{(1)} } $ such that 
$$
 \bm{F}_{t}^{(\text{pes})}(\bm{x}^{(1)}) \preceq \bm{F}_{t}^{(\text{pes})}(\bm{x}^{(2) }).
$$
Therefore, there exists ${\bm x}^{(3)} \in E^{(\text{pes})}_{t,{\bm x}^{(2)} } $ such that 
$$
 \bm{F}_{t}^{(\text{pes})}(\bm{x}^{(2)}) \preceq \bm{F}_{t}^{(\text{pes})}(\bm{x}^{(3) }).
$$
Furthermore, by combining 
$$
 \bm{F}_{t}^{(\text{pes})}(\bm{x}^{(1)}) \preceq \bm{F}_{t}^{(\text{pes})}(\bm{x}^{(2) }), \ 
 \bm{F}_{t}^{(\text{pes})}(\bm{x}^{(2)}) \preceq \bm{F}_{t}^{(\text{pes})}(\bm{x}^{(3) })
$$
we get $ \bm{F}_{t}^{(\text{pes})}(\bm{x}^{(1)}) \preceq \bm{F}_{t}^{(\text{pes})}(\bm{x}^{(3) })$. 
Thus, from \eqref{eq:w_assumption} we obtain ${\bm x}^{(3)} \notin \hat{\Pi}_t$. 
By repeating the same argument, we have ${\bm x}^{(1)},\ldots, {\bm x}^{(|\mathcal{X}|)}$, where ${\bm x}^{(k)} \notin \hat{\Pi}_t $, 
$k=1,\ldots, |\mathcal{X}|$.  
Next, we show that ${\bm x}^{(i)} \neq {\bm x}^{(j) } $ for any $i $ and $j$ with $i \neq j$. 
In fact, if there exist $i$ and $j$ with $i < j$ such that ${\bm x}^{(i)} = {\bm x}^{(j)} $, we get ${\bm  F}^{(\text{pes})}_t ({\bm x}^{(i)}) =
  {\bm  F}^{(\text{pes})}_t ({\bm x}^{(j)})$. 
Here, from $i \leq j-1$, noting that the definition of ${\bm x}^{(i)}$ and ${\bm x}^{(j-1)}$ we get 
$$
  {\bm  F}^{(\text{pes})}_t ({\bm x}^{(j)})={\bm  F}^{(\text{pes})}_t ({\bm x}^{(i)}) \leq
  {\bm  F}^{(\text{pes})}_t ({\bm x}^{(j-1)}).
$$
Similarly, from the definition of ${\bm x}^{(j-1)}$ and ${\bm x}^{(j)}$, we obtain 
$$
{\bm  F}^{(\text{pes})}_t ({\bm x}^{(j-1)}) \leq
  {\bm  F}^{(\text{pes})}_t ({\bm x}^{(j)}).
$$
Thus, we get ${\bm  F}^{(\text{pes})}_t ({\bm x}^{(j-1)}) =
  {\bm  F}^{(\text{pes})}_t ({\bm x}^{(j)})$. However, it contradicts  ${\bm x}^{(j)} \in E^{(\text{pes})}_{t,{\bm x}^{(j-1)} } $. 
Hence, it holds that ${\bm x}^{(i)} \neq {\bm x}^{(j) } $ for any $i $ and $j$ with $i \neq j$. 
Therefore, the set $\{ {\bm x}^{(1)},\ldots, {\bm x}^{(|\mathcal{X}|)}  \} $ is equal to $\mathcal{X}$. 
Recall that ${\bm x}^{(k)} \notin \hat{\Pi}_t $ for any $k=1,\ldots, |\mathcal{X}|$. 
By combining this and $\{ {\bm x}^{(1)},\ldots, {\bm x}^{(|\mathcal{X}|)}  \} = \mathcal{X}$, we have $\hat{\Pi}_t = \emptyset$. 
However, it contradicts Lemma \ref{lem:existence_pihat}. 
Hence, the assumption \eqref{eq:w_assumption} is incorrect.
\end{proof}

\begin{lemma}\label{lem:Ze}
Let ${\bm x}$ be an element of $\mathcal{X}$, and let $\bm\epsilon  =(\epsilon _1,\epsilon_2)$ be a positive vector.
 Assume that at least one of the following inequalities holds for any ${\bm x}^\prime \in \mathcal{X}$:
$$
F_1 ({\bm x}) + \epsilon _1 \geq F_1 ({\bm x}^\prime ), \ 
F_2 ({\bm x}) + \epsilon _2 \geq F_2 ({\bm x}^\prime ).
$$ 
Then, it holds that ${\bm F } ({\bm x} )  \in Z_{\bm\epsilon}$.
\end{lemma}
\begin{proof}
In order to prove Lemma \ref{lem:Ze}, we consider the following two cases:
\begin{description}
\item [(1)] For any ${\bm x},{\bm x}^\prime \in \Pi$, ${\bm F} ({\bm x}) ={\bm F} ({\bm x}^\prime )$.
\item  [(2)] There exist  ${\bm x},{\bm x}^\prime \in \Pi$ such that  ${\bm F} ({\bm x}) \neq {\bm F} ({\bm x}^\prime )$.
\end{description}
First, we consider {\sf (1)}. We define ${\bm x}^{(1)} $ and ${\bm x}^{(2)} $ as 
\begin{align*}
\tilde{\bm x} = \argmax _{ {\bm x} \in \mathcal{X} }  F_1 ({\bm x} ), \ {\bm x}^{(1)} &= \argmax_ { {\bm x}; F_1 ({\bm x} ) =F_1 (\tilde{\bm x})} F_2 ({\bm x} ), \\
{\bm x}^\dagger = \argmax _{ {\bm x} \in \mathcal{X} }  F_2 ({\bm x} ), \ {\bm x}^{(2)} &= \argmax_ { {\bm x}; F_2 ({\bm x} ) =F_2 ({\bm x}^\dagger)} F_1 ({\bm x} ).
\end{align*}
From the definition of ${\bm x}^{(1)} $ and ${\bm x}^{(2)} $, it holds that ${\bm x}^{(1)},{\bm x}^{(2)} \in \Pi$. 
Thus, from {\sf (1)}, we get ${\bm F} ({\bm x}^{(1)} ) =  {\bm F} ({\bm x}^{(2)} ) $. Hence, 
the following holds for any ${\bm x}^\prime \in \mathcal{X}$:
\begin{align*}
F_1 ({\bm x}^\prime ) \leq F_1 ({\bm x} ^{(1)} ) , \ 
F_2 ({\bm x}^\prime ) \leq F_2 ({\bm x} ^{(2)} )=F_2 ({\bm x} ^{(1)} ). \ 
\end{align*}
Therefore, we get ${\bm F} ({\bm x}^\prime ) \preceq {\bm F} ({\bm x}^{(1)} )$.
 Note that ${\bm F} ({\bm x}^{(1)} ) \in Z$. 
Here, let ${\bm x} \in \mathcal{X}$. Then, from the lemma's assumption, at least one of the following inequalities holds:
$$
F_1 ({\bm x}) + \epsilon _1 \geq F_1 ({\bm x}^{(1)} ), \ 
F_2 ({\bm x}) + \epsilon _2 \geq F_2 ({\bm x}^{(1)} ).
$$ 
If $F_1 ({\bm x}) + \epsilon _1 \geq F_1 ({\bm x}^{(1)} )$, we set ${\bm a} = ( F_1 ({\bm x}^{(1)} ),  F_2 ({\bm x} ))^\top $. 
Noting that $F_2 ({\bm x}^\prime ) \leq F_2 ({\bm x}^{(1)} ) $ for any ${\bm x}^\prime \in \mathcal{X}$, 
we have ${\bm a} \preceq {\bm F} ({\bm x}^{(1)} )$. This implies that ${\bm a} \in Z$. 
Thus, the following holds:
$$
{\bm a} = ( F_1 ({\bm x}^{(1)} ),  F_2 ({\bm x} ))^\top  \preceq  ( F_1 ({\bm x} )+\epsilon_1,  F_2 ({\bm x} )+\epsilon_2)^\top = {\bm F} ({\bm x} ) + {\bm\epsilon}.
$$
Furthermore, since ${\bm F} ({\bm x}) \preceq  {\bm F} ({\bm x}^{(1)})$ and $ {\bm F} ({\bm x}^{(1)}) \in Z$, we obtain 
${\bm F} ({\bm x} ) \in Z_{\bm\epsilon} $. 
Similarly, if $F_2 ({\bm x}) + \epsilon _2 \geq F_2 ({\bm x}^{(1)} )$,   we set ${\bm b} = ( F_1 ({\bm x} ),  F_2 ({\bm x}^{(1)} ))^\top $. Also in this case, by  using the same argument, we get ${\bm b} \in Z$ and 
$$
{\bm b}  \preceq  {\bm F} ({\bm x} ) + {\bm\epsilon}.
$$
By combining this and ${\bm F} ({\bm x}) \preceq  {\bm F} ({\bm x}^{(1)})$ (and $ {\bm F} ({\bm x}^{(1)}) \in Z$), we obtain 
${\bm F} ({\bm x} ) \in Z_{\bm\epsilon} $. 

Next, we consider {\sf (2)}. From {\sf (2)}, there exist ${\bm x}^{(1)} ,\ldots, {\bm x}^{(l)} $ such that 
$$
{\bm F} (\Pi ) = \{ {\bm F} ({\bm x} ) \mid {\bm x} \in \Pi \} = \{ {\bm F} ({\bm x}^{(i)} ) \mid i=1,\ldots,l \} ,\quad {\bm F} ({\bm x}^{(i)} ) \neq
{\bm F} ({\bm x}^{(j)} ), i \neq j.
$$
Here, without loss of generality, we may assume the following:
$$
F_1 ({\bm x} ^{(1)} ) < \cdots < F_1 ({\bm x} ^{(l)} ), \quad 
F_2 ({\bm x} ^{(1)} ) > \cdots > F_2 ({\bm x} ^{(l)} ).
$$
Let ${\bm x} $ be an element of $\mathcal{X}$. 
Assume that there exists $j $ such that 
$$
F_1 ({\bm x} ) + \epsilon_1  \geq F_1 ({\bm x} ^{(j)} )  , \ F_2 ({\bm x} ) + \epsilon_2 \geq F_2  ({\bm x} ^{(j+1)} ).
$$
Note that $( F_1 ({\bm x} ^{(j)} ), F_2 ({\bm x} ^{(j+1)} )^\top \in Z$. 
In addition, there exists $ i \in \{1,\ldots, l \}$ such that ${\bm F} ({\bm x} ) \preceq {\bm F} ({\bm x}^{(i)} ) \in Z$. 
Therefore, ${\bm F} ({\bm x}) \in Z_{\bm\epsilon }$. 

Similarly, assume that at least one of the following inequalities holds for any $j$:
\begin{align}
F_1 ({\bm x} ) + \epsilon_1 < F_1 ({\bm x} ^{(j)} )  , \ F_2 ({\bm x} ) + \epsilon_2 < F_2  ({\bm x} ^{(j+1)} ). \label{eq:ineq2}
\end{align}
Here, if $F_1 ({\bm x} ) + \epsilon_1 < F_1 ({\bm x} ^{(1)} ) $, from lemma's assumption it holds that $F_2 ({\bm x}) + \epsilon_2 \geq 
F_2  ({\bm x} ^{(1)} )$. Moreover, we define ${\bm c} = ( F_1 ({\bm x} ), F_2  ({\bm x} ^{(1)} ))^\top \in Z$. 
Then, the following holds:
$$
{\bm F} ({\bm x} ) +{\bm\epsilon} =  ( F_1 ({\bm x} )+\epsilon_1,F_2 ({\bm x} )+\epsilon_2)^\top \succeq  ( F_1 ({\bm x} ), F_2  ({\bm x} ^{(1)} ))^\top ={\bm c} \in Z.
$$
Furthermore, from the definition of ${\bm x}^{(1)}$, it holds that $F_2 ({\bm x}^{(1)} ) \geq F_2 ({\bm x} ) $. 
Thus, noting that $F_1 ({\bm x} ) + \epsilon_1 < F_1 ({\bm x} ^{(1)} ) $, we get $F_1 ({\bm x} ) \leq F_1 ({\bm x} ^{(1)} ) $.
By combining these, we have ${\bm F} ({\bm x} ) \preceq {\bm F} ({\bm x} ^{(1)} )  \in Z$. This implies that ${\bm F} ({\bm x} ) \in Z_{\bm\epsilon} $. 
On the other hand, if $F_1 ({\bm x} ) + \epsilon_1 \geq F_1 ({\bm x} ^{(1)} ) $, from \eqref{eq:ineq2} we get 
$F_2 ({\bm x} ) + \epsilon_2 < F_2 ({\bm x} ^{(2)} ) $. Therefore, from lemma's  assumption, we obtain 
$F_1 ({\bm x} ) + \epsilon_1 \geq F_1 ({\bm x} ^{(2)} ) $. By using \eqref{eq:ineq2} again, 
we have 
$F_2 ({\bm x} ) + \epsilon_2 < F_2 ({\bm x} ^{(3)} ) $. Hence, by repeating these procedures, 
we get $F_1 ({\bm x} ) + \epsilon_1 \geq F_1 ({\bm x} ^{(l)} ) $ and $F_2 ({\bm x} ) + \epsilon_2 < F_2 ({\bm x} ^{(l)} ) $.
Finally, noting that 
\begin{align*}
{\bm F} ({\bm x} ) & \preceq ( F_1 ({\bm x}^{(l)} ), F_2 ({\bm x}) + \epsilon_2 ) ^\top  \preceq  ( F_1 ({\bm x}^{(l)} ), F_2 ({\bm x}^{(l)} ) ) ^\top = {\bm F} ({\bm x}^{(l)} ) \in Z, \\
{\bm F} ({\bm x} )  + {\bm\epsilon} & \succeq (F_1 ({\bm x}^{(l)}) , F_2 ({\bm x} ) )^\top \in Z,
\end{align*}
we get ${\bm F} ({\bm x} ) \in Z_{\bm\epsilon} $.
\end{proof}

By using these lemmas, we prove Theorem \ref{thm:par_conv}. 
\begin{proof}
First, we prove that the algorithm terminates after at most $t^\prime$ iterations where 
$t^\prime$ is the positive integer satisfying 
${\rm max}_{\bm{x} \in M_{t^\prime} \cup \hat{\Pi}_{t^\prime}} \lambda_{t^\prime}(\bm{x})
= \lambda_{t^\prime}(\bm{x}_{t^\prime}) \leq \min\left\{\epsilon_1, \epsilon_2 \right\}$.
From the definition of $\lambda_t$, noting that $u_t^{(F_1)}(\bm{x}) - l_t^{(F_1)}(\bm{x}) \leq \lambda_t(\bm{x})$ 
 and $u_t^{(F_2)}(\bm{x}) - l_t^{(F_2)}(\bm{x}) \leq \lambda_t(\bm{x})$, we have
\begin{equation*}
    \max_{\bm{x} \in M_{t^\prime} \cup \hat{\Pi}_{t^\prime}}
    \left\{u_{t^\prime}^{(F_1)}(\bm{x}) - l_{t^\prime}^{(F_1)}(\bm{x})\right\} \leq \epsilon_1
\end{equation*}
and 
\begin{equation*}
    \max_{\bm{x} \in M_{t^\prime} \cup \hat{\Pi}_{t^\prime}}
    \left\{u_{t^\prime}^{(F_2)}(\bm{x}) - l_{t^\prime}^{(F_2)}(\bm{x})\right\} \leq \epsilon_2.
\end{equation*}
Then, for any $\bm{x}^\prime \in \hat{\Pi}_t$, it holds that 
\begin{equation}
    \label{eq:any_F1_band}
    u_{t^\prime}^{(F_1)}(\bm{x}^\prime) \leq l_{t^\prime}^{(F_1)}(\bm{x}^\prime) + \epsilon_1
\end{equation}
and 
\begin{equation}
    \label{eq:any_F2_band}
    u_{t^\prime}^{(F_2)}(\bm{x}^\prime) \leq l_{t^\prime}^{(F_2)}(\bm{x}^\prime) + \epsilon_2.
\end{equation}
Here, let $\bm{x}$ be an element of $  \hat{\Pi}_{t^\prime}$. 
Then, from the definition of 
 $\hat{\Pi}_t$, 
for any $\bm{x}^\prime \in \hat{\Pi}_{t^\prime}$, at least one of the following inequalities holds:
$$l_{t^\prime}^{(F_1)}(\bm{x}^\prime) \leq l_{t^\prime}^{(F_1)}(\bm{x}),\ l_{t^\prime}^{(F_2)}(\bm{x}^\prime) \leq l_{t^\prime}^{(F_2)}(\bm{x}).$$
Thus, from (\ref{eq:any_F1_band}) and  (\ref{eq:any_F2_band}), 
 for any $\bm{x}^\prime \in \hat{\Pi}_{t^\prime}$, it holds that 
$\bm{F}_{t}^{(\text{pes})}(\bm{x}) + \bm{\epsilon} \nprec \bm{F}_{t^\prime}^{(\text{opt})}(\bm{x}^\prime)$.
This implies that  $U_{t^\prime} = \emptyset$.
Similarly, if $M_{t^\prime} \neq \emptyset$, 
there exists $\bm{x} \in M_{t^\prime}$ such that 
$\bm{F}_{t^\prime}^{(\text{opt})}(\bm{x}) \npreceq_{\bm{\epsilon}} \bm{F}_{t^\prime}^{(\text{pes})}(\bm{x}^\prime)$ for any 
 $\bm{x}^\prime \in \hat{\Pi}_t$. 
On the other hand, from Lemma \ref{lem:existence_Mt}, 
there exists $\bm{x}^{\prime\prime} \in \hat{\Pi}_{t^\prime}$ such that 
$\bm{F}_{t^\prime}^{(\text{pes})}(\bm{x}) \preceq \bm{F}_{t^\prime}^{(\text{pes})}(\bm{x}^{\prime\prime})$. 
Moreover, 
from  (\ref{eq:any_F1_band}) and  (\ref{eq:any_F2_band}), 
$\bm{x}^{\prime\prime}$ satisfies 
$ \bm{F}_{t^\prime}^{(\text{opt})}(\bm{x}) \preceq_{\bm{\epsilon}} \bm{F}_{t^\prime}^{(\text{pes})}(\bm{x}^{\prime\prime})$.
However, it contradicts the definition of $M_t$. Hence, we get  $M_{t^\prime} = \emptyset$. 

Hereafter, we assume that (\ref{eq:f_cred_cond}), (\ref{eq:sigma_int}) and  (\ref{eq:sigma_int_sq}) hold. 
From the definition of  $\lambda_t$, we obtain 
\begin{equation*}
    \lambda_t(\bm{x}) \leq \left\{ u_t^{(F_1)}(\bm{x}) - l_t^{(F_1)}(\bm{x}) \right\} + \left\{ u_t^{(F_2)}(\bm{x}) - l_t^{(F_2)}(\bm{x}) \right\}.
\end{equation*}
This implies that
\begin{equation*}
    \sum_{t=1}^T \lambda_t(\bm{x}_t) \leq \sum_{t=1}^T\left\{ u_t^{(F_1)}(\bm{x}_t) - l_t^{(F_1)}(\bm{x}_t) \right\}
    + \sum_{t=1}^T \left\{ u_t^{(F_2)}(\bm{x}_t) - l_t^{(F_2)}(\bm{x}_t) \right\}.
\end{equation*}
Therefore, from 
(\ref{eq:F1_ineq_int}), (\ref{eq:F2_cred_sum}), (\ref{eq:sigma_int}) and  (\ref{eq:sigma_int_sq}), we get 
\begin{align*}
    &\quad \sum_{t=1}^T \lambda_t(\bm{x}_t) \leq 4 \beta_T^{1/2}\left\{ \sum_{t=1}^T \sigma_{t-1}(\bm{x}_t, \bm{w}_t) + 4\ln \frac{18}{\delta}\right\} \nonumber \\
    \label{eq:R_bound_nonint}
    &\hspace{20pt} + \sqrt{16T\tilde{B} \beta_T^{1/2}
    \left\{ \sum_{t=1}^T \sigma_{t-1}(\bm{x}_t, \bm{w}_t) + 4\ln \frac{18}{\delta} \right\}
     + 40 T\beta_T \left\{\sum_{t=1}^T\sigma_{t-1}^2(\bm{x}_t, \bm{w}_t) + 4\ln \frac{18}{\delta} \right\}}.
\end{align*}
Hence, from 
 (\ref{eq:sigma_sum}) and (\ref{eq:sigma_sq_sum}), it holds that
\begin{equation}\label{eq:lambda_mean}
    \frac{1}{T}\sum_{t=1}^T \lambda_t(\bm{x}_t) \leq T^{-1} \beta_T^{1/2}\left\{ \sqrt{2TC_1\gamma_T} + C_2\right\} + T^{-1} \sqrt{2T\tilde{B} \beta_T^{1/2} \left\{ \sqrt{8TC_1\gamma_T} + 2C_2 \right\}
     + 5 T\beta_T \left\{C_1\gamma_T + 2C_2 \right\}}.
\end{equation}
Here, let $T$ be a positive integer such that the right hand side in 
 (\ref{eq:lambda_mean}) is less than or equal to 
$\min\{\epsilon_1, \epsilon_2\}$.
Then,  there exists a positive integer $t^\prime$ such that 
 $t^\prime \leq T$ and 
$\lambda_{t^\prime}(\bm{x}_{t^\prime}) \leq \min\{\epsilon_1, \epsilon_2\}$. 
Therefore, we have $M_{t^\prime} = \emptyset$ and $U_{t^\prime} = \emptyset$. 
This means that the algorithm terminates after at most $t^\prime$ iterations. 

Next, under (\ref{eq:f_cred_cond}) we show that 
$\hat{\Pi}_t$ is the $\bm{\epsilon}$-accurate Pareto set when 
$M_t = \emptyset$ and $U_t=\emptyset$.
First, we prove $\bm{F}(\hat{\Pi}_t) \subset Z_{\bm{\epsilon}}$. 
Let $\bm{x}$ be an element of $  \hat{\Pi}_t$. 
For any ${\bm x}^\prime  \in \hat{\Pi}_t \setminus \{\bm{x}\}$, it holds that 
$\bm{F}_{t}^{(\text{pes})}(\bm{x}) + \bm{\epsilon} \nprec \bm{F}_{t}^{(\text{opt})}(\bm{x}^\prime)$ because 
  $U_t = \emptyset$.
Furthermore, 
noting that $M_t = \emptyset$, 
for any $\bm{x}^\prime \in \mathcal{X} \setminus \hat{\Pi}_t  $, there exists 
 $\bm{x}^{\prime\prime} \in \hat{\Pi}_t$ such that 
 $\bm{F}_{t}^{(\text{opt})}(\bm{x}^\prime) \preceq_{\bm{\epsilon}}
\bm{F}_{t}^{(\text{pes})}(\bm{x}^{\prime\prime})$. 
In addition, since 
 ${\bm x} \in \hat{\Pi}_t$, from the definition of $\hat{\Pi}_t$, at least one of the following inequalities holds:
$$
l_t^{(F_1)}(\bm{x}^{\prime\prime}) \leq l_t^{(F_1)}(\bm{x}), l_t^{(F_2)}(\bm{x}^{\prime\prime}) \leq l_t^{(F_2)}(\bm{x}).
$$
By combining this and $\bm{F}_{t}^{(\text{opt})}(\bm{x}^\prime) \preceq_{\bm{\epsilon}}
\bm{F}_{t}^{(\text{pes})}(\bm{x}^{\prime\prime})$, we get 
$\bm{F}_{t}^{(\text{pes})}(\bm{x}) + \bm{\epsilon} \nprec \bm{F}_{t}^{(\text{opt})}(\bm{x}^\prime)$.
Therefore, under  (\ref{eq:f_cred_cond}) at least one of the following inequalities holds for any ${\bm x}^\prime \in \mathcal{X} \setminus  \{{\bm x} \} $:
$$
F_1 ({\bm x} ) + \epsilon_1 \geq F_1 ({\bm x}^\prime ) , \ 
F_2 ({\bm x} ) + \epsilon_2 \geq F_2 ({\bm x}^\prime ). 
$$
Moreover, it is clear that $F_1 ({\bm x} ) + \epsilon_1 \geq F_1 ({\bm x}) $.
 Hence, from Lemma \ref{lem:Ze}, we get $\bm{F}(\hat{\Pi}_t) \subset Z_{\bm{\epsilon}}$.

Finally, we show that for any $\bm{x}^\prime \in \Pi$, 
 there exists $\bm{x} \in \hat{\Pi}_t$ such that 
$\bm{x}^\prime \preceq_{\bm{\epsilon}} \bm{x}$.
When 
 $\bm{x}^\prime \in \hat{\Pi}_t$, the existence of ${\bm x} $ is obvious because $\bm{x}^\prime \preceq_{\bm{\epsilon}} \bm{x}^\prime$.
On the other hand, when $\bm{x}^\prime \in \mathcal{X}\setminus \hat{\Pi}_t$, 
since 
$M_t = \emptyset$ there exists $\bm{x} \in \hat{\Pi}_t$ such that 
${\bm F}^{(\text{opt})}_t (\bm{x}^\prime)   \preceq_{\bm{\epsilon}} {\bm F}^{(\text{pes})}_t (\bm{x})$. 
Thus, under (\ref{eq:f_cred_cond}), this implies that $\bm{x}^\prime \preceq_{\bm{\epsilon}} \bm{x}$.
Hence, for any 
 $\bm{x}^\prime \in \Pi$, 
 there exists $\bm{x} \in \hat{\Pi}_t$ such that 
$\bm{x}^\prime \preceq_{\bm{\epsilon}} \bm{x}$.
From this and  $\bm{F}(\hat{\Pi}_t) \subset Z_{\bm{\epsilon}}$, we have that 
  $\hat{\Pi}_t$ is the $\bm{\epsilon}$-accurate Pareto set.
Here, note that  (\ref{eq:f_cred_cond}), (\ref{eq:sigma_int}) and  (\ref{eq:sigma_int_sq}) hold with probability at least 
$1-\delta$. Therefore, we get the desired result.
\end{proof}

%% file: supp_B.tex
\section{Extension to Constraint Optimization Problem}\label{sec:suppB}
In real applications, there exists a situation where the known tolerance level for the value of the function $F_2$ is defined.
For example, in the parameter tuning of an engineering system, this situation corresponds to the case where the variance of the performance must be below a certain level.
In such a situation, it is necessary to treat the functions $F_1$ and $F_2$ as in the following constrained optimization problem:
\begin{equation*}
    \bm{x}^\ast = \argmax_{\bm{x} \in \mathcal{X}} F_1(\bm{x})
    ~\text{s.t.}~F_2(\bm{x}) \geq h,
\end{equation*}
where $h < 0$ is a user-specified known threshold parameter. 
Moreover, in order to show theoretical guarantees, for a non-negative vector 
 $\bm{\epsilon} = (\epsilon_1, \epsilon_2)$, we define an $\bm{\epsilon}$-accurate solution as a solution 
 $\hat{\bm{x}}$ satisfying 
\begin{equation*}
     F_1(\hat{\bm{x}}) \geq F_1(\bm{x}^\ast) - \epsilon_1,~
     F_2(\hat{\bm{x}}) \geq h - \epsilon_2.
\end{equation*}

\paragraph{Proposed Algorithm}
First, we define  $M_t^{(\text{cons})}, S_t$ and  $M_t^{(\text{obj})}$ as 
\begin{align*}
    M_t^{(\text{cons})} &=
    \left\{ \bm{x} \in \mathcal{X} \mid u_t^{(F_2)}(\bm{x}) \geq h - \epsilon_2 \right\}, \\
    S_t &= \left\{ \bm{x} \in \mathcal{X} \mid l_t^{(F_2)}(\bm{x}) \geq h - \epsilon_2 \right\}, \\
    M_t^{(\text{obj})} &= \left\{ \bm{x} \in \mathcal{X} \mid
    u_t^{(F_1)}(\bm{x}) \geq \max_{\bm{x}^\prime \in S_t} l_t^{(F_1)}(\bm{x}^\prime) - \epsilon_1 \right\}.
\end{align*}
Here, we define $M_t^{(\text{obj})} = \mathcal{X}$ if  $S_t = \emptyset$. 
Note that an element in the complement of 
  $M_t^{(\text{cons})}$ or $M_t^{(\text{obj})}$ is not an $\epsilon$-accurate solution with high probability. 
In addition, 
$S_t$ is a set that is determined to be a feasible region with high probability.
Based on these definitions, we define a latent optimal solution set $M_t$ at the $t$th step as follows:
\begin{equation*}
    M_t = M_t^{(\text{cons})} \cap M_t^{(\text{obj})}.
\end{equation*}
In our proposed algorithm, we select the most uncertain point in the latent optimal solution set $M_t$.
In other words, the observation point $\bm{x}_t$ at the $t$th step is selected by using $\lambda_t$ as defined by Equation (\ref{eq:lambda}) as follows:
\begin{align}
    \bm{x}_t = \argmax_{\bm{x} \in M_t} \lambda_t(\bm{x}).
\end{align}

Furthermore, if $S_t \neq \emptyset$ at the $t$th step, then we define 
 the estimated optimal solution $\hat{\bm{x}}_t$ by
$\hat{\bm{x}}_t = {\rm argmax}_{\bm{x} \in S_t} l_t^{(F_1)}(\bm{x})$.
In order to ensure that $\hat{\bm{x}}_t$ is an $\epsilon$-accurate solution,  the uncertainties of the function values $F_1$ and $F_2$ for the latent optimal solution should be sufficiently small. 
In the proposed method, the algorithm terminates at the $t$th step which satisfies the following:
\begin{equation*}
    \max_{\bm{x} \in M_t} \lambda_t(\bm{x})
     \leq \min\{\epsilon_1, \epsilon_2\}.
\end{equation*}
The pseudo code of the proposed method is shown as Algorithm \ref{alg:cons}.

\begin{algorithm}[t]
    \caption{Proposed Algorithm for Constrained Optimization}
    \label{alg:cons}
    \begin{algorithmic}
        \REQUIRE GP prior $\mathcal{GP}(0,\ k)$,
        ~$\{\beta_t\}_{t \in \mathbb{N}}$,~Threshold $h$,~Non-negative vector $\bm{\epsilon} = (\epsilon_1, \epsilon_2)$.
        \STATE $M_0 \leftarrow \mathcal{X}$, $S_0 \leftarrow \emptyset$, $t \leftarrow 0$.
        \STATE Compute $\lambda_0(\bm{x})$ for any $\bm{x} \in M_0$
        \WHILE {${\rm max}_{\bm{x} \in M_t} \lambda_t(\bm{x}) \nleq \min\{\epsilon_1, \epsilon_2\}$}
            \STATE Choose $\bm{x}_t = {\rm argmax}_{\bm{x} \in M_t}
            \lambda_t(\bm{x})$.
            \STATE Sample $\bm{w}_t \sim p(\bm{w})$.
            \STATE Observe $y_t \leftarrow f(\bm{x}_t, \bm{w}_t) + \eta_t$.
            \STATE Update the GP by adding $((\bm{x}_t, \bm{w}_t), y_t)$.
            \STATE $t \leftarrow t + 1$.
            \STATE Compute $S_t$, $M_t$.
            \STATE Compute $\lambda_t(\bm{x})$ for any $\bm{x} \in M_t$
        \ENDWHILE
        \IF {$S_t \neq \emptyset$}
            \STATE Output $\hat{\bm{x}}_t = {\rm argmax}_{\bm{x} \in S_t}l_t^{(F_1)}(\bm{x})$.
        \ENDIF
    \end{algorithmic}
\end{algorithm}

\paragraph{Theoretical Analysis}
For Algorithm \ref{alg:cons}, the following theorem holds:
\begin{theorem}\label{thm:cons_conv}
    Let $k$ be a positive-definite kernel, and let   $f \in \mathcal{H}_k$ with  $\|f\|_{\mathcal{H}_k} \leq B$. 
    Also let $\delta \in (0, 1)$ and $\epsilon_1 > 0,~\epsilon_2 > 0 $, and define 
    $\beta_t = \left(\sqrt{\ln \det (\bm{I}_t + \sigma^{-2}\bm{K}_t) + 2 \ln \frac{3}{\delta}} + B\right)^2$. Then, with probability at least 
$1-\delta$, the following 1. and 2. hold:
    \begin{description}
        \item [1.] Algorithm \ref{alg:cons} terminates after at most $T$ iterations,  where  $T$ is the smallest positive integer satisfying 
        \begin{align*}
            T^{-1} \beta_T^{1/2}\left\{ \sqrt{2TC_1\gamma_T} + C_2\right\}
            + T^{-1} \sqrt{2T\tilde{B} \beta_T^{1/2}
            \left\{ \sqrt{8TC_1\gamma_T} + 2C_2 \right\}
             + 5 T\beta_T \left\{C_1\gamma_T + 2C_2 \right\}}
             \leq \min\{\epsilon_1, \epsilon_2\}.
        \end{align*}
        Here, $\tilde{B} = {\rm max}_{(\bm{x}, \bm{w}) \in (\mathcal{X} \times \Omega)}
        \left| f(\bm{x}, \bm{w}) - \mathbb{E}_{\bm{w}}[f(\bm{x}, \bm{w})]\right|$, $C_1=\frac{16}{\ln(1 + \sigma^{-2})}$ and $ C_2=16\ln \frac{18}{\delta}$.
        \item [2.] If $\bm{x}^\ast$ exists, 
        then $S_{t^\prime} \neq \emptyset$ at the termination step 
 $t^\prime \leq T$. 
Moreover, $\hat{\bm{x}}_{t^\prime} = {\rm argmax}_{\bm{x} \in S_{t^\prime}} l_t^{(F_1)}(\bm{x})$
        is an $\epsilon$-accurate solution.
    \end{description}
\end{theorem}
\begin{proof}
Assume that (\ref{eq:f_cred_cond}), (\ref{eq:sigma_int}) and  (\ref{eq:sigma_int_sq})
 hold. 
Then, by using the same argument as in the proof of Theorem 
 \ref{thm:par_conv}, we get 
\begin{equation}
    \label{eq:lambda_mean_cons}
    \frac{1}{T}\sum_{t=1}^T \lambda_t(\bm{x}_t) \leq T^{-1} \beta_T^{1/2}\left\{ \sqrt{2TC_1\gamma_T} + C_2\right\} + T^{-1} \sqrt{2T\tilde{B} \beta_T^{1/2} \left\{ \sqrt{8TC_1\gamma_T} + 2C_2 \right\}
     + 5 T\beta_T \left\{C_1\gamma_T + 2C_2 \right\}}.
\end{equation}
Here, from the definition of  $T$, the right-hand side of (\ref{eq:lambda_mean_cons}) is less than or equal to $\min\{\epsilon_1, \epsilon_2\}$. 
Hence,  there exists a positive integer $t^\prime \leq T$ such that 
 ${\rm max}_{\bm{x} \in M_{t^\prime}} \lambda_{t^\prime}(\bm{x}) =
\lambda_{t^\prime}(\bm{x}_{t^\prime}) \leq \min\{\epsilon_1, \epsilon_2\}$. 
This implies that the algorithm terminates after at most $T$ iterations. 

Next, we prove claim 2 of the theorem. Assume that ${\bm x}^\ast$ exists.  Here, we consider the two cases  ${\bm x}^\ast \in M_{t^\prime} ^{({\rm obj})}$ and ${\bm x}^\ast \notin M_{t^\prime} ^{({\rm obj})}$. For case ${\bm x}^\ast \in M_{t^\prime} ^{({\rm obj})}$, 
since (\ref{eq:f_cred_cond}) holds, the following inequality holds:
$$
h-\epsilon_2 \leq h \leq F_2 ({\bm x}^\ast ) \leq u^{(F_2)}_{t^\prime} ({\bm x}^\ast ) .
$$
This means that ${\bm x}^\ast \in M_{t^\prime} ^{({\rm cons})}$. Therefore, we have 
${\bm x}^\ast \in M_{t^\prime} $. 
Furthermore, noting that 
$u_t^{(F_1)}(\bm{x}) - l_t^{(F_1)}(\bm{x}) \leq \lambda_t(\bm{x})$ and 
$u_t^{(F_2)}(\bm{x}) - l_t^{(F_2)}(\bm{x}) \leq \lambda_t(\bm{x})$, 
it holds that 
 \begin{align}
     \label{eq:F1_cred_eps}
     \max_{\bm{x} \in M_{t^\prime}} \left\{ u_{t^\prime}^{(F_1)}(\bm{x})
      - l_{t^\prime}^{(F_1)}(\bm{x}) \right\} &\leq \epsilon_1, \\
     \label{eq:F2_cred_eps}
     \max_{\bm{x} \in M_{t^\prime}} \left\{ u_{t^\prime}^{(F_2)}(\bm{x})
      - l_{t^\prime}^{(F_2)}(\bm{x}) \right\}
      &\leq \epsilon_2.
 \end{align}
 Here, if  $l_{t^\prime}^{(F_2)}(\bm{x}^\ast) < h - \epsilon_2$, then from 
  (\ref{eq:F2_cred_eps}), we get  $u_{t^\prime}^{(F_2)}(\bm{x}^\ast) < h$. 
Thus, from (\ref{eq:f_cred_cond}), we obtain $F_2(\bm{x}^\ast) < h$. However, this contradicts the definition of $\bm{x}^\ast$, implying that $l_{t^\prime}^{(F_2)}(\bm{x}^\ast) \geq h - \epsilon_2$ and 
 $\bm{x}^\ast \in S_{t^\prime} \neq \emptyset$.
Moreover, from (\ref{eq:F1_cred_eps}) the following holds:
 \begin{align}
     \quad &\max_{\bm{x} \in M_{t^\prime}} \left\{ u_{t^\prime}^{(F_1)}(\bm{x})
      - l_{t^\prime}^{(F_1)}(\bm{x}) \right\} \leq \epsilon_1 \nonumber \\
      \Rightarrow & u_{t^\prime}^{(F_1)}(\bm{x}^\ast )
       - l_{t^\prime}^{(F_1)}({\bm x}^\ast) \leq \epsilon_1 \nonumber \\
       \Rightarrow &u_{t^\prime}^{(F_1)}(\bm{x}^\ast )
       - \max_{{\bm x} \in S_{t^\prime}}l_{t^\prime}^{(F_1)}({\bm x}) \leq \epsilon_1 \nonumber \\
\Rightarrow & u_{t^\prime}^{(F_1)}(\bm{x}^\ast )
       - l_{t^\prime}^{(F_1)}(\hat{\bm x}_{t^\prime}) \leq \epsilon_1 \nonumber \\
        \Rightarrow
         &l_{t^\prime}^{(F_1)}(\hat{\bm{x}}_{t^\prime})
         \geq  u_{t^\prime}^{(F_1)}(\bm{x}^\ast) - \epsilon_1. \nonumber
 \end{align}
 In addition, from the definition of 
  $S_{t^\prime}$, we have
 \begin{equation}
     l_{t^\prime}(\hat{\bm{x}}_{t^\prime}) \geq h - \epsilon_2. \nonumber
 \end{equation}
On the other hand, if ${\bm x}^\ast \notin M_{t^\prime} ^{({\rm obj})}$, then $M_{t^\prime} ^{({\rm obj})} \neq \mathcal{X}$. 
Thus, from the definition of $M_{t^\prime} ^{({\rm obj})}  $, it holds that $S_{t^\prime}  \neq \emptyset$. 
Therefore, we get 
$$
   l_{t^\prime}(\hat{\bm{x}}_{t^\prime}) =   \max_{{\bm x} \in S_{t^\prime}} l_{t^\prime}({\bm x})   \geq h - \epsilon_2.
$$
Furthermore, since  ${\bm x}^\ast \notin M_{t^\prime} ^{({\rm obj})}$, it holds that 
$$
u_{t^\prime}^{(F_1)}(\bm{x}^\ast) - \epsilon_1 \leq u_{t^\prime}^{(F_1)}(\bm{x}^\ast)  < l_{t^\prime}^{(F_1)}(\hat{\bm{x}}_{t^\prime}) - \epsilon_1 \leq l_{t^\prime}^{(F_1)}(\hat{\bm{x}}_{t^\prime}).
$$ 
Therefore, if ${\bm x}^\ast $ exists, then we have $S_{t^\prime} \neq \emptyset $ and 
\begin{align}
        \label{eq:cred_eps_max}
         l_{t^\prime}^{(F_1)}(\hat{\bm{x}}_{t^\prime})
        & \geq  u_{t^\prime}^{(F_1)}(\bm{x}^\ast) - \epsilon_1, \\
 \label{eq:xhat_cons}
     l_{t^\prime}(\hat{\bm{x}}_{t^\prime}) &\geq h - \epsilon_2.
\end{align}
Note that 
(\ref{eq:cred_eps_max}) and  (\ref{eq:xhat_cons}) imply that 
 $\hat{\bm{x}}_{t^\prime}$ is an $\epsilon$-accurate solution when 
  (\ref{eq:f_cred_cond}) holds. 
 Finally, since (\ref{eq:f_cred_cond}), (\ref{eq:sigma_int}) and (\ref{eq:sigma_int_sq}) hold 
with probability at least $1-\delta$, we have Theorem \ref{thm:cons_conv}.
\end{proof}


%% file: supp_C.tex
\section{Details of Section \ref{sec:ext_setting}}\label{sec:suppC}
\subsection{Noisy Input Setting}
In this subsection, we consider the setting where the input $\bm{x}$ contains a noise $\bm{\xi} \in \Delta$.
Let $\mathcal{X} \subset \mathbb{R}^{d}$ be an input space for optimization. 
In addition, assume that $\mathcal{X} $ is a finite set. 
Furthermore, let 
$\Delta \subset \mathbb{R}^{d}$ be a compact and convex set, and let 
$\bm{\xi}$ be a random noise satisfying $\bm{\xi} \in \Delta$.
Moreover, let 
  $f$ be a black-box function on $\mathcal{D}
\coloneqq \{\bm{x} + \bm{\xi} \mid \bm{x} \in \mathcal{X}, \bm{\xi} \in \Delta\}$, and 
let $k: \mathcal{D} \times \mathcal{D} \rightarrow \mathbb{R}$ be a positive-definite kernel with  
$f \in \mathcal{H}_k$ and  $\|f\|_{\mathcal{H}_k} \leq B$.

For each step $t$, we select an observation point  $\bm{x}_t \in \mathcal{X}$, and the observed value is 
obtained as 
 $y_t = f(\bm{x}_t + \bm{\xi}_t) + \eta_t$.
Here,  $\eta_t$ is the independent normal distribution 
 $\eta_t \sim \mathcal{N}(0, \sigma^2)$, and 
  $\bm{\xi}_t$ is the observed value of $\bm{\xi}$.

In this setting, the expected value and variance of $f(\bm{x})$ with respect to 
 $\bm{\xi}$ are given by 
\begin{align}
    \label{eq:noisy_e_app}
    \mathbb{E}_{\bm{\xi}}[f(\bm{x} + \bm{\xi})] &= \int_{\Delta} f(\bm{x} + \bm{\xi}) p(\bm{\xi}) \text{d}\bm{\xi}, \\
    \label{eq:noisy_v_app}
    \mathbb{V}_{\bm{\xi}}[f(\bm{x} + \bm{\xi})] &= \int_{\Delta}
    \{f(\bm{x} + \bm{\xi}) - \mathbb{E}_{\bm{\xi}}[f(\bm{x} +
     \bm{\xi})]\}^2 p(\bm{\xi}) \text{d}\bm{\xi},
\end{align}
where $p(\bm{\xi})$ is a known probability density function of $\bm{\xi}$. 
Similarly as in (\ref{eq:obj}), using \eqref{eq:noisy_e_app} and \eqref{eq:noisy_v_app} we define the optimization objective functions 
 $F_1$ and  $F_2$. 
In addition,  
let $\mu_t(\bm{x})$, $\sigma_t^2(\bm{x})$ and $Q_t(\bm{x}) \coloneqq [l_t(\bm{x}), u_t(\bm{x})]$ denote the posterior mean, posterior variance and confidence bound of $f(\bm{x})$ at the step $t$, respectively.

\paragraph{Confidence Bound}
Confidence bounds of objective functions $F_1$ and  $F_2$ defined by using 
 (\ref{eq:noisy_e_app}) and  (\ref{eq:noisy_v_app}) can also be constructed by using the same procedure as in 
 \ref{sec:cred}.
First, assume that 
$f(\tilde{\bm{x}}) \in Q_t(\tilde{\bm{x}})$ for any 
$\tilde{\bm{x}} \in \mathcal{D}$. 
Then, the following holds for any $\bm{x} \in \mathcal{X}$:
\begin{equation*}
    \int_{\Delta} l_t(\bm{x} + \bm{\xi}) p(\bm{\xi}) \text{d}\bm{\xi}
    \leq \int_{\Delta} f(\bm{x} + \bm{\xi}) p(\bm{\xi}) \text{d}\bm{\xi}
    \leq \int_{\Delta} u_t(\bm{x} + \bm{\xi}) p(\bm{\xi}) \text{d}\bm{\xi}.
\end{equation*}
Therefore, the confidence bound 
 $Q_t^{(F_1)}(\bm{x})$ of 
 $F_1(\bm{x})$ can be constructed as 
$Q_t^{(F_1)}(\bm{x}) \coloneqq
[l_t^{(F_1)}(\bm{x}), u_t^{(F_1)}(\bm{x})]$  using  
\begin{equation*}
    l_t^{(F_1)}(\bm{x}) = \int_{\Delta} l_t(\bm{x} + \bm{\xi}) p(\bm{\xi}) \text{d}\bm{\xi},~
    u_t^{(F_1)}(\bm{x}) = \int_{\Delta} u_t(\bm{x} + \bm{\xi}) p(\bm{\xi}) \text{d}\bm{\xi}.
\end{equation*}
Similarly,   the confidence bound 
$Q_t^{(F_2)}(\bm{x}) $ of 
 $F_2(\bm{x})$ can be expressed as 
$Q_t^{(F_2)}(\bm{x}) \coloneqq [l_t^{(F_2)}(\bm{x}), u_t^{(F_2)}(\bm{x})]$ using 
 \begin{align*}
     l_t^{(F_2)}(\bm{x}) = -\sqrt{\int_{\Delta} \tilde{u}_t^{(\text{sq})}
     (\bm{x} + \bm{\xi}) p(\bm{\xi}) \text{d}\bm{\xi}},~
     l_t^{(F_2)}(\bm{x}) = -\sqrt{\int_{\Delta} \tilde{l}_t^{(\text{sq})}
     (\bm{x} + \bm{\xi}) p(\bm{\xi}) \text{d}\bm{\xi}},
 \end{align*}
where $\tilde{l}_t^{(\text{sq})}(\bm{x} + \bm{\xi})$  and $ \tilde{u}_t^{(\text{sq})}(\bm{x} + \bm{\xi}) $ are given by 
 \begin{align*}
     \tilde{l}_t(\bm{x} + \bm{\xi})
     &= l_t(\bm{x} + \bm{\xi}) - \mathbb{E}_{\bm{\xi}}[u_t(\bm{x} + \bm{\xi})], \\
     \tilde{u}_t(\bm{x} + \bm{\xi})
     &= u_t(\bm{x} + \bm{\xi}) - \mathbb{E}_{\bm{\xi}}[l_t(\bm{x} + \bm{\xi})], \\
     \tilde{l}_t^{(\text{sq})}(\bm{x} + \bm{\xi})
     &= \begin{cases}
        0 & \text{if}~ \tilde{l}_t(\bm{x} + \bm{\xi}) \leq
        0 \leq \tilde{u}_t(\bm{x} + \bm{\xi}), \\
        \min \left\{\tilde{l}_t^2(\bm{x} + \bm{\xi}), \tilde{u}_t^2(\bm{x} +
        \bm{\xi}) \right\} & \text{otherwise}
     \end{cases}, \\
     \tilde{u}_t^{(\text{sq})}(\bm{x} + \bm{\xi}) &=
     \max\left\{\tilde{l}_t^2(\bm{x} + \bm{\xi}), \tilde{u}_t^2(\bm{x} +
     \bm{\xi})\right\}.
 \end{align*}
 
Using $Q_t^{(F_1)}$ and $Q_t^{(F_2)}$ above, we can construct the proposed algorithm in the same procedure.

\subsection{Simulator Based Experiment}
In this subsection, we consider the setting that $\bm{w}_t$ can be selected in the optimization phase at each step.
Furthermore, we show theoretical guarantees in this setting. 
Hereafter, we   only discuss the multi-task scenario, but the same argument can be made for multi-objective and constraint optimization scenarios by selecting  $\bm{w}_t$ and $ \bm{\xi}_t$ in the same procedure.
 
In our proposed algorithm, $(\bm{x}_t, \bm{w}_t)$  at the step $t$ is selected by
\begin{align*}
    \bm{x}_t &= \argmax_{\bm{x} \in \mathcal{X}} u_t^{(G)}(\bm{x}), \\
    \bm{w}_t &= \argmax_{\bm{w} \in \Omega} \sigma_{t-1}(\bm{x}_t, \bm{w}).
\end{align*}
In this algorithm, the following theorem holds:
\begin{theorem}\label{thm:sim_sca_conv}
    Let $k$ be a positive-definite kernel, and let   $f \in \mathcal{H}_k$ with  $\|f\|_{\mathcal{H}_k} \leq B$. 
Also let 
    $\delta \in (0, 1)$, $\epsilon > 0$, and define $\beta_t = \left( \sqrt{\ln \det (\bm{I}_t + \sigma^{-2}\bm{K}_t) + 2 \ln \frac{1}{\delta}} + B\right)^2$. 
Moreover, for any 
     $t$, define $\hat{\bm{x}}_t = {\rm argmax}_{\bm{x}_{t^\prime} \in \{\bm{x}_1, \ldots, \bm{x}_t\}} l_{t^\prime}^{(G)}(\bm{x}_{t^\prime})$. 
Then, when the proposed algorithm in the 
  simulator based setting is performed,   $\hat{\bm{x}}_T$ is the $\epsilon$-accurate solution with probability at least 
     $1-\delta$, where $T$ is the smallest positive integer satisfying 
    \begin{align*}
        \alpha T^{-1} \beta_T^{1/2} \sqrt{TC_1\gamma_T} + (1 - \alpha)T^{-1} \sqrt{4T\tilde{B} \beta_T^{1/2} \sqrt{TC_1\gamma_T}
         + 5 T\beta_T C_1\gamma_T} \leq \epsilon.
    \end{align*}
    Here,   $\tilde{B}$ and $C_1$ are given by   $\tilde{B} = {\rm max}_{(\bm{x}, \bm{w}) \in (\mathcal{X} \times \Omega)}
    \left\{ f(\bm{x}, \bm{w}) - \mathbb{E}_{\bm{w}}[f(\bm{x}, \bm{w})]\right\}$
     and $C_1=\frac{8}{\ln(1 + \sigma^{-2})}$.
\end{theorem}
\begin{proof}
Assume that (\ref{eq:f_cred_cond}) holds. Then, from Lemma 
    \ref{lem:R_bound} we have
    \begin{align*}
        &\sum_{t=1}^{T}\left\{u_t^{(G)}(\bm{x}_t) - l_t^{(G)}(\bm{x}_t)\right\}
        \leq 2 \alpha \beta_T^{1/2} \sum_{t=1}^T \int_{\Omega} \sigma_{t-1}(\bm{x}_t, \bm{w})
        p(\bm{w}) \text{d}\bm{w} \\
        &\hspace{20pt} + (1-\alpha) \sqrt{8T\tilde{B} \beta_T^{1/2}
        \sum_{t=1}^T\int_{\Omega} \sigma_{t-1}(\bm{x}_t, \bm{w}) p(\bm{w}) \text{d}\bm{w}
         + 20 T\beta_T \sum_{t=1}^T\int_{\Omega} \sigma_{t-1}^2(\bm{x}_t, \bm{w}) p(\bm{w})
          \text{d}\bm{w}}.
    \end{align*}
    In addition, from the definition of $\bm{w}_t$, it holds that 
    \begin{align*}
        \sum_{t=1}^T \int_{\Omega} \sigma_{t-1}(\bm{x}_t, \bm{w})
        p(\bm{w}) \text{d}\bm{w} \leq \sum_{t=1}^T \sigma_{t-1}(\bm{x}_t, \bm{w}_t), \\
        \sum_{t=1}^T \int_{\Omega} \sigma_{t-1}^2(\bm{x}_t, \bm{w})
        p(\bm{w}) \text{d}\bm{w} \leq \sum_{t=1}^T \sigma_{t-1}^2(\bm{x}_t, \bm{w}_t).
    \end{align*}
   Hence, we get 
    \begin{equation*}
        \sum_{t=1}^{T}\left\{u_t^{(G)}(\bm{x}_t) - l_t^{(G)}(\bm{x}_t)\right\}
        \leq 2 \alpha \beta_T^{1/2} \sum_{t=1}^T \sigma_{t-1}(\bm{x}_t, \bm{w}_t)
      + (1-\alpha) \sqrt{8T\tilde{B} \beta_T^{1/2}
        \sum_{t=1}^T \sigma_{t-1}(\bm{x}_t, \bm{w}_t)
         + 20 T\beta_T \sum_{t=1}^T \sigma_{t-1}^2(\bm{x}_t, \bm{w}_t)}.
    \end{equation*}
   Furthermore, from   (\ref{eq:sigma_sum}) and 
     (\ref{eq:sigma_sq_sum}), we obtain
    \begin{equation*}
        \sum_{t=1}^{T}\left\{u_t^{(G)}(\bm{x}_t) - l_t^{(G)}(\bm{x}_t)\right\}
         \leq \alpha \beta_T^{1/2} \sqrt{C_1T\gamma_T}
      + (1-\alpha) \sqrt{4T\tilde{B} \beta_T^{1/2}
        \sqrt{C_1T\gamma_T}
         + 5 T\beta_T C_1\gamma_T}.
    \end{equation*}
   Finally, by using the same argument as in the proof of Theorem   \ref{thm:sca_conv}, the following inequality holds:
    \begin{equation*}
        G(\bm{x}^\ast) - G(\hat{\bm{x}}_T)
       \leq  \sum_{t=1}^{T}\left\{u_t^{(G)}(\bm{x}_t) - l_t^{(G)}(\bm{x}_t)\right\}/T.
    \end{equation*}
    Therefore, noting that the definition of $T$, we get the desired result.
\end{proof}

\paragraph{Noisy Input Extension}
Here, we extend the setting defined in subsection 
\ref{sec:noise_ext} to the simulator based setting.   
Since there is the noise $\bm{\xi} \in \Delta$ instead of 
 $\bm{w}$, we consider the observation point $\bm{x}_t$ at the step $t$ as 
$\bm{x}_t \coloneqq \tilde{\bm{x}}_t + \bm{\xi}_t$, where 
  $(\tilde{\bm{x}}_t, \bm{\xi}_t)$ is given by  
\begin{align*}
    \tilde{\bm{x}}_t &= \argmax_{\bm{x} \in \mathcal{X}} u_t^{(G)}(\bm{x}), \\
    \bm{\xi}_t &= \argmax_{\bm{\xi} \in \Delta} \sigma_{t-1}(\tilde{\bm{x}}_t + \bm{\xi}).
\end{align*}
Then,  similar theorems as in Theorem    \ref{thm:sim_sca_conv} hold. 
%
However, the practical performance of this algorithm is not much different from that of Uncertainty Sampling, which was used as the base method in numerical experiments.
For this reason, in the  simulator based noisy input setting, we propose a method for selecting $(\tilde{\bm{x}}_t, \bm{\xi}_t)$ as follows:
\begin{align*}
    \tilde{\bm{x}}_t &= \argmax_{\bm{x} \in \mathcal{X}} u_t^{(G)}(\bm{x}), \\
    \bm{\xi}_t &= \argmax_{\bm{\xi} \in \Delta} \sigma_{t-1}(\tilde{\bm{x}}_t + \bm{\xi})p(\bm{\xi}).
\end{align*}
In order to derive similar convergence results as in Theorem \ref{thm:sim_sca_conv}, we assume that 
the probability density function  $p(\bm{\xi})$ of ${\bm\xi}$ is a bounded function on $\Delta$, i.e.,  
 $\sup_{\bm{\xi} \in \mathcal{D}} p(\bm{\xi}) < \infty$.
\begin{theorem}\label{thm:sim_sca_noise_conv}
    Let $\delta \in (0, 1)$, $\epsilon > 0$, and set $\beta_t = \left( \sqrt{\ln \det (\bm{I}_t + \sigma^{-2}\bm{K}_t) + 2 \ln \frac{1}{\delta}} + B \right)^2$. For any 
     $t$,  define $\hat{\bm{x}}_t = {\rm argmax}_{\bm{x}_{t^\prime} \in \{\bm{x}_1, \ldots, \bm{x}_t\}} l_{t^\prime}^{(G)}(\bm{x}_{t^\prime})$. 
    Moreover, assume that ${\rm sup}_{\bm{\bm{\xi}} \in \Delta} p(\bm{\xi}) \leq R < \infty$. Then, when the proposed algorithm in the 
    simulator based noisy input setting is performed, $\hat{\bm{x}}_T$ is the $\epsilon$-accurate solution with probability at least 
$1-\delta$, where $T$ is the smallest positive integer satisfying
    \begin{align*}
        \alpha T^{-1} \beta_T^{1/2}R \sqrt{TC_1\gamma_T} + (1 - \alpha)T^{-1} \sqrt{4T\tilde{B}R \beta_T^{1/2} \sqrt{TC_1\gamma_T}
         + 5 TR\beta_T C_1\gamma_T} \leq \epsilon.
    \end{align*}
Here,  $\tilde{B} $ and $C_1$ are given by 
    $\tilde{B} = {\rm max}_{(\bm{x}, \bm{\xi}) \in (\mathcal{X} \times \Delta)}
    \left\{ f(\bm{x} + \bm{\xi}) - \mathbb{E}_{\bm{\xi}}[f(\bm{x} + \bm{\xi})]\right\}$
    and $C_1=\frac{8}{\ln(1 + \sigma^{-2})}$.
\end{theorem}
\begin{proof}
    Similarly as in Lemma \ref{lem:R_bound}, with probability at least $1-\delta$, it holds that 
    \begin{align*}
        &\sum_{t=1}^T\left\{u_t^{(G)}(\bm{x}_t) - l_t^{(G)}(\bm{x}_t)\right\}
        \leq 2 \alpha \beta_T^{1/2} \sum_{t=1}^T \int_{\Delta} \sigma_{t-1}(\bm{x}_t + \bm{\xi})
        p(\bm{\xi}) \text{d}\bm{\xi} \\
        &\hspace{20pt} + (1-\alpha) \sqrt{8T\tilde{B} \beta_T^{1/2}
        \sum_{t=1}^T\int_{\Delta} \sigma_{t-1}(\bm{x}_t + \bm{\xi}) p(\bm{w}) \text{d}\bm{\xi}
         + 20 T\beta_T \sum_{t=1}^T\int_{\Delta} \sigma_{t-1}^2(\bm{x}_t + \bm{\xi}) p(\bm{\xi})
          \text{d}\bm{\xi}}.
    \end{align*}
    Moreover, from the definition of $\bm{\xi}_t$, we have
    \begin{align*}
        \sum_{t=1}^T \int_{\Delta} \sigma_{t-1}(\bm{x}_t + \bm{\xi})
        p(\bm{\xi}) \text{d}\bm{\xi} \leq \sum_{t=1}^T \sigma_{t-1}(\bm{x}_t + \bm{\xi}_t)p(\bm{\xi}_t)
        \leq R \sum_{t=1}^T \sigma_{t-1}(\bm{x}_t + \bm{\xi}_t), \\
        \sum_{t=1}^T \int_{\Omega} \sigma_{t-1}^2(\bm{x}_t + \bm{\xi})
        p(\bm{\xi}) \text{d}\bm{\xi} \leq \sum_{t=1}^T \sigma_{t-1}^2(\bm{x}_t + \bm{\xi}_t)p(\bm{\xi}_t)
        \leq R \sum_{t=1}^T \sigma_{t-1}^2(\bm{x}_t + \bm{\xi}_t).
    \end{align*}
    Thus, we get 
    \begin{equation*}
        \sum_{t=1}^T\left\{u_t^{(G)}(\bm{x}_t) - l_t^{(G)}(\bm{x}_t)\right\}
         \leq 2 \alpha \beta_T^{1/2}R \sum_{t=1}^T \sigma_{t-1}(\bm{x}_t + \bm{\xi}_t)
      + (1-\alpha) \sqrt{8T\tilde{B} \beta_T^{1/2}
        R \sum_{t=1}^T \sigma_{t-1}(\bm{x}_t + \bm{\xi}_t)
         + 20 T\beta_T R\sum_{t=1}^T \sigma_{t-1}^2(\bm{x}_t + \bm{\xi}_t)},
    \end{equation*}
   and
    \begin{equation*}
        \sum_{t=1}^T\left\{u_t^{(G)}(\bm{x}_t) - l_t^{(G)}(\bm{x}_t)\right\}
        \leq \alpha \beta_T^{1/2}R \sqrt{C_1T\gamma_T}
      + (1-\alpha) \sqrt{4T\tilde{B}R \beta_T^{1/2}
        \sqrt{C_1T\gamma_T}
         + 5 T\beta_TR C_1 \gamma_T}.
    \end{equation*}
    By using the same argument as in the proof of \ref{thm:sca_conv}, we obtain the following inequality:
    \begin{equation*}
        G(\bm{x}^\ast) - G(\hat{\bm{x}}_T)
        \leq \sum_{t=1}^T\left\{u_t^{(G)}(\bm{x}_t) - l_t^{(G)}(\bm{x}_t)\right\}/T.
    \end{equation*}
   Therefore, we get the desired result.
\end{proof}

%% file: supp_D.tex
\section{Extension to Continuous Set}\label{sec:suppD}
In this section, we consider the setting where $\mathcal{X}$ is a continuous set. 
First, in MT-MVA-BO, $\bm{x}_t = {\rm argmax}_{\bm{x} \in \mathcal{X}} u_t^{(G)}(\bm{x})$ 
can be calculated by using  
a continuous optimization solver. 
However, in MO-MVA-BO, it is difficult to calculate 
the estimated Pareto set 
 $\hat{\Pi}_t$ and set of latent optimal solutions $M_t$. 
In this paper, based on \cite{DBLP:conf/icml/SrinivasKKS10} we extend the proposed algorithm by using a discretization set $\tilde{\mathcal{X}}$ of $\mathcal{X}$.

Hereafter, let  $\mathcal{X} = [0, 1]^{d_1}$.  
Furthermore, assume that  
$f$ is an $L$-Lipschitz continuous function, i.e., 
  there exists  $L > 0$ such that 
\begin{align*}
    |f(\bm{x}, \bm{w}) - f(\bm{x}^\prime, \bm{w})| \leq L \|\bm{x} - \bm{x}^\prime\|_1,
\end{align*}
for any $\bm{x}, \bm{x}^\prime \in \mathcal{X}$.
Note that Lipschitz continuity holds if standard kernels are used 
 \cite{sui2015safe, DBLP:conf/icml/SuiZBY18}. 

From    Lipschitz continuity of $f$, the following lemmas  about   $F_1$ and $F_2$ hold:
\begin{lemma}\label{lem:F1_lip}
    Let $f$ be an $L$-Lipschitz continuous function.   
Then, it holds that 
    \begin{align*}
        |F_1(\bm{x}) - F_1(\bm{x}^\prime)| \leq L\|\bm{x} - \bm{x}^\prime\|_1
        ,~\forall\bm{x}, \bm{x}^\prime \in \mathcal{X},
    \end{align*}
where  
$F_1$ is given by 
(\ref{eq:obj}).
\end{lemma}
\begin{proof}
    From the definition of $F_1$ and Lipschitz continuity of   $f$, the following inequality holds: 
    \begin{align*}
        |F_1(\bm{x}) - F_1(\bm{x}^\prime)|
        &= \left|\int_{\Omega} \left\{ f(\bm{x}, \bm{w})
        - f(\bm{x}^\prime, \bm{w}) \right\} p(\bm{w})\text{d}\bm{w}\right| \\
        &\leq \int_{\Omega} |f(\bm{x}, \bm{w}) - f(\bm{x}^\prime, \bm{w})| p(\bm{w}) \text{d}\bm{w} \\
        &\leq L\|\bm{x} - \bm{x}^\prime\|_1.
    \end{align*}
\end{proof}
\begin{lemma}\label{lem:F2_lip}
    Let $f$ be an $L$-Lipschitz continuous function,        
    $\tilde{B} = {\rm max}_{(\bm{x}, \bm{w}) \in (\mathcal{X} \times \Omega)}
    \left| f(\bm{x}, \bm{w}) - \mathbb{E}_{\bm{w}}[f(\bm{x}, \bm{w})] \right|$, and define 
    $F_2$ as in (\ref{eq:obj}). Then, the following inequality holds for any 
     $\bm{x}, \bm{x}^\prime \in \mathcal{X}$:
    \begin{align*}
        |F_2(\bm{x}) - F_2(\bm{x}^\prime)| \leq
        \sqrt{4\tilde{B}L\|\bm{x} - \bm{x}^\prime\|_1}.
    \end{align*}
\end{lemma}
\begin{proof}
    From Lipschitz continuity of $f$, for any $\bm{x}, \bm{x}^\prime \in \mathcal{X}$,
    $\bm{w} \in \Omega$, it holds that 
    \begin{align*}
        &\left|\left\{f(\bm{x}, \bm{w}) - \mathbb{E}_{\bm{w}}[f(\bm{x}, \bm{w})]\right\}^2 -
        \left\{f(\bm{x}^\prime, \bm{w}) - \mathbb{E}_{\bm{w}}[f(\bm{x}^\prime, \bm{w})]\right\}^2\right| \\
        =&\left|\left\{f(\bm{x}, \bm{w}) - \mathbb{E}_{\bm{w}}[f(\bm{x}, \bm{w})]\right\} -
        \left\{f(\bm{x}^\prime, \bm{w}) - \mathbb{E}_{\bm{w}}[f(\bm{x}^\prime, \bm{w})]\right\}\right| \times
        \left|\left\{f(\bm{x}, \bm{w}) - \mathbb{E}_{\bm{w}}[f(\bm{x}, \bm{w})]\right\} +
        \left\{f(\bm{x}^\prime, \bm{w}) - \mathbb{E}_{\bm{w}}[f(\bm{x}^\prime, \bm{w})]\right\}\right| \\
        \leq& \left(|f(\bm{x}, \bm{w}) - f(\bm{x}^\prime, \bm{w})|+
        |\mathbb{E}_{\bm{w}}[f(\bm{x}, \bm{w})] - \mathbb{E}_{\bm{w}}[f(\bm{x}^\prime, \bm{w})]|\right) \times
        \left(|f(\bm{x}, \bm{w}) - \mathbb{E}_{\bm{w}}[f(\bm{x}, \bm{w})]|+
        |f(\bm{x}^\prime, \bm{w}) - \mathbb{E}_{\bm{w}}[f(\bm{x}^\prime, \bm{w})]| \right)\\
        \leq& 2L\|\bm{x} - \bm{x}^\prime\|_1\times 2\tilde{B} \\
         =& 4\tilde{B}L \|\bm{x} - \bm{x}^\prime\|_1.
    \end{align*}
    Here, if $F_2(\bm{x}) \geq F_2(\bm{x}^\prime)$, then 
    \begin{align*}
        &\left|F_2(\bm{x}) - F_2(\bm{x}^\prime)\right| \\
        =&F_2(\bm{x}) - F_2(\bm{x}^\prime) \\
        =& \sqrt{\int_{\Omega} \left\{f(\bm{x}^\prime, \bm{w})
        - \mathbb{E}_{\bm{w}}[f(\bm{x}^\prime, \bm{w})] \right\}^2 p(\bm{w}) \text{d}\bm{w}}
        - \sqrt{\int_{\Omega} \left\{f(\bm{x}, \bm{w})
        - \mathbb{E}_{\bm{w}}[f(\bm{x}, \bm{w})] \right\}^2 p(\bm{w}) \text{d}\bm{w}} \\
        \leq &\sqrt{\int_{\Omega} \left\{f(\bm{x}^\prime, \bm{w})
        - \mathbb{E}_{\bm{w}}[f(\bm{x}^\prime, \bm{w})] \right\}^2 p(\bm{w}) \text{d}\bm{w}
        - \int_{\Omega} \left\{f(\bm{x}, \bm{w})
        - \mathbb{E}_{\bm{w}}[f(\bm{x}, \bm{w})] \right\}^2 p(\bm{w}) \text{d}\bm{w}} \\
        \leq &\sqrt{\int_{\Omega} \left|\left\{f(\bm{x}^\prime, \bm{w})
        - \mathbb{E}_{\bm{w}}[f(\bm{x}^\prime, \bm{w})] \right\}^2
        - \left\{f(\bm{x}, \bm{w})
        - \mathbb{E}_{\bm{w}}[f(\bm{x}, \bm{w})] \right\}^2\right| p(\bm{w}) \text{d}\bm{w}} \\
        \leq &\sqrt{4\tilde{B}L\|\bm{x} - \bm{x}^\prime\|_1}.
    \end{align*}
    On the other hand, if $F_2(\bm{x}) < F_2(\bm{x}^\prime)$, it holds that 
    $|F_2(\bm{x}) - F_2(\bm{x}^\prime)| \leq \sqrt{4\tilde{B}L\|\bm{x} - \bm{x}^\prime\|_1}$.
     Therefore, for any $\bm{x}, \bm{x}^\prime \in \mathcal{X}$, the desired inequality holds.
\end{proof}
Moreover, the following lemma holds:
\begin{lemma}\label{lem:separate_Pareto_front}
Let $Z$ be the Pareto front for $\mathcal{X}$, and let ${\bm\epsilon} =(\epsilon_1,\epsilon_2)^\top $ be a positive vector. 
Define 
\begin{align*}
Z^{+} &=  \bigcup_ { (y_1,y_2) \in Z }  (-\infty, y_1 ] \times (-\infty,y_2], \ 
Z^{-} ({\bm\epsilon}) =  \bigcup_ { (y_1,y_2) \in Z }  (-\infty, y_1-\epsilon_1 ) \times (-\infty,y_2-\epsilon_2), \\
Z^\ast ({\bm\epsilon})  &= \{ (y_1 - \epsilon^\prime_1, y_2 -\epsilon^\prime_2 ) \mid (y_1,y_2 ) \in Z, \ 0 \leq  \epsilon^\prime_1 \leq \epsilon_1, 
 0 \leq  \epsilon^\prime_2 \leq \epsilon_2 \}.
\end{align*}
Then, it holds that 
$$
Z^+ = Z^-   ({\bm\epsilon})  \cup Z^\ast  ({\bm\epsilon})  , \ Z^-  ({\bm\epsilon})  \cap Z^\ast  ({\bm\epsilon})  = \emptyset.
$$
\end{lemma}
\begin{proof}
First, we show $Z^-  ({\bm\epsilon})  \cap Z^\ast  ({\bm\epsilon})  = \emptyset$. 
Let ${\bm y} $ be an element of $Z^-  ({\bm\epsilon}) $. Then, there exists $(y^\prime_1,y^\prime_2 ) \in Z$ such that 
$$
y_1 < y^\prime_1 - \epsilon_1, \ y_2 < y^\prime_2 - \epsilon_2.
$$
Here, for any  $(y^{\prime\prime}_1, y^{\prime\prime}_2 ) \in Z$, $ y^{\prime\prime}_1$ satisfies $y^\prime_1 \leq y^{\prime\prime}_1$
 or $y^\prime_1 > y^{\prime\prime}_1$. 
If $y^\prime_1 \leq y^{\prime\prime}_1$, from $y_1 < y^\prime_1 - \epsilon_1$ we get 
$$
{\bm y} \notin \{  (y^{\prime\prime}_1- \epsilon^\prime_1, y^{\prime\prime}_2- \epsilon^\prime_2) \mid 0 \leq  \epsilon^\prime_1 \leq \epsilon_1, 
 0 \leq  \epsilon^\prime_2 \leq \epsilon_2 \}.
$$
On the other hand, if $y^\prime_1 > y^{\prime\prime}_1$, then $y^{\prime\prime}_2$ satisfies $y^\prime_2 \leq  y^{\prime\prime}_2$
 because the inequality $y^\prime_2 > y^{\prime\prime}_2$ implies that
 $ (y^{\prime\prime}_1,y^{\prime\prime}_2) \in (-\infty, y^\prime_1) \times (-\infty,y^\prime_2 ) $. However, it contradicts that $(y^{\prime\prime}_1,y^{\prime\prime}_2) \in Z$.
 From $y^\prime_2 \leq  y^{\prime\prime}_2$ and $y_2 < y^\prime_2 - \epsilon_2 $, we have 
$$
{\bm y} \notin \{  (y^{\prime\prime}_1- \epsilon^\prime_1, y^{\prime\prime}_2- \epsilon^\prime_2) \mid 0 \leq  \epsilon^\prime_1 \leq \epsilon_1, 
 0 \leq  \epsilon^\prime_2 \leq \epsilon_2 \}.
$$
Therefore, it holds that ${\bm y} \notin Z^\ast ({\bm\epsilon})$. This implies that $Z^-  ({\bm\epsilon})  \cap Z^\ast  ({\bm\epsilon})  = \emptyset$. 

Next, we show $Z^+ = Z^-   ({\bm\epsilon})  \cup Z^\ast  ({\bm\epsilon})  $. 
It is clear that  $Z^+ \supset Z^-   ({\bm\epsilon})  \cup Z^\ast  ({\bm\epsilon})  $. 
Thus, we only show that $Z^+ \subset  Z^-   ({\bm\epsilon})  \cup Z^\ast  ({\bm\epsilon})  $. 
Let ${\bm y} $ be an element of $Z^+$. If ${\bm y} \in Z^- ({\bm \epsilon})$, it holds that ${\bm y} \in 
Z^-   ({\bm\epsilon})  \cup Z^\ast  ({\bm\epsilon})$. 
On the other hand, if ${\bm y} \notin  Z^- ({\bm \epsilon})$, 
at least one of the following inequalities holds for any $(y^\prime_1,y^\prime_2 ) \in Z$:
$$
y_1 \geq y^\prime_1 - \epsilon_1,\ y_2 \geq y^\prime _2 -\epsilon_2.
$$
If there exists   $\epsilon^\prime_1 \in [0,\epsilon_1] $ such that 
$(y_1+\epsilon^\prime_1,y_2)  \in Z$, then ${\bm y}   \in Z^\ast  ({\bm\epsilon}) $.
 Next, we consider the case that $(y_1+\epsilon^\prime_1,y_2)  \notin Z$ for any  $\epsilon^\prime_1 \in [0,\epsilon_1] $. 
Let $Z^\prime =\{ {\bm a}=(a_1,a_2) \in Z \mid  y_1 \leq a_1 \leq y+ \epsilon_1 \}$.  
Here, assume that $y_2 < a_2 -\epsilon_2 $ for any ${\bm a} \in Z^\prime$. 
Then, from continuity of $Z$, there exists $\hat{\bm y}=(\hat{y}_1,\hat{y}_2)  \in Z$ such that $y_1  < \hat{y}_1-\epsilon_1$ and 
$y_2  < \hat{y}_2-\epsilon_2$. However, it contradicts ${\bm y} \notin  Z^- ({\bm \epsilon})$. 
Hence, there exists an element ${\bm a} =(a_1,a_2 ) \in Z^\prime $ such that $y_2 \geq a_2 - \epsilon_2$. 
Moreover, there exists $b \geq y_2$ such that $(y_1,b) \in Z$. 
This implies that there exist $\tilde{\epsilon}_1$ and $\tilde{\epsilon}_2$ such that 
$0 \leq \tilde{\epsilon}_1 \leq \epsilon_1$,  
$0 \leq \tilde{\epsilon}_2 \leq \epsilon_2$ and $(y_1 + \tilde{\epsilon}_1,y_2 +\tilde{\epsilon}_2 ) \in Z$. 
Therefore,  it holds that ${\bm y} \in   Z^\ast ({\bm\epsilon})$.
\end{proof}

Next, we explain the method of constructing  $\tilde{\mathcal{X}}$.
Let $\tilde{\mathcal{X}}$ be a set of grid points when each dimension of $\mathcal{X}=[0, 1]^{d_1}$ is divided into $\tau$ evenly spaced segments.
Also let $[\bm{x}] \in \tilde{\mathcal{X}}$  be a point closest to $\bm{x} \in \mathcal{X}$  with respect to the $L1$-distance. 
Then, it holds that 
\begin{equation}
    \label{eq:dis_dis}
    \|\bm{x} - [\bm{x}]\|_1 \leq \frac{d_1}{\tau},~\forall \bm{x} \in \mathcal{X}.
\end{equation}

In the proposed algorithm for the continuous set setting, 
Algorithm \ref{alg:pareto} is performed by using $\tilde{\mathcal{X}}$ instead of $\mathcal{X}$.
Then, we define the  estimated Pareto set $\hat{\Pi}_t$, latent Pareto set $M_t$ and 
uncertain set $U_t$ in Algorithm 
 \ref{alg:pareto} as
\begin{align*}
    \hat{\Pi}_t &= \left\{ \bm{x} \in \tilde{\mathcal{X}} \mid
    \forall \bm{x}^\prime \in \tilde{E}^{(\text{pes})}_{t,{\bm x} } ,~
    \bm{F}_{t}^{(\text{pes})}(\bm{x}) \npreceq \bm{F}_{t}^{(\text{pes})}(\bm{x}^\prime)\right\},\  \tilde{E}^{(\text{pes})}_{t,{\bm x} } =
\{  {\bm x}^\prime \in \tilde{\mathcal{X} } \mid  {\bm F} ^{(\text{pes})}_t ({\bm x}) \neq  {\bm F} ^{(\text{pes})}_t ({\bm x}^\prime) \} , \\
    M_t &= \left\{\bm{x} \in \tilde{\mathcal{X}} \setminus \hat{\Pi}_t \mid
    \forall \bm{x}^\prime \in \hat{\Pi}_t,~
    \bm{F}_{t}^{(\text{opt})}(\bm{x}) \npreceq_{\bm{\epsilon}/2} \bm{F}_{t}^{(\text{pes})}(\bm{x}^\prime)\right\}, \\
    U_t &= \left\{\bm{x} \in \hat{\Pi}_t \mid \exists \bm{x}^\prime \in
    \hat{\Pi}_t \setminus \{\bm{x}\},
    \bm{F}_{t}^{(\text{pes})}(\bm{x}) + \bm{\epsilon}/2 \prec
    \bm{F}_{t}^{(\text{opt})}(\bm{x}^\prime) \right\}.
\end{align*}
Note that $\bm{\epsilon}/2$, not $\bm{\epsilon}$  is used to calculate  $\tilde{M}_t$ and $\tilde{U}_t$.

In the algorithm using 
 $\tilde{\mathcal{X}}$, the following theorem holds:
\begin{theorem}\label{thm:cont_par_conv}
    Let $\tilde{B} = {\rm max}_{(\bm{x}, \bm{w}) \in (\mathcal{X} \times \Omega)}
    \left| f(\bm{x}, \bm{w}) - \mathbb{E}_{\bm{w}}[f(\bm{x}, \bm{w})]\right|$, and let 
     $\delta \in (0, 1)$, $\bm{\epsilon} = (\epsilon_1, \epsilon_2)$ where  $\epsilon_1 > 0$ and 
    $\epsilon_2 > 0$. 
Define 
    $\beta_t = \left(\sqrt{\ln \det (\bm{I}_t + \sigma^{-2}\bm{K}_t) + 2 \ln \frac{3}{\delta}} + B\right)^2$ and 
    $\tau = \max\left\{ \frac{2Ld_1}{\epsilon_1}, \frac{16\tilde{B}Ld_1}{\epsilon_2^2} \right\}$. 
    Then,  the following (1) and (2) hold with probability at least $1-\delta$:
    \begin{description}
        \item [(1)]  The algorithm terminates after at most $T$ iterations, where $T$ is the smallest positive integer satisfying
        \begin{align*}
            T^{-1} \beta_T^{1/2}\left( \sqrt{2TC_1\gamma_T} + C_2 \right) + T^{-1} \sqrt{2T\tilde{B} \beta_T^{1/2} \left( \sqrt{8TC_1\gamma_T} + 2C_2 \right)
             + 5 T\beta_T \left(C_1\gamma_T + 2C_2 \right)} \leq \min\{\epsilon_1, \epsilon_2\}/2.
        \end{align*}
       Here, $C_1$ and $C2$ are given by  $C_1=\frac{16}{\ln(1 + \sigma^{-2})}$ and $ C_2=16\ln \frac{18}{\delta}$.
        \item [(2)] When the algorithm is terminated, the estimated Pareto set $\hat{\Pi}$ is the $\epsilon$-accurate Pareto Set.
    \end{description}
\end{theorem}
\begin{proof}
We omit the proof of {\sf (1)} because its proof is  the same as in the proof of Theorem 
   \ref{thm:par_conv}. 
 We only prove {\sf (2)}. 
 From  (\ref{eq:dis_dis}) and Lemma   \ref{lem:F1_lip}--\ref{lem:F2_lip}, the following holds for any 
     $\bm{x} \in \mathcal{X}$:
    \begin{align}
        |F_1(\bm{x}) - F_1([\bm{x}])|
        &\leq L \|\bm{x} - [\bm{x}]\|_1 \nonumber \\
        \label{eq:F1_eps1}
        &= \frac{\epsilon_1}{2},\\
        |F_2(\bm{x}) - F_2([\bm{x}])|
        &\leq \sqrt{4\tilde{B}L\|\bm{x} - [\bm{x}]\|_1} \nonumber \\
        \label{eq:F2_eps2}
        &= \frac{\epsilon_2}{2}.
    \end{align}
   Assume that (\ref{eq:f_cred_cond}) holds.
    Let $\tilde{Z}$ be a Pareto front for 
    $\tilde{\mathcal{X}}$. 
Then, for any ${\bm y} \in \tilde{Z}$, it holds that 
\begin{equation}
{\bm y} \in \bigcup_{ (y^\prime_1,y^\prime_2 ) \in Z }   (-\infty,y^\prime_1] \times (-\infty, y^\prime_2 ] , \label{eq:yinZplus}
\end{equation}
where $Z$ is the Pareto front for $\mathcal{X}$. 
Similarly, let 
$$
Z^- (\bm\epsilon/2) = \bigcup_{ (y^\prime_1,y^\prime_2 ) \in Z }   (-\infty,y^\prime_1-\epsilon_1/2) \times (-\infty, y^\prime_2-\epsilon_2/2).  
$$
Then, for any ${\bm y}^{\prime\prime}  \in Z^-  (\bm\epsilon/2)$, there exists ${\bm x} \in \mathcal{X}$ such that 
$$
y^{\prime\prime}_1 <   F_1 ({\bm x}) -\epsilon_1/2, \ 
y^{\prime\prime}_2 <   F_2 ({\bm x}) -\epsilon_2/2.
$$
Here, from (\ref{eq:F1_eps1}) and  (\ref{eq:F2_eps2}) we have 
$$
 F_1 ({\bm x}) \leq  F_1 ([{\bm x}]) + \epsilon_1/2,  F_2 ({\bm x}) \leq  F_2 ([{\bm x}]) + \epsilon_2/2.
$$
Thus, it holds that $y^{\prime\prime}_1 <   F_1 ([{\bm x}])$ and   $y^{\prime\prime}_2 <   F_2 ([{\bm x}])$. 
This implies that 
$$
Z^-  (\bm\epsilon/2) \subset \{ {\bm y } \in \mathbb{R} \mid \exists {\bm x} \in \tilde{\mathcal{X}}, \ {\bm y} \preceq {\bm F} ({\bm x} ) \}    \equiv A.
$$ 
Here, since $Z^-  (\bm\epsilon/2)$ is the open set, noting that  $Z^-  (\bm\epsilon/2) \subset A$ we get 
$Z^-  (\bm\epsilon/2) \subset \text{int} (A) $, where $\text{int} (A) $ is the interior of $A$. 
In addition, from the definition of the interior and boundary (frontier), we obtain $\text{int} (A)  \cap \partial A = \emptyset$.  
Therefore, from $\partial A = \tilde{Z}$ and $Z^-  (\bm\epsilon/2) \subset \text{int} (A) $, it holds that 
$Z^-  (\bm\epsilon/2) \cap \tilde{Z} = \emptyset$. Hence, for any ${\bm y} \in \tilde{Z}$, ${\bm y} \notin Z^-  (\bm\epsilon/2)$.
 Thus, by using this and \eqref{eq:yinZplus}, from Lemma \ref{lem:separate_Pareto_front}, it holds that 
$$
\tilde{Z} \subset Z^\ast ({\bm\epsilon}/2 ) . 
$$
Hence, for any ${\bm y} \in \tilde{Z}$, there exists ${\bm a} \in Z$ such that 
\begin{align}
y_1 = a_ 1 - \epsilon ^\prime_1, y_2 = a_2 - \epsilon^\prime_2 , \quad 0 \leq \epsilon^\prime_1 \leq \epsilon_1/2, \ 
 0 \leq \epsilon^\prime_2 \leq \epsilon_2/2. \label{eq:yexists}
\end{align}
Furthermore, from Theorem \ref{thm:par_conv}, for any ${\bm x} \in \hat{\Pi}_t $, 
there exists ${\bm y}^\dagger \in \tilde{Z} $ such that 
$$
y^\dagger_1 \leq F_1({\bm x} ) + \epsilon_1/2, \ y^\dagger_2 \leq F_2 ({\bm x} )  + \epsilon_2 /2.
$$
By combining this and \eqref{eq:yexists}, we get 
\begin{align*}
a_1 = y^\dagger_1 + \epsilon^\prime_1 \leq F_1 ({\bm x}) + \epsilon_1/2  + \epsilon^\prime_1 \leq F_1 ({\bm x}) + \epsilon_1, \\
a_2 = y^\dagger_2 + \epsilon^\prime_2 \leq F_2 ({\bm x}) + \epsilon_2/2  + \epsilon^\prime_2 \leq F_2 ({\bm x}) + \epsilon_2.
\end{align*}
Therefore, we have $\bm{F}(\hat{\Pi}_t) \subset Z_{\bm{\epsilon}}$.

    Furthermore, let $\bm{x} \in \Pi$. 
For $[\bm{x}] \in \tilde{\mathcal{X}}$, since $\hat{\Pi}_t$ is the $(\bm{\epsilon}/2)$-accurate Pareto set for $\tilde{\mathcal{X}}$, 
there exists 
$\bm{x}^\prime \in \hat{\Pi}_t$ such that $\bm{F}([\bm{x}]) \preceq_{\bm{\epsilon}/2}
    \bm{F}(\bm{x}^\prime)$. 
Moreover, form (\ref{eq:F1_eps1}) and  (\ref{eq:F2_eps2}), it holds that 
      $\bm{F}(\bm{x}) \leq \bm{F}([\bm{x}]) + \bm{\epsilon}/2$.
 This implies that 
    $\bm{F}(\bm{x}) \preceq \bm{F}([\bm{x}]) + \bm{\epsilon}/2 \preceq \bm{F}(\bm{x}^\prime) + \bm{\epsilon}$. 
Therefore, for any $\bm{x} \in \Pi$, there exits 
$\bm{x}^\prime \in \hat{\Pi}_t$ such that 
 $\bm{x} \preceq_{\bm{\epsilon}} \bm{x}^\prime$.
    Thus, 
    $\hat{\Pi}_t$ is the $\bm{\epsilon}$-accurate Pareto set for $\mathcal{X}$.
\end{proof}